\documentclass[aos, preprint]{imsart}

\RequirePackage[OT1]{fontenc}
\RequirePackage{amsthm,amsmath}
\RequirePackage[numbers]{natbib}
\RequirePackage[colorlinks,citecolor=blue,urlcolor=blue]{hyperref}
\RequirePackage[colorinlistoftodos]{todonotes}
\RequirePackage{mathrsfs}

\usepackage[utf8]{inputenc}
\usepackage{amssymb,amsfonts}
\usepackage{amsmath}
\usepackage{amsthm}
\usepackage{dsfont}
\usepackage[all,arc]{xy}
\usepackage{enumerate}
\usepackage{mathrsfs}
\usepackage{stmaryrd}
\usepackage[colorinlistoftodos]{todonotes}
\usepackage{multirow}
\usepackage{subcaption}
\usepackage{placeins}
\usepackage{float}
\usepackage{booktabs}
\usepackage{empheq}
\usepackage{bm}

\usepackage{algorithmic}
\usepackage{algorithm}

\usepackage{tikz} 
\usetikzlibrary{shapes,arrows,calc,decorations.pathreplacing}
\usetikzlibrary{arrows.meta}

\definecolor{myred}{rgb}{0.77, 0.0, 0.1}
\definecolor{mygreen}{rgb}{0.0,0.6,0.0}
\definecolor{newgreen}{RGB}{0,153,0}
\definecolor{myturq}{rgb}{0.1, 0.7, 0.7}
\definecolor{mygold}{rgb}{1.0, 0.84, 0.0}
\definecolor{almond}{rgb}{0.94, 0.87, 0.8}
\definecolor{lightyellow}{rgb}{0.94, 0.92, 0.8}
\definecolor{lightblue}{rgb}{0.145,0.6666,1} 
\definecolor{darkblue}{rgb}{0.2,0.2,0.6}

\arxiv{arXiv:0000.0000}

\startlocaldefs

\numberwithin{equation}{section}
\theoremstyle{plain}

\newtheorem{proposition}{Proposition}
\newtheorem{theorem}{Theorem}
\newtheorem{lemma}{Lemma}

\newtheorem{corollary}{Corollary}
\newtheorem{fact}{Fact}

\theoremstyle{definition}
\newtheorem{definition}{Definition}
\newtheorem{assumption}{Assumption}

\theoremstyle{remark}
\newtheorem{remark}{Remark} 

\renewcommand{\leq}{\leqslant}
\renewcommand{\geq}{\geqslant}

\newcommand{\di}{\mathrm{d}}

\newcommand{\eps}{\varepsilon}

\newcommand{\vol}{\mathop{\mathrm{vol}}}
\newcommand{\argmin}{\mathop{\mathrm{argmin}}}

\newcommand{\wt}{\widetilde}
\newcommand{\wh}{\widehat}

\newcommand{\N}{\mathbf N}

\newcommand{\R}{\mathbf R}

\newcommand{\F}{\mathscr F}

\newcommand{\cl}{\mathscr C}

\newcommand{\G}{\mathcal{G}}

\usepackage{xspace}
\newcommand{\ie}{\textit{i.e.}\@\xspace} 
\newcommand{\eg}{e.g.\@\xspace}
\newcommand{\iid}{i.i.d.\@\xspace}

\newcommand{\E}{\mathbb E}
\renewcommand{\P}{\mathbb P}
\newcommand{\Var}{\mathrm{Var}}

\newcommand{\cond}{\,|\,}

\newcommand{\indic}[1]{\bm 1 ( #1 )}

\newcommand{\expdist}{\mathsf{Exp}}

\newcommand{\gammadist}{\mathsf{Gamma}}

\newcommand{\binomdist}{\mathsf{Bin}}
\newcommand{\uniformdist}{\mathcal{U}}

\newcommand{\risk}{R}

\newcommand{\dataset}{\mathscr D}

\newcommand{\MP}{\mathop{\mathsf{MP}}} 
\newcommand{\splits}{\Sigma} 
\newcommand{\asplit}{\sigma} 
\newcommand{\node}{\mathbf{v}} 

\newcommand{\depth}{\mathrm{depth}}
\newcommand{\nodes}{\mathcal N} 
\newcommand{\inodes}{\mathcal N^{\circ}} 
\newcommand{\leaves}{\mathcal L} 

\newcommand{\leaf}{\node}
\newcommand{\diam}{{\mathrm{diam}}\, }
\newcommand{\cell}{C}
\newcommand{\cut}{s} 
\newcommand{\Cut}{S} 
\renewcommand{\root}{\epsilon} 

\newcommand{\birth}{\tau} 
\newcommand{\births}{\mathfrak{T}} 
\newcommand{\globaltree}{\mathbf{T}} 
\newcommand{\tree}{T} 
\newcommand{\treepart}{\Pi} 
\newcommand{\mondrian}{\Pi} 

\newcommand{\prefix}{\sqsubset} 
\newcommand{\prefixeq}{\sqsubseteq}

\newcommand{\expleft}[1]{E_{{#1},L}}
\newcommand{\expright}[1]{E_{#1,R}}
\newcommand{\infp}{p_0}
\newcommand{\supp}{p_1}

\newcommand{\SampleMondrian}{\mathtt{SampleMondrian}}

\endlocaldefs

\usepackage[acronym,nomain,nogroupskip,nonumberlist,nopostdot,toc]{glossaries}        

\newglossary[slg]{symbolslist}{syi}{syg}{List of Symbols} 

\makenoidxglossaries

\newglossaryentry{dataset}{type=symbolslist,name={\ensuremath{\dataset_n}},sort=aaa,
description={Data set}}

\newglossaryentry{mu}{type=symbolslist,name={\ensuremath{\mu}},sort=bbb,
description={Distribution of $X$ on $[0,1]^d$}}

\newglossaryentry{cell}{type=symbolslist,name={\ensuremath{\cell, ~\textrm{resp.} ~|C| }},sort=fnc,
description={A generic cell $\cell \subset [0,1]^d$, resp.  half-perimeter of $C$}}

\newglossaryentry{lifetime}{type=symbolslist,name={\ensuremath{\lambda}},sort=fnc,
description={Lifetime parameter of Mondrian process}}

\newglossaryentry{Mondrian_lambda_c}{type=symbolslist,name={\ensuremath{\MP(\lambda, C)}},sort=fnc,
description={Distribution of a Mondrian process defined on cell $C$ with lifetime parameter $\lambda$.}}

\newglossaryentry{Partition_Mondrian_lambda}{type=symbolslist,name={\ensuremath{\Pi_{\lambda}, ~\textrm{resp.}~\Pi_{\lambda} | C}},sort=fnc,
description={Partition drawn from $\MP(\lambda, [0,1]^d )$, resp. from $\MP(\lambda, C )$}}

\newglossaryentry{cell_mondrian}{type=symbolslist,name={\ensuremath{\cell_{\lambda}(x)}},sort=fnc,
description={Cell of a Mondrian Tree with parameter $\lambda$ containing $x$.}}

\newglossaryentry{diam_cell_mondrian}{type=symbolslist,name={\ensuremath{D_\lambda (x)}},sort=fnc,
description={Diameter of $\cell_{\lambda}(x)$}}

\newglossaryentry{nb_cell_mondrian}{type=symbolslist,name={\ensuremath{K_\lambda}},sort=fnc,
description={Number of cells in a Mondrian Tree partition $\Pi_{\lambda}$}}

\newglossaryentry{tree_estimate}{type=symbolslist,name={\ensuremath{\widehat{f}_{\lambda, n}^{(m)}(x
)}},sort=fnc,
description={Mondrian Tree estimate at query point $x$ based on the Mondrian partition $\Pi_\lambda^{(m)}$}}

\newglossaryentry{forest_estimate}{type=symbolslist,name={\ensuremath{\widehat{f}_{\lambda, n, M}(x)}},sort=fnc,
description={Mondrian Forest estimate at query point $x$ based on the Mondrian partitions $\Pi_{\lambda, M} = (\Pi_\lambda^{(1)}, \ldots, \Pi_\lambda^{(M)})$}}

\newglossaryentry{bias}{type=symbolslist,name={\ensuremath{\bar f_\lambda^{(m)} (x)}},sort=fnc,
description={Expected value of the regression function $f$ inside the cell $\cell_\lambda^{(m)} (x)$}}

\newglossaryentry{variance}{type=symbolslist,name={\ensuremath{\wt f_\lambda (x) }},sort=fnc,
description={Expected value of $\bar f_\lambda^{(m)} (x)$ over $\Pi_{\lambda}^{(m)} \sim \MP(\lambda, [0,1]^d)$}}

\newglossaryentry{nodes}{type=symbolslist,name={\ensuremath{\nodes (T),\inodes (T), \leaves (T)}},sort=fnc,
description={Nodes, interior nodes and leaves of a tree}}

\newglossaryentry{set_of_splits}{type=symbolslist,name={\ensuremath{\splits = (\asplit_\node)_{\node \in \inodes (T)}}},sort=fnc,
description={Set of splits for all nodes in the tree}}

\newglossaryentry{one_split}{type=symbolslist,name={\ensuremath{\sigma_\node = (j_\node, \cut_\node)}},sort=fnc,
description={A split at node $\node$ characterized by its split dimension $j_\node \in \{ 1, \dots, d \}$ and its threshold $\cut_\node \in [0, 1]$}}

\newglossaryentry{birth_time}{type=symbolslist,name={\ensuremath{\tau_\node}},sort=fnc,
description={Birth time of a node $\node$ }}

\newglossarystyle{notationlong}{%
\setglossarystyle{long}
  {\end{longtable}}%

}

\begin{document}

\begin{frontmatter}
\title{Minimax optimal rates for Mondrian trees and forests}
\runtitle{Minimax optimal rates for Mondrian trees and forests}


\begin{aug}
\author{\fnms{Jaouad} \snm{Mourtada}\thanksref{t1,m1}\ead[label=e1]{jaouad.mourtada@polytechnique.edu}},
\author{\fnms{St\'ephane} \snm{Ga\"iffas}\thanksref{t1,m2}\ead[label=e2]{stephane.gaiffas@lpsm.paris}}
\and
\author{\fnms{Erwan} \snm{Scornet}\thanksref{t1,m1}
\ead[label=e3]{erwan.scornet@polytechnique.edu}}

\thankstext{t1}{Data Science Initiative of \'Ecole polytechnique}

\runauthor{J. Mourtada, S. Ga\"iffas and E. Scornet}

\affiliation{\'Ecole polytechnique\thanksmark{m1} \and Universit\'e Paris Diderot\thanksmark{m2}}

\address{
Jaouad Mourtada \\
CMAP, Ecole Polytechnique \\
Route de Saclay \\
91128 Palaiseau cedex, France \\
\printead{e1}}

\address{
St\'ephane Ga\"iffas \\
LPMA - Univ. Paris Diderot \\
Bâtiment Sophie Germain \\
Case courrier 7012  \\
75205 PARIS CEDEX 13, France \\
\printead{e2}}

\address{
Erwan Scornet \\
CMAP, Ecole Polytechnique \\
Route de Saclay \\
91128 Palaiseau cedex, France \\
\printead{e3}}
\end{aug}

\begin{abstract}
  Introduced by Breiman~\cite{breiman2001randomforests}, Random Forests are widely used classification and regression algorithms.
  While being initially designed as batch algorithms, several variants have been proposed to handle online learning.
  One particular instance of such forests is the \emph{Mondrian Forest}~\cite{lakshminarayanan2014mondrianforests,lakshminarayanan2016mondrianuncertainty}, whose trees are built using the so-called Mondrian process, therefore allowing to easily update their construction in a streaming fashion.
  In this paper, we provide a thorough theoretical study of Mondrian Forests in a batch learning setting, based on new results about Mondrian partitions.
  Our results include consistency and convergence rates for Mondrian Trees and Forests, that turn out to be minimax optimal on the set of $s$-H\"older function with $s \in (0,1]$ (for trees and forests) and $s \in (1,2]$ (for forests only), assuming a proper tuning of their complexity parameter in both cases.
  Furthermore, we prove that an adaptive procedure (to the unknown $s \in (0, 2]$) can be constructed by combining Mondrian Forests with a standard model aggregation algorithm.
  These results are the first demonstrating that some particular random forests achieve minimax rates \textit{in arbitrary dimension}.
  Owing to their remarkably simple distributional properties, which lead to minimax rates, Mondrian trees are a promising basis for more sophisticated yet theoretically sound random forests variants.
\end{abstract}



\end{frontmatter}

\section{Introduction}
\label{sec:introduction}

Introduced by Breiman~\cite{breiman2001randomforests}, \emph{Random Forests} (RF) are state-of-the-art classification and regression algorithms that proceed by averaging the forecasts of a number of randomized decision trees grown in parallel.
Many extensions of RF have been proposed to tackle quantile estimation problems~\cite{meinshausen2006quantile}, survival analysis~\cite{ishwaran2008random} and ranking~\cite{clemenccon2013ranking};
improvements of original RF are provided in literature, to cite but a few, better sampling strategies~\cite{geurts2006extremely}, new splitting methods~\cite{menze2011oblique} or Bayesian alternatives~\cite{chipman2010bart}.
Despite their widespread use and remarkable success in practical applications, the theoretical properties of such algorithms are still not fully understood~(for an overview of theoretical results on RF, see~\cite{biau2016rf_tour}).
As a result of the complexity of the procedure, which combines sampling steps and feature selection, Breiman's original algorithm has proved difficult to analyze. A recent line of research~\cite{scornet2015consistency_rf, wager2015adaptive, mentch2016quantifying, cui2017some, wager2017estimation, athey2019generalized} has sought to obtain some theoretical guarantees for RF variants that  closely resembled the algorithm used in practice.
It should be noted, however, that most of these theoretical guarantees only offer limited information on the quantitative behavior of the algorithm (guidance for parameter tuning is scarce) or come at the price of conjectures on the true behavior of the RF algorithm itself, being thus still far from explaining the excellent empirical performance of it. 

In order to achieve a better understanding of the random forest algorithm, another line of research focuses on modified and stylized versions of RF. 
Among these methods, \emph{Purely Random Forests} (PRF)~\cite{Breiman2000sometheory, biau2008consistency_rf,biau2012analysis_rf,genuer2012variance_purf,arlot2014purf_bias, klusowski2018complete} 
grow the individual trees independently of the sample, and are thus particularly amenable to theoretical analysis.
The consistency of such algorithms (as well as other idealized RF procedures) was first obtained by~\cite{biau2008consistency_rf}, as a byproduct of the consistency of individual tree estimates. 
These results aim at quantifying the performance guarantees by analyzing the bias/variance of simplified versions of RF, such as 
PRF models~\cite{genuer2012variance_purf,arlot2014purf_bias}.
In particular,~\cite{genuer2012variance_purf} shows that some PRF variant achieves the minimax rate for the estimation of a Lipschitz regression function in dimension one.
The bias-variance analysis is extended in~\cite{arlot2014purf_bias}, showing that PRF can also achieve minimax rates for $\cl^2$ regression functions in dimension one.
These results are much more precise than mere consistency, and offer insights on the proper tuning of the procedure.
Quite surprisingly,
these optimal rates are only obtained in the \emph{one-dimensional case} (where decision trees reduce to histograms).
In the multi-dimensional setting, where trees exhibit an intricate recursive structure, only suboptimal rates are derived.
As shown by lower bounds from \cite{klusowski2018complete}, this is not merely a limitation from the analysis: centered forests, a standard variant of PRF, exhibit suboptimal rates under nonparametric assumptions.

From a more practical perspective, an important limitation of the most commonly used RF algorithms, such as Breiman's Random Forests~\cite{breiman2001randomforests} and the Extra-Trees algorithm~\cite{geurts2006extremely}, is that they are typically trained in a batch manner, where the whole dataset, available at once, is required to build the trees.
In order to allow their use in situations where large amounts of data have to be analyzed in a streaming fashion, several online variants of decision trees and RF algorithms have been proposed~\cite{domingos2000hoeffdingtree,saffari2009online-rf,taddy2011dynamictrees,denil2013online,denil2014narrowing}.

Of particular interest in this article is the \emph{Mondrian Forest} (MF) algorithm, an efficient and accurate online random forest classifier introduced by~\cite{lakshminarayanan2014mondrianforests}, see also~\cite{lakshminarayanan2016mondrianuncertainty}. 
This algorithm is based on the Mondrian process~\cite{roy2009mondrianprocess,roy2011phd, orbanz2015exchangeable}, a natural probability distribution on 
the set of recursive partitions of the unit cube $[0,1]^d$.
An appealing property of Mondrian processes is that they can be updated in an online fashion. 
In~\cite{lakshminarayanan2014mondrianforests}, the use of the \emph{conditional Mondrian} process enables the authors
to design an online algorithm which matches its batch counterpart: training the algorithm one data point at a time leads to the same randomized estimator as training the algorithm on the whole dataset at once.
The algorithm proposed in~\cite{lakshminarayanan2014mondrianforests} depends on a lifetime parameter $\lambda > 0$ that guides the complexity of the trees by stopping their building process.
However, a theoretical analysis of MF is lacking, in particular, the tuning of $\lambda$ is unclear from a theoretical perspective. 
In this paper, we show that, aside from their appealing computational properties, Mondrian Forests are amenable to a precise theoretical analysis.
We study MF in a batch setting and provide theoretical guidance on the tuning of $\lambda$.

Based on a detailed analysis of Mondrian partitions, we prove consistency and convergence rates for MF \textit{in arbitrary dimension}, that turn out to be minimax optimal on the set of $s$-H\"older function with $s \in (0, 2]$, assuming that $\lambda$ and the number of trees in the forest (for $s \in (1, 2]$) are properly tuned. 
Furthermore, we construct a procedure that adapts to the unknown smoothness $s \in (0, 2]$ by combining Mondrian Forests with a standard model aggregation algorithm.
To the best of our knowledge, such results have only been proved for very specific purely random forests, where the covariate space is of dimension one~\cite{arlot2014purf_bias}.
Our analysis also sheds light on the benefits of Mondrian Forests compared to single Mondrian Trees: the bias reduction of Mondrian Forests allow them to be minimax for $s \in (1, 2]$, while a single tree fails to be minimax in this case.

\paragraph{Agenda}

This paper is organized as follows.
In Section~\ref{sec:setting-notations}, we describe the considered setting  and set the notations for trees and forests.
Section~\ref{sec:def-mondrian} defines the Mondrian process introduced by~\cite{roy2009mondrianprocess} and describes the MF algorithm.
Section~\ref{sec:properties-mondrian} provides new sharp properties for Mondrian partitions: cells distribution in Proposition~\ref{prop:cell-distribution} and a control of the cells diameter in Corollary~\ref{lem:diameter}, while the expected number of cells is provided in Proposition~\ref{prop:number-cells}.
Building on these properties, we provide, in Section~\ref{sec:minimax-mondrian},  statistical guarantees for MF: Theorem~\ref{thm:consistency-mondrian} proves consistency, while Theorems~\ref{thm:minimax-regression} and~\ref{thm:minimaxc2} provide minimax rates for $s \in (0, 1]$ and $s \in (1, 2]$ respectively.
Finally, Proposition~\ref{prop:adaptive-rate} proves that a combination of MF with a model aggregation algorithm adapts to the unknown smoothness $s \in (0, 2]$.


\section{Setting and notations}
\label{sec:setting-notations}

We first describe the setting of the paper and set the notations related to the Mondrian tree structure.
For the sake of conciseness, we consider the regression setting, and show how to extend the results to classification in Section~\ref{sec:results_for_binary_classification}.

\paragraph{Setting}

We consider a regression framework, where the dataset $\dataset_n = \{ (X_1, Y_1), \ldots, (X_n,Y_n) \}$ consists of \iid $[0,1]^d \times \R$-valued random variables.
We assume throughout the paper that the dataset is distributed as a generic 
pair $(X,Y)$ such that $\E [Y^2] < \infty$.
This unknown distribution, characterized by the distribution $\mu$ of $X$ on $[0,1]^d$ and by the conditional distribution of $Y|X$, can be written as
\begin{equation}
  \label{eq:signal_plus_noise}
  Y = f(X) + \varepsilon,  
\end{equation}
where $f (X) = \E [Y \cond X]$ is the conditional expectation of $Y$ given $X$, and $\varepsilon$ is a noise satisfying $\E [\varepsilon |X] = 0$.
Our goal is to output a \emph{randomized estimate}
$\wh f_n ( \cdot ,Z , \dataset_n): [0,1]^d \to \R$, where $Z$ is a random variable that accounts for the randomization procedure. To simplify notation, we will denote $\wh f_n (x,Z ) = \wh f_n (x, Z,\dataset_n)$. The quality of a randomized estimate $\wh f_n$ is measured by its quadratic risk
\begin{align*}
  \label{eq:error}  
  \risk (\wh f_n) = \E [(\wh f_n(X, Z) - f(X) )^2]  
\end{align*}
where the expectation is taken with respect to $(X,Z,\dataset_n)$.
We say that a sequence $(\wh f_n)_{n \geq 1}$ is \emph{consistent} whenever $\risk (\wh f_n) \to 0$ as $n \to \infty$.

  \paragraph{Trees and Forests}
  A regression tree is a particular type of partitioning estimate.
  First, a recursive partition $\mondrian$ of $[0,1]^d$ is built by performing successive axis-aligned splits (see Section~\ref{sec:def-mondrian}), then the regression tree prediction is computed by averaging the labels $Y_i$ of observations falling in the same cell as the query point $x \in [0,1]^d$, that is
\begin{equation}
\wh f_n(x, \mondrian) = \sum_{i=1}^n \frac{\mathds{1}_{X_i \in C_{\mondrian}(x)}}{N_n(C_{\mondrian}(x))} Y_i, \label{eq:tree_estimate}
\end{equation}
where $C_{\mondrian}(x)$ is the cell of the tree partition containing $x$ and $N_n(C_{\mondrian}(x))$ is the number of observations falling into $C_{\mondrian}(x)$, with the convention that the estimate returns $0$ if the cell $C_{\mondrian}(x)$ is empty.

A random forest estimate is obtained by averaging the predictions of $M$ randomized decision trees; more precisely, we will consider purely random forests, where the randomization of each tree (denoted above by $Z$) comes exclusively from the random partition, which is independent of $\dataset_n$.
Let $\mondrian_M = (\mondrian^{(1)}, \dots, \mondrian^{(M)})$, where $\mondrian^{(m)}$ (for $m=1, \hdots, M$) are \iid random partitions of $[0, 1]^d$. The random forest estimate is thus defined as  
  \begin{equation}
\label{eq_def_RF}
\wh f_{n,M}(x, \mondrian_M) = \frac 1M \sum_{m=1}^M \wh f_n(x,\mondrian^{(m)})  \,, 
\end{equation}
where $\wh f_n(x,\mondrian^{(m)})$ is the prediction, at point $x$, of the tree with random partition $\mondrian^{(m)}$, defined in  (\ref{eq:tree_estimate}).

  
 The Mondrian Forest, whose construction is described below, is a particular instance of \eqref{eq_def_RF}, in which the Mondrian process plays a crucial role by specifying the randomness $\mondrian$ of tree partitions.

\section{The Mondrian Forest algorithm}
\label{sec:def-mondrian}

Given a rectangular box $C = \prod_{j=1}^d [a_j, b_j] \subseteq \R^d$, we denote $|C| := \sum_{j=1}^d (b_j -a_j)$ its \emph{linear dimension}. 
The Mondrian process $\MP (C)$ is a distribution on (infinite) tree partitions of $C$ introduced by~\cite{roy2009mondrianprocess}, see also~\cite{roy2011phd} for a rigorous construction.
Mondrian partitions are built by iteratively splitting cells at some random time, which depends on the linear dimension of the cell; the splitting probability on each side is proportional to the side length of the cell, and the position is drawn uniformly.   

The Mondrian process distribution $\MP (\lambda, C)$ is a distribution on tree partitions of $C$, resulting from the pruning of partitions drawn from $\MP (C)$. The pruning is done by removing all splits occurring after time $\lambda > 0$. In this perspective, $\lambda$ is called the lifetime parameter and controls the complexity of the partition: large values of $\lambda$ corresponds to deep trees (complex partitions).  

Sampling from the distribution $\MP (\lambda, C)$ can be done efficiently by applying the recursive procedure $\SampleMondrian (C, \tau =0, \lambda)$ described in Algorithm~\ref{alg:split-cell}.
Figure~\ref{fig:mondrian} below shows a particular instance of Mondrian partition on a square box, with lifetime parameter $\lambda = 3.4$.
In what follows, $\expdist(\lambda)$ stands for the exponential distribution with intensity $\lambda > 0$.
  
 \begin{algorithm}[htbp]
   \caption{$\SampleMondrian (\cell, \tau, \lambda)$:
     samples a Mondrian partition of $C$, starting from time $\tau$ and until time $\lambda$.
   }
\label{alg:split-cell} 
\begin{algorithmic}[1]
  \STATE \textbf{Inputs:} A cell 
  $\cell = \prod_{1\leq j \leq d} [a_j, b_j]$,
  starting time $\tau$ and lifetime parameter $\lambda$.
  \STATE Sample a random variable 
  $E_\cell \sim \expdist(|\cell|)$
  \IF{$\tau + E_\cell \leq \lambda$}
  \STATE Sample a split dimension $J \in \{ 1,\dots, d \}$, with $\P (J = j) = (b_j - a_j)/ |\cell|$
  \STATE Sample a split threshold $S_J$ uniformly in 
  $[a_J, b_J]$
  \STATE Split $\cell$ along the split $(J, S_J)$:
  let $C_0 = \{ x \in C : x_J \leq S_J \}$ and 
  $C_1 = C \setminus C_0$
  \STATE \textbf{return} $\SampleMondrian (\cell_0, \tau + E_\cell, \lambda) \cup \SampleMondrian (\cell_1, \tau + E_\cell, \lambda)$
  \ELSE
  \STATE
  \textbf{return} $\{ C \}$ (\ie, do not split $C$).
  \ENDIF
\end{algorithmic}
\end{algorithm}

\begin{figure}[h!]
  \centering    
  \begin{tikzpicture}[scale=0.4]

  \pgfmathsetmacro{\TR}{6.5}  
  \pgfmathsetmacro{\Tl}{4.3}  
  \pgfmathsetmacro{\Tr}{2.7}  
  \pgfmathsetmacro{\Tlr}{1.4}  
  \pgfmathsetmacro{\TL}{0.6}  
  \pgfmathsetmacro{\TE}{-0.4}
  
    \coordinate (BG) at (0,0) ;
    \coordinate (BD) at (0,1) ;
    \coordinate (HG) at (1,0) ;
    \coordinate (HD) at (1,1) ;
    \coordinate (A1) at (4,0) ;
    \coordinate (B1) at (4,10) ;
    \coordinate (C1) at (4,5) ;
    \coordinate (A2) at (0,3) ;
    \coordinate (B2) at (4,3) ;
    \coordinate (C2) at (2,3) ;
    \coordinate (A3) at (4,6) ;
    \coordinate (B3) at (10,6) ;
    \coordinate (C3) at (7,6) ;
    \coordinate (A4) at (0,8) ;
    \coordinate (B4) at (4,8) ;
    \coordinate (C4) at (2,8) ;

    
    \draw [thick] (0,0) rectangle (10,10) ; 
    \draw [thick,fill=lightblue!30] (0,0) rectangle (B2) ;
    \draw [thick,fill=yellow!30] (A4) rectangle (B2) ;
    \draw [thick,fill=red!30] (B1) rectangle (A4) ;
    \draw [thick,fill=green!30] (B1) rectangle (B3) ;
    \draw [thick,fill=magenta!30] (A1) rectangle (B3) ;
    
    \draw (A1) -- (B1) ; 
    \draw (A2) -- (B2) ; 
    \draw (A3) -- (B3) ;
    \draw (A4) -- (B4) ;
    \draw [dotted, fill=gray!0] (C1) circle (0.6) node {${\scriptstyle 1.3}$} ; 
    \draw [dotted, fill=gray!0] (C2) circle (0.6) node {${\scriptstyle 2.3}$} ;
    \draw [dotted,fill=gray!0] (C3) circle (0.6) node {${\scriptstyle 2.7}$} ;
    \draw [dotted,fill=gray!0] (C4) circle (0.6) node {${\scriptstyle 3.2}$} ;
    
    \coordinate (R) at (18,\TR 
    ) ;
    \coordinate (N0) at (15,\Tl 
    ) ;
    \coordinate (N1) at (21,\Tr 
    ) ;
    \coordinate (N00) at (13.5,\TL 
    ) ;
    \coordinate (N01) at (16.5,\Tlr 
    ) ;
    \coordinate (N10) at (19.5,\TL 
    ) ;
    \coordinate (N11) at (22.5,\TL 
    ) ;
    \coordinate (N010) at (15.5,\TL 
    ) ;
    \coordinate (N011) at (17.5,\TL 
    ) ;

    \draw (18,10) node {$\bullet$} -- (R) ; 
    \draw (15,\TR) -- (21, \TR) ;
    \draw (15,\TR) -- (N0) ;
    \draw (21,\TR) -- (N1) node {$\circ$} ;
    \draw (13.5,\Tl) -- (16.5, \Tl) ;
    \draw (13.5,\Tl) -- (N00) ;
    \draw (16.5,\Tl) -- (N01) node {$\circ$} ;
    \draw (15.5,\Tlr)  -- (17.5, \Tlr) ;
    \draw (15.5,\Tlr) -- (N010) ;
    \draw (17.5,\Tlr) -- (N011) ;
    \draw (19.5,\Tr) -- (22.5, \Tr) ;
    \draw (19.5,\Tr) -- (N10) ;
    \draw (22.5,\Tr) -- (N11) ;

    \coordinate (T) at (26,10) ; 
    \coordinate (TR) at (26,\TR) ; 
    \coordinate (T0) at (26,\Tl) ; 
    \coordinate (T1) at (26,\Tr) ; 
    \coordinate (T01) at (26,\Tlr) ; 
    \coordinate (TL) at (26,\TL) ;
    \coordinate (TE) at (26,\TE ) ;

    \draw[>=stealth,->] (T) -- (TE) ;
    \draw (T) node {$-$} ;
    \draw (TR) node {$-$} ;
    \draw (T0) node {$-$} ;
    \draw (T1) node {$-$} ;
    \draw (T01) node {$-$} ;
    \draw (TL) node {$-$} ;

    \draw (T) node[right] {$
      {\scriptstyle
        0
      }
      $} ;
    \draw (TR) node[right] {${\scriptstyle 1.3}$} ;
    \draw (T0) node[right] {${\scriptstyle 2.3}$} ;
    \draw (T1) node[right] {${\scriptstyle 2.7}$} ;
    \draw (T01) node[right] {${\scriptstyle 3.2}$} ;
    \draw (TL) node[right] {${\scriptstyle \lambda = 3.4}$} ;
    \draw (TE) node[right] {{\small time}} ;    
    
    \draw [dotted] (18,10) -- (T) ;
    \draw [dotted] (R) -- (TR) ;
    \draw [dotted] (N0) -- (T0) ;
    \draw [dotted] (N1) -- (T1) ;
    \draw [dotted] (N01) -- (T01) ;
    \draw [dashed] (N00) -- (TL) ; 

    \draw [fill=gray!0] (R) circle (0.16) ;
    \draw [fill=gray!0] (N0) circle (0.16) ;
    \draw [fill=gray!0] (N1) circle (0.16) ;
    \draw [fill=gray!0] (N01) circle (0.16) ;        

    \draw [fill=lightblue!30] (N00) circle (0.25) ;
    \draw [fill=magenta!30] (N10) circle (0.25) ;
    \draw [fill=green!30] (N11) circle (0.25) ;
    \draw [fill=yellow!30] (N010) circle (0.25) ;
    \draw [fill=orange!30] (N011) circle (0.25) ;
    

  \end{tikzpicture}

  \caption{A Mondrian partition (left) with corresponding tree structure (right), which shows the evolution of the tree over time.
      The split times are indicated on the vertical axis, while the splits are denoted with bullets ($\circ$).
  }  
  \label{fig:mondrian}
\end{figure}
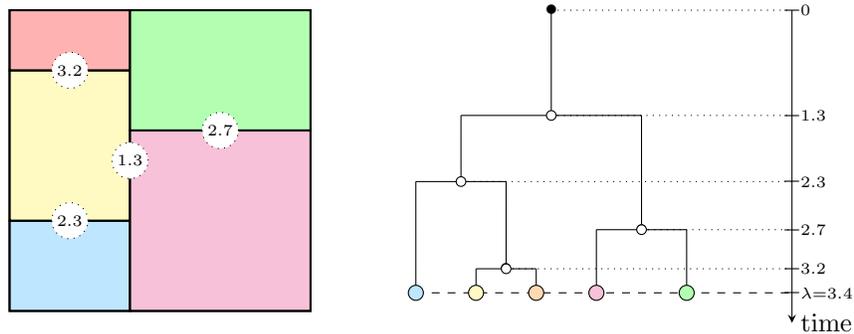

\begin{remark}
  Using the fact that $\expdist$ is memoryless (if $E \sim \expdist(\lambda)$ and $u > 0$ then $E - u | E > u \sim \expdist(\lambda)$), it is possible to efficiently sample $\mondrian_{\lambda'} \sim \MP (\lambda', C)$ given its pruning $\mondrian_\lambda \sim \MP (\lambda, C)$ at time $\lambda \leq \lambda'$.
\end{remark}

A Mondrian Tree estimator is given by Equation~\eqref{eq:tree_estimate} where the partition $\Pi^{(m)}$ is sampled from the distribution $\MP (\lambda, [0,1]^d)$. 
The Mondrian Forest grows randomized tree partitions $\mondrian_{\lambda}^{(1)}, \dots, \mondrian_\lambda^{(M)}$, fits each one with the dataset $\dataset_n$ by averaging the labels falling into each leaf, then combines the resulting Mondrian Tree estimates by averaging their predictions.
In accordance with Equation~\eqref{eq_def_RF}, we let 
\begin{equation}
  \label{eq:mondrian-forest-estimator}
  \wh f_{\lambda,n, M}(x, \mondrian_{\lambda, M}) = \frac 1M \sum_{m=1}^M \wh f_{\lambda,n}^{(m)}(x, \mondrian_\lambda^{(m)})
\end{equation}
be the Mondrian Forest estimate described above, where $\wh f_{\lambda,n}^{(m)}(x, \mondrian_\lambda^{(m)})$ denotes the Mondrian Tree based on the random partition $\mondrian_\lambda^{(m)}$ and $ \mondrian_{\lambda, M} = (\mondrian_{\lambda}^{(1)}, \dots, \mondrian_\lambda^{(M)})$. To ease notation, we will write $\wh f_{\lambda,n}^{(m)}(x)$ instead of $\wh f_{\lambda,n}^{(m)}(x, $ $ \mondrian_\lambda^{(m)})$.
Although we use the standard definition of Mondrian processes, the way we compute the prediction in a Mondrian Tree differs from the original one.
Indeed, in~\cite{lakshminarayanan2014mondrianforests}, prediction is given by the expectation over a posterior distribution, where a hierarchical prior is assumed on the label distribution of each cell of the tree.
In this paper, we simply compute the average of the observations falling into a given cell.

\section{Local and global properties of the Mondrian process}
\label{sec:properties-mondrian}

In this Section, we show that the properties of the Mondrian process enable us to compute explicitly some local and global quantities related to the structure of Mondrian partitions. To do so, we will need the following two facts, exposed by~\cite{roy2009mondrianprocess}.

\begin{fact}[Dimension $1$]
  \label{fac:mondrian-poisson}
  For $d = 1$, the splits from a Mondrian process $\mondrian_\lambda \sim \MP (\lambda, [0,1])$ form a subset of $[0, 1]$, which is distributed as a Poisson point process of intensity $\lambda \di x$.
\end{fact}

\begin{fact}[Restriction]
  \label{fac:mondrian-restriction}
  Let $\mondrian_\lambda \sim \MP (\lambda, [0,1]^d)$ be a Mondrian partition, and $C = \prod_{j=1}^d [a_j, b_j] \subset [0,1]^d$ be a box. Consider the \emph{restriction} $\mondrian_\lambda|_C$ of $\mondrian_\lambda$ on $C$, \ie the partition on $C$ induced by the partition $\mondrian_\lambda$ of $[0,1]^d$. Then $\mondrian_\lambda|_C \sim \MP (\lambda, C)$.
\end{fact}

Fact~\ref{fac:mondrian-poisson} deals with the one-dimensional case by making explicit the distribution of splits for Mondrian process, which follows a Poisson point process. 
The restriction property stated in Fact~\ref{fac:mondrian-restriction} is fundamental, and enables one to precisely characterize the behavior of the Mondrian partitions.

Given any point $x \in [0, 1]^d$, Proposition~\ref{prop:cell-distribution} below is a sharp result giving the exact distribution of the cell $\cell_\lambda (x)$ containing $x$ from the Mondrian partition.
Such a characterization is typically unavailable for other randomized trees partitions involving a complex recursive structure.

\begin{proposition}[Cell distribution]
  \label{prop:cell-distribution}  
  Let $x \in [0, 1]^d$ and denote by  
  \begin{equation*}
    \cell_\lambda (x) = \prod_{1\leq j \leq d} [L_{j, \lambda} (x), R_{j, \lambda} (x)]
  \end{equation*}
  the cell containing $x$ in a partition $\mondrian_\lambda \sim \MP(\lambda, [0, 1]^d)$ \textup(this cell corresponds to a leaf\textup).
  Then, the distribution of $\cell_\lambda (x)$ is characterized by the following properties\textup:
  \begin{enumerate}
  \item[$(i)$] $L_{1, \lambda} (x), R_{1, \lambda} (x), \dots, L_{d, \lambda} (x), R_{d, \lambda} (x)$ are independent\textup;
  \item[$(ii)$] For each $j = 1, \dots, d$, $L_{j, \lambda} (x)$ is distributed as $(x - \lambda^{-1} \expleft{j}) \vee 0$ and $R_{j, \lambda} (x)$ as $ (x + \lambda^{-1} \expright{j}) \wedge 1$, where $\expleft{j},\expright{j} \sim \expdist (1)$.
  \end{enumerate}
\end{proposition}

\begin{figure}[h!]
  \centering
  \begin{tikzpicture}[scale=0.5, every node/.style={scale=0.82}]  

      \pgfmathsetmacro{\xa}{9} 
      \pgfmathsetmacro{\xb}{4}
      \pgfmathsetmacro{\Rl}{0} 
      \pgfmathsetmacro{\Rr}{15}
      \pgfmathsetmacro{\Rd}{0}
      \pgfmathsetmacro{\Ru}{10}
      \pgfmathsetmacro{\rl}{5} 
      \pgfmathsetmacro{\rr}{12}
      \pgfmathsetmacro{\rd}{2}
      \pgfmathsetmacro{\ru}{7}      

      \coordinate (X) at (\xa, \xb) ;
      \coordinate (CLD) at (\rl, \rd) ;
      \coordinate (CLU) at (\rl, \ru) ;
      \coordinate (CRD) at (\rr, \rd) ;
      \coordinate (CRU) at (\rr, \ru) ;

      \fill [black!3] (\Rl,\Rd) rectangle (\Rr,\Ru) ;
      
      \draw [line width=1pt, black!25] (\rl,\Rd) -- (\rl, \Ru) ;
      \draw [line width=1pt, black!25] (\rl,\ru) -- (\Rr, \ru) ;
      \draw [line width=1pt, black!25] (\rr,\ru) -- (\rr, \Rd) ;
      \draw [line width=1pt, black!25] (\Rl, 4.9) -- (\rl, 4.9) ;

      \fill [yellow!10] (CLD) rectangle (CRU) ;

      \draw [darkblue!60, line width=0.8pt, dashed] (\rl, \xb) -- (X) ;
      \draw [darkblue!60, line width=0.8pt, dashed] (\rr, \xb) -- (X) ;
      \draw [darkblue!60, line width=0.8pt, dashed] (\xa, \rd) -- (X) ;
      \draw [darkblue!60, line width=0.8pt, dashed] (\xa, \ru) -- (X) ;      

      \draw [decorate, decoration={brace,raise=1.5pt,amplitude=4pt}, darkblue] (\rl+0.1, \xb) -- (\xa-0.1,\xb) ;
      \draw [decorate, decoration={brace,raise=1.5pt,amplitude=4pt}, darkblue] (\rr-0.08, \xb) -- (\xa+0.1,\xb) ;
      \draw [decorate, decoration={brace,raise=1.5pt,amplitude=4pt}, darkblue] (\xa, \rd+0.08) -- (\xa,\xb-0.1) ;
      \draw [decorate, decoration={brace,raise=1.2pt,amplitude=4pt}, darkblue] (\xa, \ru-0.08) -- (\xa,\xb+0.1) ;

      \draw [darkblue] (0.5*\rl+0.5*\xa, \xb+0.68) node {${\scriptstyle \lambda^{-1} \expleft{1}}$} ;
      \draw [darkblue] (0.5*\rr+0.5*\xa, \xb-0.75) node {${\scriptstyle \lambda^{-1} \expright{1}}$} ;
      \draw [darkblue] (\xa-1.37, 0.5*\rd+0.5*\xb) node {${\scriptstyle \lambda^{-1} \expleft{2}}$} ;
      \draw [darkblue] (\xa+1.42, 0.5*\ru+0.5*\xb) node {${\scriptstyle \lambda^{-1} \expright{2}}$} ;

      \draw [line width=1pt,darkblue] (CLU) -- (CLD) ;
      \draw [line width=1pt,darkblue] (CRD) -- (CRU) ;
      \draw [line width=1pt,darkblue] (CLD) -- (CRD) ;
      \draw [line width=1pt,darkblue] (CLU) -- (CRU) ;

      \draw [line width=1.1pt] (\Rl,\Rd) rectangle (\Rr,\Ru) ;

      \draw [<-,line width=0.5pt] (\rl+1,\ru-1) arc (50:80:4) ;
      \draw (\rl-3, \ru-0.6) node[above right] {$\cell_\lambda(x)$} ;
      
      \draw [fill=black] (X) circle (0.1) ;
      \draw (\xa+0.15, \xb) node[above right] {$x$} ;
      
\end{tikzpicture}

  \caption{Cell distribution in a Mondrian partition (Proposition~\ref{prop:cell-distribution}).
  }
  \label{fig:cell-distribution}  
\end{figure}
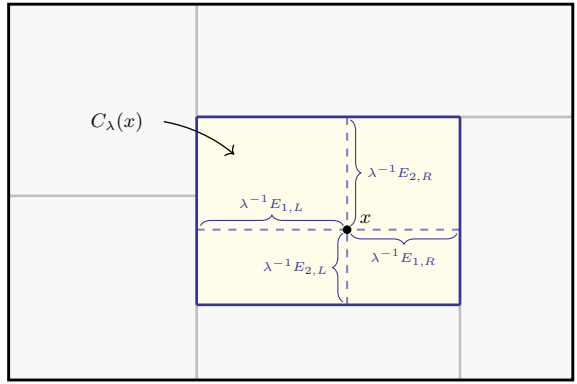

The proof of Proposition~\ref{prop:cell-distribution} is given in Section~\ref{sec:proofs}.
Figure~\ref{fig:cell-distribution} is a graphical representation of Proposition~\ref{prop:cell-distribution}. 
A consequence of Proposition~\ref{prop:cell-distribution} is the next Corollary~\ref{lem:diameter}, which gives a precise upper bound on the diameter of the cells. 
In particular, this result is used in the proofs of the theoretical guarantees for Mondrian Trees and Forests from Section~\ref{sec:minimax-mondrian} below.

\begin{corollary}[Cell diameter]
  \label{lem:diameter}
  Set $\lambda>0$ and $\mondrian_\lambda \sim \MP (\lambda, [0, 1]^d)$ be a Mondrian partition. 
  Let $x \in [0,1]^d$ and let $D_\lambda (x)$ be the $\ell^2$-diameter of the cell $\cell_\lambda (x)$ containing $x$ in $\mondrian_\lambda$. 
  For every $\delta >0$\textup, we have
  \begin{equation}    
    \label{eq:diameter-bound}    
    \P ( D_\lambda (x) \geq \delta )    
    \leq d \Big( 1 + \frac{\lambda \delta}{\sqrt{d}} \Big)
    \exp \Big( - \frac{\lambda \delta}{\sqrt{d}} \Big)    
  \end{equation}  
  and  
  \begin{equation}    
    \label{eq:diameter-square}    
    \E \big[ D_{\lambda} (x)^2 \big] \leq \frac{4d}{\lambda^2} \, .   
  \end{equation}  
\end{corollary}

In order to control the risk of Mondrian Trees and Forests, we need an upper bound on the number of cells in a Mondrian partition.
Quite surprisingly, the expectation of this quantity can be computed exactly, as shown in Proposition~\ref{prop:number-cells}.

\begin{proposition}[Number of cells]
  \label{prop:number-cells}
  Set $\lambda>0$ and $\mondrian_\lambda \sim \MP (\lambda, [0, 1]^d)$ be a Mondrian partition.
  If $K_\lambda$ denotes the number of cells in $\mondrian_\lambda$\textup,
  we have $\E [K_\lambda] = (1 + \lambda)^d$.
\end{proposition}

The proof of Proposition~\ref{prop:number-cells} is given in the Supplementary Material, while a sketch of proof is provided in Section~\ref{sec:proofs}.
Although the proof is technically involved, it relies on a natural coupling argument: we introduce a recursive modification of the construction of the Mondrian process which keeps the expected number of leaves unchanged, and for which this quantity can be computed directly using the Mondrian-Poisson equivalence in dimension one (Fact~\ref{fac:mondrian-poisson}).
A much simpler result is $\E [K_\lambda] \leq (e(1 + \lambda))^d$, which was previously obtained in~\cite{mourtada2017mondrian}. 
By contrast, Proposition~\ref{prop:number-cells} provides the \emph{exact} value of this expectation, which removes a superfluous $e^d$ factor.

\begin{remark}
  Proposition~\ref{prop:number-cells} naturally extends (with the same proof) to the more general case of a Mondrian process with finite measures with no atoms $\nu_1, \dots, \nu_d$ on the sides $C^1, \dots, C^d$ of a box $C\subseteq \R^d$ (for a definition of the Mondrian process in this more general case, see~\cite{roy2011phd}). In this case, we have $\E \left[ K_\lambda \right] = \prod_{1\leq j \leq d} (1+ \nu_j (C^j))$.  
\end{remark}

As illustrated in this Section, a remarkable fact with the Mondrian Forest is that the quantities of interest for the statistical analysis of the algorithm can be made explicit.
In particular, we have seen in this Section that, roughly speaking, a Mondrian partition is balanced enough so that it contains $O(\lambda^d)$ cells of diameter $O (1/\lambda)$, which is the minimal number of cells to cover $[0, 1]^d$.

\section{Minimax theory for Mondrian Forests}
\label{sec:minimax-mondrian}

This Section gathers several theoretical guarantees for Mondrian Trees and Forests.
Section~\ref{sec:consistency} states the universal consistency of the procedure, provided that the lifetime $\lambda_n$ belongs to an appropriate range.
We provide convergence rates which turn out to be minimax optimal for $s$-H\"older regression functions with $s \in (0, 1]$ in Section~\ref{sec:minimax} and with $s \in (1, 2]$ in Section~\ref{sec:mf-bias}, provided in both cases that $\lambda_n$ is properly tuned.
Note that in particular, we illustrate in Section~\ref{sec:mf-bias} the fact that Mondrian Forests improve over Mondrian trees, when $s \in (1, 2]$.
In Section~\ref{sec:adapt-smoothn-class}, we prove that a combination of MF with a model aggregation algorithm adapts to the unknown $s \in (0, 2]$.
Finally, results for classification are given in Section~\ref{sec:results_for_binary_classification}.

\subsection{Consistency of Mondrian Forests}
\label{sec:consistency}

The consistency of the Mondrian Forest estimator is established in  Theorem~\ref{thm:consistency-mondrian} below, assuming a proper tuning of the lifetime parameter $\lambda_n$.

\begin{theorem}[Universal consistency]
  \label{thm:consistency-mondrian}
  Let $M \geq 1$. 
  Consider Mondrian Trees $\wh f_{\lambda_n, n}^{(m)}$ \textup(for $m = 1, \hdots, M$\textup) and Mondrian Forest 
  $\wh f_{\lambda_n, n, M}$ given by Equation~\eqref{eq:mondrian-forest-estimator} for a sequence $(\lambda_n)_{n \geq 1}$ satisfying $\lambda_n \to \infty$ and ${\lambda_n^d}/{n} \to 0$.
  Then\textup, under the setting described in Section~\ref{sec:setting-notations} above\textup, the individual trees 
  $\wh f_{\lambda_n, n}^{(m)}$ \textup(for $m=1, \hdots, M$\textup) are consistent\textup, and as a consequence\textup, the forest 
  $\wh f_{\lambda_n,n,  M}$ is consistent for any $M \geq 1$.
\end{theorem}

The proof of Theorem~\ref{thm:consistency-mondrian} is given in the Supplementary Material.
It uses the properties of Mondrian partitions established in Section~\ref{sec:properties-mondrian} together with general consistency results for histograms.
This result is universal, in the sense that it makes no assumption on the joint distribution of $(X,Y)$, apart from $\E[Y^2] < \infty$ in order to ensure that the quadratic risk is well-defined (see Section~\ref{sec:setting-notations}).

The only tuning parameter of a Mondrian Tree is the lifetime $\lambda_n$, which encodes the complexity of the trees.
Requiring an assumption on this parameter is natural, and confirmed by the well-known fact that the tree-depth is an important tuning parameter for Random Forests, see~\cite{biau2016rf_tour}.
However, Theorem~\ref{thm:consistency-mondrian} leaves open the question of a theoretically optimal tuning of $\lambda_n$ under additional assumptions on the regression function $f$, which we address in what follows.

\subsection{Mondrian Trees and Forests are minimax for $s$-H\"older functions with $s \in (0, 1]$}
\label{sec:minimax}

The bounds obtained in Corollary~\ref{lem:diameter} and Proposition~\ref{prop:number-cells} are explicit and sharp in their dependency on~$\lambda$. 
 Based on these properties, we now establish a theoretical upper bound on the risk of Mondrian Trees, which gives the optimal theoretical tuning of the lifetime parameter~$\lambda_n$. 
 To pursue the analysis, we need the following assumption. 

 \begin{assumption}
  \label{ass:signal_plus_noise}
  Consider $(X, Y)$ from the setting described in Section~\ref{sec:setting-notations} and assume also that
  $\E [\eps \cond X] = 0$ and $\Var (\eps \cond X) \leq \sigma^2 < \infty$ almost surely, where $\eps$ is given by Equation~\eqref{eq:signal_plus_noise}.
\end{assumption}

Our minimax results hold for a class of $s$-H\"older regression functions defined below.
  \begin{definition}
    \label{def:holder-class}
    Let $p \in \N$, $\beta \in (0, 1]$ and $L > 0$.
    The $(p,\beta)$-H\"older ball of norm $L$, denoted $\cl^{p, \beta} (L) = \cl^{p, \beta} ([0, 1]^d, L)$, 
    is the set of $p$ times differentiable functions $f : [0, 1]^d \to \R$ such that
    \begin{equation*}
    	\| \nabla^p f (x) - \nabla^p f (x') \| \leq L \| x - x' \|^{\beta} \quad \text{ and } \quad
    	\| \nabla^k f(x) \| \leq L
    \end{equation*}
    for every $x, x' \in [0, 1]^d$ and $k \in \{ 1, \ldots, p\}$. Whenever $f \in \cl^{p, \beta}(L)$, we 
    say that $f$ is $s$-H\"older with $s = p + \beta$.
  \end{definition}

Note that in what follows we will assume $s \in (0, 2]$, so that $p \in \{ 0, 1\}$.
Theorem~\ref{thm:minimax-regression} below states an upper bound on the risk of Mondrian Trees and Forests, which explicitly depends on the lifetime parameter $\lambda$. 
Selecting $\lambda$ that minimizes this bound leads to a convergence rate which turns out to be minimax optimal over the class of $s$-H\"older functions for $s \in (0,1]$ (see for instance~\cite{stone1982optimal}, Chapter I.3 in~\cite{nemirovski2000nonparametric} or Theorem~3.2 in~\cite{gyorfi2002nonparametric}).
  
\begin{theorem}
  \label{thm:minimax-regression}
  Grant Assumption~\ref{ass:signal_plus_noise} and assume that
 $f \in \cl^{0, \beta} (L)$, where $\beta \in (0, 1]$ and $L > 0$.
  Let $M \geq 1$. The quadratic risk of the Mondrian Forest $\wh f_{\lambda,n,M}$ with 
  lifetime parameter $\lambda>0$ satisfies
  \begin{equation}
    \label{eq:risk-regression}
    \E \big[ (\wh f_{\lambda,n,M} (X) - f(X))^2 \big]
    \leq \frac{(4 d)^\beta L^2}{\lambda^{2\beta}} + \frac{(1+\lambda)^d}{n} 
    \big( 2 \sigma^2 + 9 \| f \|_{\infty}^2 \big).
  \end{equation}
  In particular\textup, as $n \to \infty$\textup, the choice
   $\lambda := \lambda_n \asymp L^{2/(d+2\beta)} n^{1/(d+2\beta)}$
  gives
  \begin{align}
    \label{eq:minimax-eta-rates}
    \E \big[( \wh f_{\lambda_n,n,M} (X) - f (X) )^2\big]
    &= O ( L^{2d/(d+2\beta)} n^{-2 \beta/(d+2\beta)} ),
  \end{align}
  which corresponds to the minimax rate over the class 
  $\cl^{0, \beta} (L)$.
\end{theorem}

The proof of Theorem~\ref{thm:minimax-regression} is given in Section~\ref{sec:proofs}. 
It relies on the properties about Mondrian partitions stated in 
Section~\ref{sec:properties-mondrian}.
Namely, Corollary~\ref{lem:diameter} allows to control the bias of Mondrian Trees (first term on the right-hand side of Equation~\ref{eq:risk-regression}), while Proposition~\ref{prop:number-cells} helps in controlling the variance of Mondrian Trees (second term on the right-hand side of Equation~\ref{eq:risk-regression}).
 
To the best of our knowledge, Theorem~\ref{thm:minimax-regression} is the first to prove that a purely random forest (Mondrian Forest in this case) can be minimax optimal \textit{in arbitrary dimension}.
Minimax optimal upper bounds are obtained for $d=1$ in~\cite{genuer2012variance_purf} and~\cite{arlot2014purf_bias} for models of purely random forests such as Toy-PRF (where the individual partitions correspond to random shifts of the regular partition of $[0, 1]$ in $k$ intervals) and PURF (Purely Uniformly Random Forests, where the partitions are obtained by drawing $k$ random thresholds uniformly in $[0, 1]$).
However, for $d=1$, tree partitions reduce to partitions of $[0, 1]$ in intervals, and do not possess the recursive structure that appears in higher dimensions, which makes their analysis challenging. 
For this reason, the analysis of purely random forests for $d > 1$ has typically produced sub-optimal results: for example,~\cite{biau2012analysis_rf} exhibit an upper bound on the risk of the centered random forests (a particular instance of PRF) which turns out to be much slower than the minimax rate for Lipschitz regression functions. 
A more in-depth analysis of the same random forest model in~\cite{klusowski2018complete} exhibits a new upper and lower bound of the risk, which is still slower than minimax rates for Lipschitz functions.
A similar result was proved by~\cite{arlot2014purf_bias}, who studied the BPRF (Balanced Purely Random Forests algorithm, where all leaves are split, so that the resulting tree is complete), and obtained suboptimal rates.
In our approach, the convenient properties of the Mondrian process enable us to bypass the inherent difficulties met in previous attempts.
One specificity of Mondrian forests compared to other PRF variants is that the largest sides of cells are more likely to be split.
By contrast, variants of PRF (such as centered forests) where the coordinate of the split is chosen with equal probability, may give rise to unbalanced cells with large diameter.

Theorem~\ref{thm:minimax-regression} provides theoretical guidance on the choice of the lifetime parameter, and suggests to set $\lambda := \lambda_n \asymp n^{1/(d+2)}$. Such an insight cannot be gleaned from an analysis that focuses on consistency alone.
Theorem~\ref{thm:minimax-regression} is valid for  Mondrian Forests with any number of trees, and thus in particular for a Mondrian Tree (this is also true for Theorem~\ref{thm:consistency-mondrian}). However, it is a well-known fact that forests outperform single trees in practice~\cite{delgado2014classifiers}. 
Section~\ref{sec:mf-bias} proposes an explanation for this phenomenon, by assuming $f \in \cl^{1, \beta} (L)$.

\subsection{Improved rates for Mondrian Forests compared to a Mondrian Tree}    
\label{sec:mf-bias}

The convergence rate stated in Theorem~\ref{thm:minimax-regression} for $f \in \cl^{0, \beta}(L)$ is valid for both trees and forests, and the risk bound does not depend on the number $M$ of trees that compose the forest. In practice, however, forests exhibit much better performances than individual trees. 
In this Section, we provide a result that illustrates the benefits of forests over trees by assuming that $f \in \cl^{1, \beta}(L)$.
As the counterexample in Proposition~\ref{prop:lower_bound_tree} below shows, single Mondrian trees do not benefit from this additional smoothness assumption, and achieve the same rate as in the Lipschitz case.
This comes from the fact that the bias of trees is highly sub-optimal for such functions.

\begin{proposition}
  \label{prop:lower_bound_tree}  
  Assume that $Y = f(X) + \varepsilon$ with $f(x) = 1 + x$\textup, where $X \sim 
  \uniformdist([0, 1])$ and $\eps$ is independent of $X$ with variance $\sigma^2$.
  Consider a single Mondrian Tree estimate $\wh f_{\lambda,n}^{(1)}$. Then, there exists a constant $C_0 > 0$ such that
  \begin{equation*}    
    \inf_{\lambda \in \R_+^*} \E \big[(\wh f_{\lambda, n}^{(1)}(X) - f (X) )^2 \big]    
    \geq C_0 \wedge \frac{1}{4} \Big( \frac{3 \sigma^2}{n} \Big)^{2/3}
  \end{equation*}  
  for any $n \geq 18$.
\end{proposition}

The proof of Proposition~\ref{prop:lower_bound_tree} is given in the Supplementary Material.
Since the minimax rate over $\cl^{1, 1}$ in dimension $1$ is $O(n^{-4/5})$, Proposition~\ref{prop:lower_bound_tree} proves that a single Mondrian Tree is not minimax optimal over this set of functions.
However, it turns out that large enough Mondrian Forests, which average Mondrian trees, are minimax optimal over $\cl^{1, 1}$.
Therefore, Theorem~\ref{thm:minimaxc2} below highlights the benefits of a forest compared to a single tree.

\begin{theorem}
  \label{thm:minimaxc2}  
  Grant Assumption~\ref{ass:signal_plus_noise} and assume that 
  $f \in \cl^{1,\beta} (L)$\textup, with $\beta \in (0, 1]$ and $L > 0$.
  In addition\textup, assume that $X$ has a positive and $C_p$-Lipschitz density $p$ w.r.t the Lebesgue measure on $[0, 1]^d$.
  Let $\wh f_{\lambda, n, M}$ be the Mondrian Forest estimate given by~\eqref{eq:mondrian-forest-estimator}. Set $\eps \in (0, 1/2)$ and $B_\eps = [\eps, 1 - \eps]^d$. Then\textup, we have
  \begin{align}
    \label{minimax-c2-rate}
    &\E \big[ (\wh f_{\lambda, n, M} (X) - f(X))^2 | X \in B_\eps \big]
      \leq \frac{2 (1+\lambda)^d}{n} \frac{2 \sigma^2 + 9 \| f \|_{\infty}^2}{\infp (1-2\eps)^{d}} \nonumber \\
    &+ \frac{144 L^2 d \supp}{\infp (1-2\eps)^{d}} \frac{e^{-\lambda \eps}}{\lambda^3} + \frac{72 L^2 d^3}{\lambda^4} \Big( \frac{\supp C_p}{\infp^2}\Big)^2 + \frac{16 L^2 d^{1+\beta}}{\lambda^{2(1+\beta)}} \Big(\frac{\supp}{\infp} \Big)^2
      + \frac{8 d L^2}{M \lambda^2} ,
  \end{align}
  where
  $\infp = \inf_{x \in [0, 1]^d} p(x)$ and $\supp = \sup_{x \in [0, 1]^d} p(x)$.
  In particular\textup, letting $s = 1 + \beta$\textup, the choices 
  \begin{align*}
\lambda_n \asymp L^{2/(d+2s)} n^{1/(d+2s)} \quad \textrm{and} \quad M_n \gtrsim L^{4\beta/(d+2s)} n^{2\beta/(d+2s)}
  \end{align*}
  give 
  \begin{equation}    
    \label{eq:minimax_C2_no_bound}
    \E \big[(\wh f_{\lambda_n, n,M_n}(X) - f (X) )^2 \cond X \in B_\eps \big]
    = O (L^{2d/(d+2s)} n^{-2s/(d+2s)}),
  \end{equation} 
  which corresponds to the minimax risk over the class $\cl^{1,\beta} (L)$.

In the case where $\eps = 0$\textup, which corresponds to integrating over the whole hypercube\textup, 
the bound~\eqref{eq:minimax_C2_no_bound} holds if $2s \leq 3$. On the other hand\textup, 
if $2s > 3$\textup, letting 
\begin{align*}
\lambda_n \asymp L^{2/(d+3)} n^{1/(d+3)} \quad \textrm{and} \quad  M_n \gtrsim L^{4/(d+3)} n^{2/(d+3)}
\end{align*}
yields the following upper bound on the integrated risk of the Mondrian Forest estimate over $B_0$
  \begin{equation}    
    \label{minimax-c2-rate-with-boundary}    
    \E \big[(\wh f_{\lambda_n, n, M_n} (X) - f (X) )^2 \big]
    = O ( L^{2d/(d+3)} n^{-3/(d+3)}).    
  \end{equation}  
\end{theorem}

The proof of Theorem~\ref{thm:minimaxc2} is given in Section~\ref{sec:proofs} below. 
It relies on an improved control of the bias, compared to the one used in Theorem~\ref{thm:minimax-regression} in the Lipschitz case:
it exploits the knowledge of the distribution of the cell $C_\lambda(x)$ given in Proposition~\ref{prop:cell-distribution} instead of merely the cell diameter given in Corollary~\ref{lem:diameter} (which was enough for Theorem~\ref{thm:minimax-regression}).
The improved rate for Mondrian Forests compared to Mondrian Trees comes from the fact that large enough forests have a smaller bias than single trees for smooth regression functions.
This corresponds to the fact that averaging randomized trees tends to smooth the decision function of single trees,
which are discontinuous piecewise constant functions that approximate smooth functions sub-optimally.
Such an effect
was already noticed by~\cite{arlot2014purf_bias} for purely random forests.

\begin{remark}
  While Equation~\eqref{eq:minimax_C2_no_bound} gives the minimax
  rate for $\cl^{1,1}$
  functions, it suffers from an unavoidable standard artifact, namely a boundary effect which impacts local averaging estimates, such as kernel estimators \cite{wasserman2006nonparametric,arlot2014purf_bias}.
  It is however possible to set $\eps = 0$ in~\eqref{minimax-c2-rate}, which leads to the sub-optimal rate stated in~\eqref{minimax-c2-rate-with-boundary}.
\end{remark}

\subsection{Adaptation to the smoothness}
\label{sec:adapt-smoothn-class}

The minimax rates of Theorems~\ref{thm:minimax-regression} and~\ref{thm:minimaxc2} for trees and forests
are achieved through a specific tuning of the lifetime parameter $\lambda$, which depends on the considered smoothness class $\cl^{p,\beta} (L)$ through $s = p + \beta$ and $L > 0$, while
on the other hand, the number of trees $M$ simply needs to be large enough in the statement of 
Theorem~\ref{thm:minimaxc2}.
Since in practice such smoothness parameters are unknown, it is of interest to obtain a single method that \emph{adapts} to them.

In order to achieve this, we adopt a standard approach based on model aggregation~\cite{nemirovski2000nonparametric}.
More specifically, we split the dataset into two part: the first is used to fit Mondrian Forest estimators with 
$\lambda$ varying in an exponential grid, while the second part is used to fit the STAR procedure for model aggregation, introduced in~\cite{audibert2008deviation}.
The appeals of this aggregation procedure are its simplicity, its optimal guarantee
and the lack of parameter to tune.

Let $n_0 = \lfloor{n/2}\rfloor$, $\dataset_{n_0} = \{ (X_i, Y_i) : 1\leq i \leq n_0 \}$ and $\dataset_{n_0+1 : n} = \{ (X_i, Y_i) : n_0+1 \leq i \leq n \}$.
Also, let $I_\eps = \{ i \in \{ n_0+1 , \dots, n \} : X_i \in [\eps, 1 - \eps]^d \}$ for some $\eps \in (0, 1/2)$.
If $I_\eps$ is empty, we let the estimator be $\wh g_n = 0$.
We define $A = \lfloor \log_2 (n^{1/d}) \rfloor$ and $M = \lceil n^{2/d} \rceil$ and consider the geometric grid
$\Lambda = \{ 2^{\alpha} : \alpha = 0, \dots, A \}$.
Now, let 
\begin{equation*}
  \mondrian^{(1)}_{n^{1/d}}, \dots, \mondrian^{(M)}_{n^{1/d}} \sim
  \MP(n^{1/d}, [0, 1]^d)
\end{equation*}
be \iid Mondrian partitions. For $m=1, \ldots, M$, we let $\mondrian_\lambda^{(m)}$ be the pruning of $\mondrian^{(m)}_{n^{1/d}}$ in which only splits occurring before time $\lambda$ have been kept.
We consider now the Mondrian Forest estimators
\begin{equation*}
  \wh f_\alpha = \wh f_{2^\alpha,n_0,M}
\end{equation*}
for every $\alpha = 0, \dots, A$, where we recall that these estimators are given 
by~\eqref{eq:mondrian-forest-estimator}. The estimators $\wh f_\alpha$ are computed using the sample $\dataset_{n_0}$ and the Mondrian partitions $\mondrian_{2^\alpha}^{(m)}$, $1 \leq m \leq M$.
Let
\begin{equation*}
  \wh \alpha = \argmin_{\alpha=0, \ldots, A} \frac{1}{|I_\eps|} \sum_{i\in I_\eps} (\wh f_\alpha (X_i) - Y_i)^2
\end{equation*}
be a risk minimizer and let $\wh \G = \bigcup_{\alpha} [\wh f_{\wh \alpha}, \wh f _{\alpha}]$ where $[f, g] = \{ (1-t) f + t g : t \in [0, 1] \}$.
Note that $\wh \G$ is a star domain with origin  at the empirical risk minimizer $\hat f_{\hat  \alpha}$, hence the name STAR~\cite{audibert2008deviation}.
Then, the adaptive estimator is a convex combination of two Mondrian forests estimates with different lifetime parameters, given by
\begin{equation}
  \label{eq:star-aggregated-estimator}
  \wh g_n = \argmin_{g \in \wh\G} \Big\{ \frac{1}{|I_\eps|} \sum_{i \in I_\eps} (g (X_i) - Y_i)^2  \Big\}\, .
\end{equation}

\begin{proposition}
  \label{prop:adaptive-rate}
  Grant  Assumption~\ref{ass:signal_plus_noise}\textup, with $|Y| \leq B$ almost surely and $f \in \cl^{p,\beta} (L)$ with $p\in \{ 0, 1\}$, $\beta \in (0, 1]$ and $L > 0$.
  Also\textup, assume that the density $p$ of $X$ is $C_p$-Lipschitz and satisfies $\infp \leq p \leq \supp$.
  Then, the estimator $\wh g_n$ defined by~\eqref{eq:star-aggregated-estimator} satisfies\textup:
  \begin{equation}
    \label{eq:oracle-inequality}
    \begin{split}
    \E \big[ ( \wh g_n (X) &- f(X) )^2 \cond X \in B_\eps \big] \\
    &\leq \min_{\alpha=0, \ldots, A} \E \big[ ( \wh f_{\alpha} (X) - f(X) )^2 \cond X \in B_\eps \big] \\
    & \quad + 4 B^2 e^{-c_1 n/4} + 
    \frac{600 B^2 (\log (1+ \log_2 n) + 1)}{c_1 n}
    \end{split}
  \end{equation}
  where $B_\eps = [\eps, 1-\eps]^d$ and $c_1 = p_0 (1-2\eps)^d/4$.
  In particular\textup, we have 
  \begin{equation}
    \label{eq:adaptive-rate}
    \E \big[ ( \wh g_n (X) - f(X) )^2 \cond X \in B_\eps \big]
    = O \big( L^{2d/(d+2s)} n^{-2s/(d+2s)} \big),
  \end{equation}
  where $s = p+\beta$.
\end{proposition}
The proof of Proposition~\ref{prop:adaptive-rate} is to be found in the Supplementary Material.
Proposition~\ref{prop:adaptive-rate} proves that the estimator $\widehat g_n$, which is a STAR aggregation of Mondrian Forests, is adaptive to the smoothness of $f$, whenever $f$ is $s$-H\"older with $s \in (0, 2]$.

\subsection{Results for binary classification}
\label{sec:results_for_binary_classification}

We now consider, as a by-product of the analysis conducted for regression estimation, the setting of binary classification.
Assume that we are given a dataset 
$\dataset_n = \{(X_1, Y_1), \hdots, (X_n,Y_n)\}$
of \iid random variables with values in $[0,1]^d \times \{0,1\}$, 
distributed as a generic pair $(X,Y)$ and define
$\eta(x) = \P[Y=1|X=x]$.
We define the Mondrian Forest classifier $\wh g_{\lambda, n, M}$ as a plug-in estimator of the regression estimator.
Namely, we introduce
\begin{equation*}
\wh g_{\lambda, n, M} (x) = \indic{\wh f_{\lambda, n, M} (x) \geq 1/2}
\end{equation*}
for all $x \in [0,1]^d$, where $\wh f_{\lambda, n, M}$ is the Mondrian Forest estimate defined in the regression setting. 
The performance of $\wh g_{\lambda, n, M}$ is assessed by the $0$-$1$ classification error defined as 
\begin{equation}  
  \label{eq:risk_classif}  
  L(\wh g_{\lambda,n,M}) = \P (\wh g_{\lambda,n,M} (X) \neq Y ),  
\end{equation}
where the probability is taken with respect to $(X, Y, \Pi_{\lambda, M},  \dataset_n)$, 
where $\Pi_{\lambda, M}$ is the set sampled Mondrian partitions, see~\eqref{eq:mondrian-forest-estimator}.
Note that~\eqref{eq:risk_classif} is larger than the Bayes risk defined as
\begin{equation*}
  L(g^{*}) = \P (g^*(X) \neq Y),  
\end{equation*}
where $g^{*}(x) = \indic{\eta(x) \geq 1/2}$.
A general theorem \cite[Theorem~6.5]{devroye1996ptpr} allows us to derive an upper bound on the distance between the classification risk of $\wh g_{\lambda,n,M}$ and the Bayes risk, based on
Theorem~\ref{thm:minimax-regression}.

\begin{corollary}
  \label{thm:minimax-classification}
  Let $M \geq 1$ and assume that $\eta \in \cl^{0, 1}(L)$. 
  Then\textup, the Mondrian Forest classifier $\wh g_n = \wh g_{\lambda_n,n,M}$ with parameter $\lambda_n \asymp n^{1/(d+2)}$ satisfies
  \begin{equation*}
    L(\wh g_n) - L(g^{*})  = o (n^{-1/(d+2)}).
  \end{equation*}
\end{corollary}

The rate of convergence $o(n^{-1/(d+2)})$ for the error probability with a Lipschitz conditional probability $\eta$ is optimal~\cite{yang1999minimax}.
We can also extend in the same way Theorem~\ref{thm:minimaxc2} to the context of classification.
This is done in the next Corollary, where we only consider the 
$\cl^{1, 1}$ case for convenience.

\begin{corollary}
  \label{cor:minimaxc2-classification}
  Assume that $X$ has a positive and Lipschitz density $p$ w.r.t the Lebesgue measure on $[0, 1]^d$ and that $\eta \in \cl^{1, 1}(L)$. 
  Let $\wh g_n = \wh g_{\lambda_n, n, M_n}$ be the Mondrian Forest classifier composed of $M_n \gtrsim n^{2/(d+4)}$ trees\textup, with lifetime 
  $\lambda_n \asymp n^{1/(d+4)}$. Then\textup, we have 
  \begin{equation}
    \label{eq:minimax-error-rates}    
    \P [\wh g_n (X) \neq Y | X \in B_\eps ] - \P [g^{*}(X) \neq Y | X \in B_\eps ]  = o (n^{-2/(d+4)})  
  \end{equation}
  for all $\eps \in (0, 1/2)$\textup, 
  where $B_\eps = [\eps, 1-\eps]^d$.
\end{corollary}
This shows that Mondrian Forests achieve an improved rate compared to Mondrian trees for classification.

\section{Conclusion}
\label{sec:conclusion}

Despite their widespread use in practice, the theoretical understanding of Random Forests is still incomplete. 
In this work, we show that the Mondrian Forest, originally introduced to provide an efficient online algorithm, leads to an algorithm that is not only consistent, but in fact minimax optimal under nonparametric assumptions in arbitrary dimension.
This provides, to the best of our knowledge, the first results of this nature for a random forest method in arbitrary dimension.
Besides, our analysis allows to illustrate improved rates for forests compared to individual trees.
Mondrian partitions possess nice geometric properties, which can be controlled in an exact and direct fashion, while previous approaches~\cite{biau2008consistency_rf,arlot2014purf_bias} require arguments that work conditionally on the structure of the tree.
This suggests that Mondrian Forests can be viewed as an optimal variant of purely random forests, which could set a foundation for more sophisticated and theoretically sound random forest algorithms.

The minimax rate $O(n^{-2s/(2s + d)})$ for a $s$-H\"older regression with $s \in (0, 2]$ obtained in this paper is very slow when the number of features $d$ is large.
This comes from the well-known curse of dimensionality phenomenon, a problem affecting all fully nonparametric algorithms.
A standard approach used in high-dimensional settings is to work under a sparsity assumption, where only $s \ll d$ features are informative. 
A direction for future work is to improve Mondrian Forests using a data-driven choice of the features along which the splits are performed, reminiscent of Extra-Trees~\cite{geurts2006extremely}. 
From a theoretical perspective, it would be interesting to see how the minimax rates obtained here can be combined with results on the ability of forests to select informative variables (see, for instance, \cite{scornet2015consistency_rf}).

 \section{Proofs}
 \label{sec:proofs}
 
 This Section gathers the proofs of Proposition~\ref{prop:cell-distribution} and Corollary~\ref{lem:diameter} (cell distribution and cell diameter). 
 Then, a sketch of the proof of Proposition~\ref{prop:number-cells} is described in this Section (the full proof, which involves some technicalities, can be found in the Supplementary Material). 
 Finally, we provide the proofs of Theorem~\ref{thm:minimax-regression} and Theorem~\ref{thm:minimaxc2}.
 
 
 \begin{proof}[Proof of Proposition~\ref{prop:cell-distribution}]
 	Let $0 \leq a_1, \dots, a_n, b_1 , \dots, b_n \leq 1$ be such that $a_j \leq x_j \leq b_j$ for $1\leq j \leq d$.  
 	Let  $C := \prod_{j=1}^d [a_j, b_j]$.
 	Note that the event  
 	\begin{equation*}    
 	E_\lambda(C, x) = \big\{ L_{1, \lambda} (x) \leq a_1 , 
 		R_{1, \lambda} (x) \geq b_1, \dots, L_{d, \lambda} (x) 
 		\leq a_d , R_{d, \lambda} (x) \geq b_d \big\}
 	\end{equation*}  
 	coincides --- up to the negligible event that one of the splits of $\mondrian_\lambda$ occurs on coordinate $j$ at $a_j$ or $b_j$ --- with the event that $\mondrian_\lambda$ does not cut $C$, \ie that the restriction $\mondrian_\lambda|_C$ of $\mondrian_\lambda$ to $C$ contains no split.  
 	Now, by the restriction property of the Mondrian process (Fact~\ref{fac:mondrian-restriction}), $\mondrian_\lambda|_C$ is distributed as $\MP (\lambda, C)$;  
 	in particular, the probability that $\mondrian_\lambda|_C$ contains no split is $\exp (- \lambda |C|)$.  
 	Hence, we have  
 	\begin{equation}    
 	\label{eq:mf-bias-law1}
 	\P (E_\lambda(C, x)) = e^{- \lambda (x - a_1)} e^{-\lambda (b_1 - x)} 
 		\times \cdots \times e^{- \lambda (x - a_d)} e^{-\lambda (b_d - x)}\,.
 	\end{equation}  
 	In particular, setting $a_j = b_j = x$ in~\eqref{eq:mf-bias-law1} except for one $a_j$ or $b_j$, and using that $L_{j,\lambda} (x ) \leq x$ and $R_{j,\lambda} (x) \geq x$, we obtain
 	\begin{equation}    
 	\label{eq:mf-bias-law2}
 		\P (R_{j, \lambda} (x) \geq b_j ) = e^{-\lambda (b_j - x)}
 		\quad \mbox{and} \quad 
 		\P (L_{j, \lambda} (x) \leq a_j ) = e^{-\lambda (x - a_j)}.
 	\end{equation}  
 	Since clearly $R_{j, \lambda} (x) \leq 1$ and $L_{j, \lambda} (x) \geq  0$, Equation~\eqref{eq:mf-bias-law2} implies~$(ii)$.
 	Additionally, plugging Equation~\eqref{eq:mf-bias-law2} back into Equation~\eqref{eq:mf-bias-law1} shows that $L_{1, \lambda} (x), R_{1, \lambda} (x), \dots, L_{d, \lambda} (x), R_{d, \lambda} (x)$ are independent, \ie 
 	point $(i)$. This completes the proof.  
 \end{proof}

 \begin{proof}[Proof of Corollary~\ref{lem:diameter}]
 	Using Proposition~\ref{prop:cell-distribution}, for $1\leq j \leq d$, $D_{j,\lambda} (x) = R_{j, \lambda} (x) - x_j + x_j - L_{j,\lambda}(x)$ is stochastically upper bounded by $\lambda^{-1}(E_1 + E_2)$ with $E_1, E_2$ two independent $\expdist (1)$ random variables, which is distributed as  $\gammadist (2, \lambda)$.
 	This implies that
 	\begin{equation}  
 	\label{eq:diameter-proof-1}  
 	\P (D_{j,\lambda} (x) \geq \delta) \leq (1 + \lambda \delta) 
 	e^{- \lambda \delta}
 	\end{equation}
 	for every $\delta >0$ (with equality if $\delta \leq x_j \wedge (1-x_j)$) and
 	$\E [D_{j,\lambda} (x)^2] \leq \lambda^{-2} ( \E [E_1^2] + \E [E_2^2]) = 4 / \lambda^2$.
 	The bound~\eqref{eq:diameter-bound} for the diameter $D_\lambda (x) = [\sum_{j=1}^d D_{j,\lambda} (x)^2 ]^{1/2}$ is obtained by noting that
 	\begin{equation*}  
 	\P (D_\lambda (x) \geq \delta)
 	\leq \P \Big(\exists j : D_{j,\lambda} (x) \geq \frac{\delta}{\sqrt{d}}\Big)
 	\leq \sum_{j=1}^d \P \Big(D_{j,\lambda} (x) \geq \frac{\delta}{\sqrt{d}}\Big) \, ,
 	\end{equation*}
 	while~\eqref{eq:diameter-square} follows from the identity
 	$\E [D_\lambda(x)^2] = \sum_{j=1}^d \E [D_{j,\lambda} (x)^2]$.
 \end{proof}

 \begin{proof}[Sketch of Proof of Proposition~\ref{prop:number-cells}]
 	Let us provide here an outline of the argument; a fully detailed proof is available in the Supplementary Material. 
 	The general idea of the proof is to modify the construction of Mondrian partitions (and hence their distribution) in a way that leaves the expected number of cells unchanged, while making this quantity directly computable.
 	
 	Consider a Mondrian partition $\mondrian_\lambda \sim \MP (\lambda, [0, 1]^d)$ and a cell $C$ formed at time $\tau$ in it (\eg, $C = [0, 1]^d$ for $\tau = 0$).
 	By the properties of exponential distributions, the split of $C$ (if it exists) from Algorithm~\ref{alg:split-cell} can be obtained as follows.
 	Sample independent variables $E_{j}, U_j$ with $E_j \sim \expdist(1)$ and $U_j \sim \uniformdist([0,1])$ for $j = 1, \dots, d$.
 	Let $T_j = (b_j-a_j)^{-1} E_j$ and $S_j = a_j + (b_j - a_j) U_j$, where $C = \prod_{j=1}^d [a_j,b_j]$, and set $J = \argmin_{1\leq j \leq d} T_j$.
 	If $\tau + T_J > \lambda$ then $C$ is not split (and is thus a cell of $\mondrian_\lambda$). 
 	On the other hand, if $\tau + T_J \leq \lambda$ then $C$ is split along coordinate $J$ at $S_J$ (and at time $\tau + T_J$) into $C' = \{ x \in C : x_J \leq S_J \}$ and $C'' = C \setminus C'$.
 	This process is then repeated for the cells $C'$ and $C''$, by using independent random variables $E_j',U_j'$ and $E_j'', U_j''$ respectively.
 	
 	Now, note that the number of cells $K_\lambda (C)$ in $\mondrian_\lambda$ contained in $C$ is the sum of the number of cells in $C'$ and $C''$, namely $K_\lambda (C')$ and $K_\lambda (C'')$.
 	Hence, the expectation of $K_\lambda (C)$ (conditionally on previous splits) only depends on the distribution of the split $(J, S_J, T_J)$, as well as on the marginal distributions of $K_\lambda (C')$ and $K_\lambda (C'')$, but not on the joint distribution of $(K_\lambda (C'), K_\lambda (C''))$.
 	
 	Consider the following change: instead of splitting $C'$ and $C''$ based on the independent random variables $E_j',U_j'$ and $E_j'',U_j''$ respectively, we reuse for both $C'$ and $C''$ the variables $E_j,U_j$ (and thus $S_j,T_j$) for $j \neq J$, which were not used to split $C$.
 	It can be seen that, for both $C'$ and $C''$, these variables have the same conditional distribution given $J,S_J,T_J$ as the independent ones.
 	One can then form the modified random partition $\wt\mondrian_\lambda$ by recursively applying this change to the construction of $\mondrian_\lambda$, starting with the root and propagating the unused variables at each split.
 	By the above outlined argument, its number of cells $\wt K_\lambda$ satisfies $\E [\wt K_\lambda] = \E [K_\lambda]$.
 	
 	On the other hand, one can show that the partition $\wt \mondrian_\lambda$ is a ``product'' of independent one-dimensional Mondrian partition $\mondrian^j_{\lambda} \sim \MP (\lambda, [0, 1])$ along the coordinates $j = 1, \dots, d$ (this means that the cells of $\wt\mondrian_\lambda$ are the Cartesian products of cells of the $\mondrian^j_{\lambda}$).
 	Since the splits of a one-dimensional Mondrian partition of $[0, 1]$ form a Poisson point process of intensity $\lambda \di x$ (Fact~\ref{fac:mondrian-poisson}), the expected number of cells of $\mondrian^j_\lambda$ is $1+\lambda$.
 	Since the $\mondrian_\lambda^j$ for $j = \{1, \ldots, d\}$ are independent, this implies that $\E [\wt K_\lambda] = (1+\lambda)^d$.
 	Once again, the full proof is provided in the Supplementary Material.
 \end{proof}
 
 \begin{proof}[Proof of Theorem~\ref{thm:minimax-regression}]
   Recall that the Mondrian Forest estimate at $x$ is given by  
\begin{equation*}
\wh f_{\lambda,n, M}(x) = \frac 1M \sum_{m=1}^M \wh f_{\lambda,n}^{(m)}(x)  \,.
\end{equation*}
By convexity of the function $y' \mapsto (y - y')^2$ for any $y \in \R$, we have 
\begin{equation*}
R (\wh f_{\lambda,n,M}) \leq \frac 1M \sum_{m=1}^M R(\wh f_{\lambda,n}^{(m)}) = R (\wh f_{\lambda,n}^{(1)}),
\end{equation*}
since the random trees estimators $\wh f_{\lambda,n}^{(m)}$ have the same distribution for $m=1 \ldots M$.
Hence, it suffices to prove Theorem~\ref{thm:minimax-regression} for the tree estimator $\wh f_{\lambda,n}^{(1)}$. 
We will denote for short $\wh f_{\lambda} := \wh f_{\lambda,n}^{(1)}$ all along this proof.

\paragraph{Bias-variance decomposition}

We establish a \emph{bias-variance} decomposition of the risk of a Mondrian tree, akin to the one stated for purely random forests by \cite{genuer2012variance_purf}.
Denote $\bar f_{\lambda} (x) := \E [ f(X) | X \in \cell_{\lambda} (x) ]$ (which depends on $\mondrian_{\lambda}$) for every $x$ in the support of~$\mu$.
Given $\mondrian_\lambda$, the function $\bar f_{\lambda}$ is the orthogonal projection of $f \in L^2([0, 1]^d, \mu)$ on the subspace of functions that are constant on the cells of $\mondrian_\lambda$.
Since $\wh f_{\lambda}$ belongs to this subspace given $\dataset_n$, we have conditionally on $(\mondrian_\lambda, \dataset_n)$:
\begin{equation*}
\E_X \big[( f(X) - \wh f_{\lambda} (X))^2 \big]
= \E_X \big[( f(X) - \bar f_{\lambda} (X))^2 \big] + \E_X \big[( \bar f_{\lambda} (X) - \wh f_{\lambda} (X))^2 \big].
\end{equation*}
This gives the following decomposition of the risk of $\wh f_{\lambda}$ by taking the expectation over $(\mondrian_\lambda, \dataset_n)$:
\begin{equation}
\label{eq:bias-variance}
R (\wh f_{\lambda}) = \E \big[ (f (X) - \bar f_{\lambda} (X))^2 \big] +      
\E \big[ (\bar f_{\lambda} (X) - \wh f_{\lambda} (X))^2 \big] \, .      
\end{equation}
The first term of the sum, the \emph{bias}, measures how close $f$ is to  its best approximation $\bar f_{\lambda}$ that is constant on the leaves of $\mondrian_{\lambda}$ (on average over $\mondrian_{\lambda}$).
The second term, the \emph{variance}, measures how well the expected value
$\bar f_{\lambda} (x)$ is estimated by the empirical average 
$\wh f_{\lambda} (x)$ (on average over $\dataset_n,\mondrian_{\lambda}$). 

Note that~\eqref{eq:bias-variance} holds for the estimation risk \emph{integrated over the hypercube $[0, 1]^d$}, and not for the pointwise estimation risk.
This is because in general, we have $\E_{\dataset_n} \big[ \wh f_{\lambda} (x) \big] \neq \bar f_\lambda (x)$: indeed, the cell $\cell_\lambda (x)$ may contain no data point in $\dataset_n$, in which case the estimate $\wh f_{\lambda}(x)$ equals $0$.
It seems that a similar difficulty occurs for the decomposition in \cite{genuer2012variance_purf,arlot2014purf_bias}, which should only hold for the integrated risk.

\paragraph{Bias term}

For each $x\in [0, 1]^d$ in the support of $\mu$, we have
\begin{align*}
| f (x) - \bar f_{\lambda} (x) |
&= \Big| \frac{1}{\mu (\cell_{\lambda} (x))} \int_{\cell_{\lambda} (x)} ( f(x) - f(z)) \mu (\di z) \Big| \nonumber \\
&\leq \sup_{z \in \cell_{\lambda} (x)} | f(x) - f(z) | \leq L D_{\lambda} (x)^\beta, \nonumber
\end{align*}    
where $D_{\lambda} (x)$ is the $\ell^2$-diameter of $\cell_{\lambda}(x)$, since $f \in \cl^{0, \beta}(L)$. 
By concavity of $x \mapsto x^\beta$ for $\beta \in (0, 1]$ and Corollary~\ref{lem:diameter}, this implies
\begin{equation}
\label{eq:bound-bias-1}
\E \big[ (f (x) - \bar f_{\lambda} (x))^2 \big]
\leq L^2 \E [D_{\lambda} (x)^{2 \beta}]
\leq L^2 \E [D_{\lambda} (x)^{2}]^\beta
\leq L^2 \Big( \frac{4 d}{\lambda^2} \Big)^\beta.
\end{equation}
Integrating~\eqref{eq:bound-bias-1} with respect to $\mu$ yields the following bound on the bias:
\begin{equation}
\label{eq:bound-bias}
\E \big[ (f(X) - \bar f_{\lambda} (X))^2 \big] \leq \frac{(4 d)^\beta L^2}{\lambda^{2 \beta}} \, .
\end{equation}

\paragraph{Variance term}

In order to bound the variance term,
we use Proposition~2 in \cite{arlot2014purf_bias}: if $\treepart$ is a random tree partition of the unit cube in $k$ cells (with $k \in \N^*$ deterministic) formed independently of the dataset $\dataset_n$, then
\begin{equation}
\label{eq:proof-var-1}
\E \big[ (\bar f_\treepart (X) - \wh f_\treepart (X))^2 \big]
\leq
\frac{k}{n} ( 2 \sigma^2 + 9 \| f \|_{\infty}^2) \, .
\end{equation}
Note that Proposition~2 in \cite{arlot2014purf_bias}, stated in the case where the noise variance is constant, still holds when the noise variance is just upper bounded, based on Proposition~1 in \cite{arlot2008vfold}.
For every $k \in \N^*$, applying~\eqref{eq:proof-var-1} to the random partition $\mondrian_{\lambda} \sim \MP (\lambda , [0, 1]^d)$ conditionally on the event $\{ K_{\lambda} = k \}$, 
we get
\begin{align}
\label{eq:proof-var-2}
\E \big[ (\bar f_{\lambda} (X) - \wh f_{\lambda, n} (X))^2 \big]
&= \sum_{k = 1}^{\infty} \P (K_{\lambda} = k) \, \E [ (\bar f_{\lambda} (X) - \wh f_{\lambda} (X))^2 \cond K_{\lambda} = k] \nonumber\\
&\leq \sum_{k = 1}^{\infty} \P (K_{\lambda} = k) \,\frac{k}{n}
\left( 2 \sigma^2 + 9 \| f \|_{\infty}^2 \right) \nonumber\\
&= \frac{\E [K_{\lambda}] }{n} \left( 2 \sigma^2 + 9 \| f \|_{\infty}^2 \right). \nonumber     
\end{align}
Using Proposition~\ref{prop:number-cells}, we obtain an upper bound of the variance term:
\begin{equation}
\label{eq:bound-variance}
\E \big[ (\bar f_{\lambda} (X) - \wh f_{\lambda} (X))^2 \big]
\leq \frac{(1+\lambda)^d}{n}
\left( 2 \sigma^2 + 9 \| f \|_{\infty}^2 \right) \, .
\end{equation}
Combining~\eqref{eq:bound-bias} and~\eqref{eq:bound-variance} leads 
to~\eqref{eq:risk-regression}. 
Finally, the bound~\eqref{eq:minimax-eta-rates} follows by using
$\lambda = \lambda_n$ in~\eqref{eq:risk-regression}, which concludes the proof of Theorem~\ref{thm:minimax-regression}.
\end{proof}

\begin{proof}[Proof of Theorem~\ref{thm:minimaxc2}]

Consider a Mondrian Forest 
\begin{align*}
\wh f_{\lambda, M} (x) = \frac{1}{M} \sum_{m=1}^M \wh f_{\lambda}^{(m)} (x),
\end{align*}
where the Mondrian Trees $\wh f_{\lambda}^{(m)}$ for $m=1, \ldots, M$ are based on independent partitions $\mondrian_\lambda^{(m)} \sim \MP (\lambda, [0, 1]^d)$.
Also, for $x$ in the support of $\mu$ let 
\begin{equation*}
\bar f_\lambda^{(m)} (x) = \E_X [ f (X) \cond X \in \cell_\lambda^{(m)} (x) ],
\end{equation*}
which depends on $\mondrian_\lambda^{(m)}$.
Let $\wt f_\lambda (x) = \E [\bar f_\lambda^{(m)} (x)]$, which is deterministic and does not depend on $m$.
Denoting $\bar{f}_{\lambda, M} (x) = \frac{1}{M} \sum_{m=1}^M \bar{f}_\lambda^{(m)} (x)$, we have
\begin{equation*}
\E [ (\wh f_{\lambda, M} (x) - f(x))^2 ] \leq 2 \E [ ( \wh f_{\lambda, M} (x) - \bar f_{\lambda, M} (x) )^2 ] + 2 \E [ ( \bar f_{\lambda, M} (x) - f(x) )^2 ].
\end{equation*}
In addition, Jensen's inequality implies that
\begin{align*}
\E \big[(  \wh f_{\lambda,M} (x) - \bar f_{\lambda,M} (x))^2 \big]
&\leq \frac{1}{M} \sum_{m=1}^M \E \big[ (\wh f_{\lambda}^{(m)} (x) - \bar f_{\lambda}^{(m)} (x))^2  \big] \\
&= \E \big[(  \wh f_{\lambda}^{(1)} (x) - \bar f_{\lambda}^{(1)} (x))^2 \big] \, .
\end{align*}
For every $x$ we have that $\bar{f}_\lambda^{(m)} (x)$ are \iid for $m=1, \ldots, M$ with 
expectation $\wt f_\lambda (x)$, so that
\begin{equation*}
\E \big[( \bar f_{\lambda,M} (x) - f(x))^2 \big]
= ( \wt f_{\lambda} (x) - f(x) )^2 + \frac{\Var (\bar f_{\lambda}^{(1)} (x))}{M} \, .
\end{equation*}
Since $f \in \cl^{1,\beta} (L)$ we have in particular that $f$ is $L$-Lipschitz, hence
\begin{align*}
\Var (\bar f_\lambda^{(1)}(x)) \leq \E \big[ ( \bar f_\lambda^{(1)} (x) - f(x) )^2 \big]
\leq L^2 \E [ D_\lambda (x)^2 ] 
\leq \frac{4 d L^2}{\lambda^2}
\end{align*}
for all $x\in [0,1]^d$, where we used Corollary~\ref{lem:diameter} and where $D_\lambda (x)$ stands for the diameter of $\cell_\lambda (x)$.
Consequently, taking the expectation with respect to $X$, we obtain
\begin{align*}
\E \big[ (\wh f_{\lambda, M} (X) - f(X))^2 \big] &
\leq  \frac{8 d L^2}{M \lambda^2} + 2 \E \big[ ( 
\wh f_{\lambda}^{(1)} (X) - \bar f_{\lambda}^{(1)} (X))^2  \big] \\
& \qquad + 2 \E \big[ ( \wt f_{\lambda}(X) - f(X) )^2  \big].
\end{align*}
The same upper bound holds also conditionally on $X \in B_\eps := [\varepsilon, 1- \varepsilon]^d$: 
\begin{equation}
\label{eq:bias-variance-forests}
\begin{split}
&\E \big[ (\wh f_{\lambda, M} (X) - f(X))^2 | X \in B_\eps \big]
\leq \frac{8 d L^2}{M \lambda^2} +  \\
&2 \E \big[ (  \wh f_{\lambda}^{(1)} (X) - \bar f_{\lambda}^{(1)} (X))^2 | X \in B_\eps \big] %
+ 2 \E \big[ ( \wt f_{\lambda}(X) - f(X))^2 | X \in B_\eps \big].		
\end{split}
\end{equation}

\paragraph{Variance term}

Recall that the distribution $\mu$ of $X$ has a positive density $p: [0, 1]^d \to \R_+^*$ which is $C_p$-Lipschitz, and recall that $\infp = \inf_{x \in [0, 1]^d} p(x)$ and $\supp = \sup_{x \in [0, 1]^d} p(x)$, both of which are positive and finite, since the continuous function $p$ reaches its maximum and minimum over the compact 
set $[0, 1]^d$.
As shown in the proof of Theorem~\ref{thm:minimax-regression}, the variance term satisfies
\begin{align*}
\E \big[ (\bar f_{\lambda}^{(1)} (X) - \wh f_{\lambda,n}^{(1)} (X))^2 \big]
\leq \frac{(1+\lambda)^d}{n}
\left( 2 \sigma^2 + 9 \| f \|_{\infty}^2 \right) \, .
\end{align*}
Hence, the conditional variance in the decomposition~\eqref{eq:bias-variance-forests} satisfies
\begin{align}
\E \big[(\bar f_{\lambda}^{(1)} (X) - \wh f_{\lambda}^{(1)} (X))^2 | X \in B_\eps\big] 
\nonumber 
&\leq \P(X \in B_\eps)^{-1} \E [\big(\bar f_{\lambda}^{(1)} (X) - \wh f_{\lambda}^{(1)} (X))^2\big] 
\nonumber \\
& \leq p_0^{-1} (1-2\eps)^{-d} \frac{(1+\lambda)^d}{n} 
( 2 \sigma^2 + 9 \| f \|_{\infty}^2 ) .
\label{eq:variance-forests-cond}
\end{align}

\paragraph{Expression of $\wt f_{\lambda}$}

It remains to control the bias term in the decomposition~\eqref{eq:bias-variance-forests}, which is the most involved part of the proof.
Let us recall that $\cell_\lambda (x)$ stands for the cell of $\mondrian_\lambda$ which contains $x \in [0, 1]^d$.
We have
\begin{align}
\wt f_{\lambda} (x)
&= \E \Big[ \frac{1}{\mu(\cell_\lambda (x))} \int_{[0, 1]^d}  f(z) p(z)  \bm 1 (z \in \cell_\lambda (x)) \, \di z \Big] \nonumber \\
&= \int_{[0, 1]^d} f(z) \, F_{p,\lambda} (x, z) \, \di z 
\label{eq:mf-bias-0},
\end{align}
where we defined
\begin{equation*}
F_{p,\lambda} (x, z) =
\E \left[ \frac{ p(z)\bm 1 (z \in \cell_\lambda (x))}{\mu (\cell_\lambda (x))} \right] \, .  
\end{equation*}
In particular, $\int_{[0, 1]^d} F_{p,\lambda} (x,z) \di z = 1$ for any $x\in [0, 1]^d$ (letting $f \equiv 1$ above).
Let us also define the function $F_\lambda$, which corresponds to the
case $p \equiv 1$:
\begin{equation*}
F_{\lambda} (x, z) =
\E \Big[ \frac{ \bm 1 (z \in \cell_\lambda (x))}{\vol (\cell_\lambda (x))} \Big] \,,
\end{equation*}
where $\vol(\cell)$ stands for the volume of a box $\cell$.

\paragraph{Second order expansion}

Assume that $f \in \cl^{1+\beta} ([0, 1]^d)$ for some $\beta \in (0, 1]$.
This implies that
\begin{align*}
& | f (z) - f(x) - \nabla f (x)^\top (z - x) |\\
&= \Big| \int_{0}^1 [ \nabla f (x + t (z-x)) - \nabla f (x) ]^\top (z-x) \di t \Big| \\  
&\leq \int_{0}^1 L (t \| z - x \|)^\beta \| z - x \| \di t \leq L \| z - x \|^{1+\beta} \, .
\end{align*}
Now, by the triangle inequality,
\begin{align*}
\bigg| \, \Big|\int_{[0, 1]^d} ( &f(z) - f(x) ) F_{p, \lambda} (x, z) \di z \Big| - 
\Big| \int_{[0, 1]^d} \nabla f (x)^\top (z-x) F_{p, \lambda} (x, z) \di z \Big| \, \bigg| \\
& \leq \Big| \int_{[0, 1]^d} ( f(z) - f(x) - \nabla f (x)^\top (z-x) )
F_{p, \lambda} (x, z) \di z \Big| \\
& \leq L \int_{[0, 1]^d} \| z - x \|^{1+\beta} F_{p, \lambda} (x, z) \di z \, ,
\end{align*}
so that, using together $\int F_{p, \lambda} (x, z) \di z =1$ and~\eqref{eq:mf-bias-0}, we obtain
\begin{equation}
\label{eq:bias-decomp-ab}
\begin{split}
| \wt f_\lambda (x) - f (x) | &\leq 
\Big| \nabla f(x)^\top \underbrace{\int_{[0, 1]^d} (z-x) 
	F_{p, \lambda} (x, z) \di z}_{:= A} \Big| \\
& \quad + L \underbrace{\int_{[0, 1]^d} \| z - x \|^{1+\beta} 
	F_{p, \lambda} (x, z) \di z}_{:= B} \, .
\end{split}
\end{equation}
Hence, it remains to control the two terms $A,B$ from Equation~\eqref{eq:bias-decomp-ab}.
We will start by expressing $F_{p, \lambda}$ in terms of $p$, using the distribution of 
the cell $C_\lambda (x)$ given by Proposition~\ref{prop:cell-distribution} above.
Next, both terms will be bounded by approximating $F_{p, \lambda}$ by $F_\lambda$ and controlling these terms for $F_\lambda$ (this is done in Technical Lemma~\ref{lem:int_against_Flambda} below).

\paragraph{Explicit form of $F_{p,\lambda}$}

First, we provide an explicit form of $F_{p, \lambda}$ in terms of $p$.
We start by determining the distribution of the cell $C_\lambda (x)$ conditionally on the event $z \in C_\lambda (x)$.
Let $C = C (x, z) = \prod_{1\leq j \leq d} [ x_j \wedge z_j , x_j \vee z_j] \subseteq [0, 1]^d$ be the smallest box containing both $x$ and $z$; also, let $a_j = x_j \wedge z_j$, $b_j = x_j \vee z_j$, $a = (a_j)_{1 \leq j \leq d}$ and $b = (b_j)_{1\leq j \leq d}$.
Note that $z \in \cell_\lambda (x)$ if and only if $\mondrian_\lambda$ does not cut $C$.
Since $C = C(x,z) = C(a,b)$, we have that $z \in \cell_\lambda (x)$ if and only if $b \in \cell_\lambda (a)$, and in this case $\cell_\lambda (x) = \cell_\lambda (a)$.
In particular, the conditional distribution of $\cell_\lambda (x)$ given $z \in C_\lambda (x)$ equals the conditional distribution of $\cell_\lambda (a)$ given $b \in C_\lambda (a)$.

Write $C_\lambda (a) = \prod_{j=1}^d [L_{\lambda, j} (a), R_{\lambda, j} (a)]$; by Proposition~\ref{prop:cell-distribution}, we have $L_{\lambda, j} (a) = (a_j - \lambda^{-1} \expleft{j}) \vee 0$, $R_{\lambda, j} (a) = (a_j + \lambda^{-1} \expright{j}) \wedge 1$, where $\expleft{j}, \expright{j}$, $1 \leq j \leq d$ are \iid $\expdist(1)$ random variables.
Note that $b \in C_\lambda (a)$ is equivalent to $R_{\lambda, j} (a) \geq b_j$ for $j = 1, \dots, d$,
\emph{i.e.} to $\expright{j} \geq \lambda (b_j - a_j) $.
By the memory-less property of the exponential distribution, 
the distribution of $\expright{j} - \lambda (b_j - a_j)$ conditionally on $\expright{j} \geq \lambda (b_j - a_j)$ is $\expdist (1)$. 
As a result (using the independence of the variables $\expleft{j},\expright{j}$),
we obtain the following statement:
\begin{quote}
	\normalsize
	Conditionally on $b \in C_\lambda (a)$, the coordinates
	$L_{\lambda, j} (a), R_{\lambda, j} (a)$, $1 \leq j \leq d$,
	are distributed as $(a_j - \lambda^{-1} \expleft{j}') \vee 0, (b_j + \lambda^{-1} \expright{j}') \wedge 1$, where $\expleft{j}', \expright{j}'$ are \iid $\expdist (1)$ random variables.
\end{quote}
Hence, the distribution of $C_\lambda (x)$ conditionally on $z \in C_\lambda (x)$ has the same distribution as
\begin{equation}
\label{eq:cell-xz}
\cell_\lambda (x, z)
:= \prod_{j=1}^d \left[ (x_j \wedge z_j - \lambda^{-1} \expleft{j}) \vee 0,
(x_j \vee z_j + \lambda^{-1} \expright{j}) \wedge 1 \right]
\end{equation}
where $\expleft{1}, \expright{1}, \dots, \expleft{d}, \expright{d}$ are \iid $\expdist(1)$ random variables.
In addition, note that $z \in C_\lambda (x)$ if and only if the restriction of $\mondrian_\lambda$ to $C(x,z)$ has no split (\ie, its first sampled split occurs after time $\lambda$).
Since this restriction is distributed as $\MP (\lambda, C(x,z))$ using Fact~\ref{fac:mondrian-restriction}, this occurs with probability $\exp(-\lambda | C(x,z) |) = \exp(-\lambda \| x - z \|_1)$.
Therefore, 
\begin{align}  
F_{p, \lambda} (x, z)
&= \P ( z \in C_\lambda (x) ) \, \E \Big[ \frac{p(z)}{\mu (\cell_\lambda (x))} \,\Big|\, 
z \in C_\lambda (x) \Big] \nonumber \\
&= e^{-\lambda \| x-z \|_1} \, \E \bigg[ \Big\{
\int_{\cell_\lambda (x, z)} \frac{p(y)}{p(z)} \di y
\Big\}^{-1} \bigg] \,,
\label{eq:expression-Fp}
\end{align}
where $C_\lambda (x, z)$ is as in~\eqref{eq:cell-xz}.
In addition, applying~\eqref{eq:expression-Fp} to $p \equiv 1$ yields
\begin{equation}
\label{eq:F_lambda}
\begin{split}
F_\lambda (x,z) = \lambda^d e^{- \lambda \|x-z\|_1} &\prod_{1 \leq j \leq d} \E \Big[ 
\big\{ \lambda |x_j - z_j| + \expleft{j} \wedge \lambda (x_j \wedge z_j) \\
&  \quad \quad \quad \quad  + \expright{j} \wedge \lambda (1- x_j \vee z_j) \big\}^{-1} \Big].
\end{split}
\end{equation}
The following technical Lemma will prove useful in what remains of the proofs, whose proof is given in Supplementary Material.
\begin{lemma}
	\label{lem:int_against_Flambda}
	The function $F_{p, \lambda}$ given by~\eqref{eq:F_lambda} satisfies
	\begin{equation*}
	\Big\| \int_{[0, 1]^d} (z-x) F_\lambda (x,z) \di z \Big\|^2
	\leq \frac{9}{\lambda^2} \sum_{j=1}^d e^{-\lambda [x_j \wedge (1-x_j)]}
	\end{equation*}
	and
	\begin{equation*}
	\int_{[0, 1]^d} \frac{1}{2} \| z - x \|^2 F_\lambda (x,z) \di z \leq 
	\frac{d}{\lambda^2}
	\end{equation*}
	for any $x \in [0,1]^d$.
\end{lemma}

\paragraph{Control of the term B in Equation~\eqref{eq:bias-decomp-ab}}

It follows from~\eqref{eq:expression-Fp} and from the bound $p(y) / p(z) \geq \infp / \supp$ that 
\begin{equation}
\label{eq:bound-density}
F_{p,\lambda} (x,z) \leq \frac{\supp}{\infp} F_\lambda (x,z),
\end{equation}
so that
\begin{align}
\int_{[0, 1]^d} \| z- x \|^{1+\beta} F_{p,\lambda} (x,z) \di z
&\leq \frac{\supp}{\infp} \int_{[0, 1]^d} \| z- x \|^{1+\beta} F_{\lambda} (x,z) \di z \nonumber \\
&\leq \frac{\supp}{\infp} \Big( \int_{[0, 1]^d} \| z- x \|^{2} F_{\lambda} (x,z) \di z \Big)^{(1+\beta)/2} \label{eq:jensen-1beta} \\
&\leq \frac{\supp}{\infp} \Big( \frac{2 d}{\lambda^2} \Big)^{(1+\beta)/2},
\label{eq:bound-B}
\end{align}
where~\eqref{eq:jensen-1beta} follows from the concavity of $x \mapsto x^{(1+\beta)/2}$ for 
$\beta \in (1, 2]$, while~\eqref{eq:bound-B} comes from Lemma~\ref{lem:int_against_Flambda}.

\paragraph{Control of the term $A$ in Equation~\eqref{eq:bias-decomp-ab}}

It remains to control $A = \int_{[0, 1]^d} (z-x) F_{p,\lambda} (x,z) \di z$.
Again, this quantity is controlled
in the case of a uniform density ($p\equiv 1$) in Lemma~\ref{lem:int_against_Flambda}.
However, this time the crude bound~\eqref{eq:bound-density} is no longer sufficient, since we need first-order terms to compensate in order to obtain the optimal rate.
Rather, we will show that $F_{p,\lambda} (x,z) = (1 + O ( \| x- z \| ) + O (1/\lambda) ) F_\lambda (x, z)$.

\paragraph{A first upper bound on $|F_{p, \lambda}(x,z) - F_\lambda(x,z)|$}

Since $p$ is $C_p$-Lipschitz and lower bounded by $\infp$, we have
\begin{equation}
	\label{eq:lipschitz-ratio}
	\Big| \frac{p(y)}{p(z)} - 1 \Big| = \frac{ | p(y) - p(z) |}{p(z)} \leq 
	\frac{C_p}{\infp} \| y - z \| \leq \frac{C_p}{\infp} \diam \cell_\lambda (x, z)
\end{equation}
for every $y \in \cell_\lambda (x, z)$, so that
\begin{equation*}
1 - \frac{C_p}{\infp} \diam \cell_\lambda (x, z)
\leq \frac{p(y)}{p(z)}
\leq 1+ \frac{C_p}{\infp} \diam \cell_\lambda (x, z).
\end{equation*}
Integrating over $\cell_\lambda (x, z)$ and using $p(y)/p(z) \geq \infp / \supp$ gives
\begin{multline}
\label{eq:bound-density-1}
\Big\{ 1+ \frac{C_p}{\infp} \diam \cell_\lambda (x, z) \Big\}^{-1} \vol \cell_\lambda (x, z)^{-1}
\leq \Big\{ \int_{\cell_\lambda(x,z)} \frac{p(y)}{p(z)} \di y \Big\}^{-1}  \\
\leq \Big\{ \Big(1 - \frac{C_p}{\infp} \diam \cell_\lambda (x, z) \Big) 
\vee \frac{\infp}{\supp} \Big\}^{-1} \vol \cell_\lambda (x, z)^{-1}
\, .
\end{multline}	
In addition, since $(1+u)^{-1} \geq 1-u$ for $u\geq 0$, we have 
\begin{equation*}
\Big\{ 1+ \frac{C_p}{\infp} \diam \cell_\lambda (x, z) \Big\}^{-1}
\geq 1 - \frac{C_p}{\infp} \diam \cell_\lambda (x, z)
\, ,
\end{equation*}
so that setting $a := \big( 1 - \frac{C_p}{\infp} \diam \cell_\lambda (x, z) \big)
\vee \frac{\infp}{\supp} \in (0, 1]$ gives
\begin{equation*}
a^{-1} - 1 = \frac{1-a}{a}
\leq \frac{(C_p/\infp) \diam \cell_\lambda (x,z)}{\infp/\supp}
= \frac{\supp C_p}{\infp^2} \diam \cell_\lambda (x,z).
\end{equation*}
Now, Equation~\eqref{eq:bound-density-1} implies that
\begin{align*}
- \frac{C_p}{\infp} \diam \cell_\lambda (x, z) \vol \cell_\lambda (x, z)^{-1}
&\leq \Big\{ \int_{\cell_\lambda(x,z)} \frac{p(y)}{p(z)} \di y \Big\}^{-1} - \vol \cell_\lambda (x, z)^{-1} \\
&\leq \frac{\supp C_p}{\infp^2} \diam \cell_\lambda (x,z) \vol \cell_\lambda (x, z)^{-1}
\, .
\end{align*}
Taking the expectation over $\cell_\lambda (x,z)$ and using~\eqref{eq:expression-Fp} leads to
\begin{align*}
- \frac{C_p}{\infp} \E \big[ \diam \cell_\lambda (x,z) \vol \cell_\lambda (x, z)^{-1} \big]
& \leq e^{\lambda \|x-z\|_1} ( F_{p, \lambda} (x,z) - F_\lambda (x,z) ) \\
& \leq \frac{\supp C_p}{\infp^2} \E \big[ \diam \cell_\lambda (x,z) \vol \cell_\lambda (x, z)^{-1} \big]
\end{align*}
so that
\begin{equation}
\label{eq:bound-density-2}
\begin{split}
	| F_{p, \lambda} (x,z) - F_\lambda (x,z) | &\leq \frac{\supp C_p}{\infp^2} 
	e^{- \lambda \| x - z\|_1} \\
	&\quad \times \E \big[ \diam \cell_\lambda (x,z) \vol \cell_\lambda (x, z)^{-1} \big] \, .
\end{split}
\end{equation}

\paragraph{Control of $\E \big[ \diam \cell_\lambda (x,z) \vol \cell_\lambda (x, z)^{-1} \big]$} 

Let us define the interval
\begin{equation*}
\cell_\lambda^j (x,z) := \big[ (x_j \wedge z_j - \lambda^{-1} \expleft{j}) \vee 0, 
(x_j \vee z_j + \lambda^{-1} \expright{j}) \wedge 1 \big]
\end{equation*}
and let $|\cell_\lambda^j(x,z)| = (x_j \vee z_j + \lambda^{-1} \expright{j}) \wedge 1 - (x_j \wedge z_j - \lambda^{-1} \expleft{j}) \vee 0$
be its length. 
We have 
$\diam \cell_\lambda (x,z) \leq \diam_{\ell^1} \cell_\lambda (x,z)$ using the triangular inequality, so that
\begin{align}
	\E \big[ &\diam \cell_\lambda (x,z) \vol \cell_\lambda (x, z)^{-1} \big] 
	\leq \E \Big[ \sum_{j=1}^d |\cell_\lambda^j(x,z)| \vol \cell_\lambda (x,z)^{-1} \Big] \nonumber \\
	&= \sum_{j=1}^d \E \Big[ |\cell_\lambda^j(x,z)| \prod_{l=1}^d |\cell_\lambda^l(x,z)|^{-1} \Big]
	= \sum_{j=1}^d \E \Big[ \prod_{l\neq j} |\cell_\lambda^l(x,z)|^{-1} \Big] \nonumber \\
&\leq \sum_{j=1}^d \E \Big[ |\cell_\lambda^j(x,z)| \Big] \E \Big[ |\cell_\lambda^j(x,z)|^{-1} \Big] \E \Big[ \prod_{l\neq j} |\cell_\lambda^l(x,z)|^{-1} \Big] \label{eq:xx-1} \\          
&= \sum_{j=1}^d \E \Big[ |\cell_\lambda^j(x,z)| \Big] \times \E \Big[ \prod_{l=1}^d 
|\cell_\lambda^l(x,z)|^{-1} \Big] \label{eq:indep-coord} \\
&= \E \big[ \diam_{\ell^1} \cell_\lambda (x,z) \big] \times \exp 
(\lambda \|x- z\|_1) F_\lambda (x,z) \label{eq:ineq-diam-0}.
\end{align}
Inequality~\eqref{eq:xx-1} relies on the fact that $\E [X] \E [X^{-1}] \geq 1$ for any positive 
random variable $X$ with $X= |\cell_\lambda^j(x,z) |$. 
Equality~\eqref{eq:indep-coord} comes from the independence of $|\cell_\lambda^1(x,z)|, \ldots, 
|\cell_\lambda^d (x,z)|$.
Multiplying both sides of~\eqref{eq:ineq-diam-0} by $e^{-\lambda \|x- z\|_1}$ leads to
\begin{equation}
\label{eq:ineq-diam}
	\begin{split}
	e^{- \lambda \|x- z\|_1} & \E \big[ \diam \cell_\lambda (x,z) \vol \cell_\lambda (x, z)^{-1} \big] \\
	&\quad \leq \E \big[ \diam_{\ell^1} \cell_\lambda (x,z) \big] F_\lambda (x,z)\, .		
	\end{split}
\end{equation}
In addition,
\begin{align}
	\E \big[ \diam_{\ell^1} \cell_\lambda (x,z) \big]
	&\leq \sum_{j=1}^d \E \big[ |x_j - z_j| + \lambda^{-1} (\expright{j} + \expleft{j}) \big] 
	\nonumber \\
	&= \| x - z \|_1 + \frac{2 d}{\lambda} \label{eq:bound-diam-xz}
\, .
\end{align}
Finally, combining Equations~\eqref{eq:bound-density-2},~\eqref{eq:ineq-diam} 
and~\eqref{eq:bound-diam-xz} gives
\begin{equation}
\label{eq:bound-density-3}
| F_{p,\lambda} (x,z) - F_\lambda (x,z) |
\leq \frac{\supp C_p}{\infp^2} \Big( \| x-z \|_1 + \frac{2d}{\lambda} \Big) F_\lambda (x,z)
\, .
\end{equation}

\paragraph{Control of $A$}

From~\eqref{eq:bound-density-3}, we can control $\int_{[0, 1]^d} (z-x) F_{p,\lambda} (x,z) \di z$ 
by approximating $F_{p, \lambda}$ by $F_\lambda$.      
Indeed, we have
\begin{equation}
	\label{eq:non_uniform2}
	\begin{split}
	& \Big\| \int_{[0, 1]^d} (z-x) F_{p,\lambda} (x,z) \di z - \int_{[0, 1]^d} 
	(z-x) F_{\lambda} (x,z) \di z \Big\| \\
	&\leq \int_{[0, 1]^d} \| z - x \| \times | F_{p,\lambda} (x,z) - F_\lambda (x,z) |
	 \di z,
	\end{split}
\end{equation}
with
\begin{align}
& \int_{[0, 1]^d} \| z - x \| \times | F_{p,\lambda} (x,z) - F_\lambda (x,z) | \di z \nonumber \\
&\leq  \frac{\supp C_p}{\infp^2} \int_{[0, 1]^d} \| z - x \| \Big[ \| x - z \|_1 + \frac{2d}{\lambda} 
\Big] F_\lambda (x,z) \di z \quad \quad \mbox{(by~\eqref{eq:bound-density-3})} \nonumber\\
&\leq \frac{\supp C_p}{\infp^2} \Big[ \sqrt{d} \int_{[0, 1]^d} \| z - x \|^2 F_\lambda (x,z) \di z
+ \frac{2d}{\lambda} \int_{[0, 1]^d} \| z - x \| F_\lambda (x,z) \di z \Big] \nonumber\\
&\leq \frac{\supp C_p}{\infp^2} \Big[ \frac{d \sqrt{d}}{\lambda^2} + \frac{2d}{\lambda} \Big(\int_{[0, 1]^d} \| z - x \|^2 F_\lambda (x,z) \di z \Big)^{1/2} \, \Big] \nonumber,
\end{align}
where we used the inequalities $\|v\| \leq \|v\|_1 \leq \sqrt{d} \|v\|$ as well as the Cauchy-Schwarz inequality.
Hence, using Lemma~\ref{lem:int_against_Flambda}, we end up with
\begin{align}
\int_{[0, 1]^d} \| z - x \| \times | F_{p,\lambda} (x,z) - F_\lambda (x,z) | \di z 
&\leq \frac{\supp C_p}{\infp^2} \Big[ \frac{d \sqrt{d}}{\lambda^2} + \frac{2d}{\lambda} \sqrt{\frac{d}{\lambda^2}} \Big] \nonumber\\
&= \frac{\supp C_p}{\infp^2} \frac{3 d \sqrt{d}}{\lambda^2}
\, . \label{eq:non_uniform1}
\end{align}
Inequalities~\eqref{eq:non_uniform2} and~\eqref{eq:non_uniform1} together with Lemma~\ref{lem:int_against_Flambda} entail that 
\begin{align}
\Big\| \int_{[0, 1]^d} (z-x) F_{p,\lambda} (x,z) \di z \Big\|^2 
& \leq 2 \, \Big\| \int_{[0, 1]^d} (z-x) F_{\lambda} (x,z) \di z \Big\|^2 \nonumber \\
& + 2 \Big(\int_{[0, 1]^d} \| z - x \| | F_{p,\lambda} (x,z) - F_\lambda (x,z) | \di z \Big)^2 
\nonumber \\
&\leq \frac{18}{\lambda^2} \sum_{j=1}^d e^{- \lambda [x_j \wedge (1-x_j)]}
+ 2 \Big( \frac{\supp C_p}{\infp^2} \frac{3 d \sqrt{d}}{\lambda^2}\Big)^2.
\label{eq:bound-A}
\end{align}

\paragraph{Control of the bias}

The upper bound~\eqref{eq:bias-decomp-ab} on the bias writes
\begin{equation*}
(\wt f_\lambda (x) - f(x))^2
\leq ( | \nabla f(x)^\top A | + L B )^2 
\leq 2 (| \nabla f (x) \|^2 \times \| A \|^2 + L^2 B^2)\, ,
\end{equation*}
so that plugging the bounds~\eqref{eq:bound-B} of $B$ and~\eqref{eq:bound-A} of $\|A \|$ gives
\begin{align*}
&( \wt f_\lambda (x) - f(x))^2 \\
&\leq 2 L^2 \Big[ \frac{18}{\lambda^2} \sum_{j=1}^d e^{- \lambda [x_j \wedge (1-x_j)]}
+ 2 \Big( \frac{\supp C_p}{\infp^2} \frac{3 d \sqrt{d}}{\lambda^2}\Big)^2 \Big] + 2 L^2 \frac{\supp}{\infp} \Big( \frac{2 d}{\lambda^2} \Big)^{(1+\beta)/2} \\
&\leq \frac{36 L^2}{\lambda^2} \sum_{j=1}^d e^{- \lambda [x_j \wedge (1-x_j)]}
+ \frac{36 L^2 d^3}{\lambda^4} \Big( \frac{\supp C_p}{\infp^2}\Big)^2 + 
\frac{8 L^2 d^{1+\beta}}{\lambda^{2(1+\beta)}} \Big(\frac{\supp}{\infp} \Big)^2 \, .
\end{align*}
By integrating over $X$ conditionally on $X \in B_\eps$, this implies
\begin{equation}
\label{eq:bound-expected}
\begin{split}
\E \big[ ( \wt f_\lambda (X) - f (X) )^2 | X \in B_\eps \big]
	&\leq \frac{36 L^2}{\lambda^2} \psi_\eps (\lambda) + \frac{36 L^2 d^3}{\lambda^4} 
	\Big( \frac{\supp C_p}{\infp^2}\Big)^2 \\
& + \frac{8 L^2 d^{1+\beta}}{\lambda^{2(1+\beta)}} \Big(\frac{\supp}{\infp} \Big)^2,
\end{split}
\end{equation}
where we have, using the fact that $\infp \leq p(x) \leq \supp$ for any $x \in [0, 1]$,
\begin{align}
\psi_\eps (\lambda)
&:= \sum_{j=1}^d \E \big[ e^{-\lambda [X_j \wedge (1-X_j)]} \cond X \in B_\eps \big] 
\leq \frac{d \supp}{\infp (1 - 2\eps)^d} \int_{\eps}^{1-\eps} e^{-\lambda [u \wedge (1-u)]} \di u \nonumber \\
&= \frac{d \supp}{\infp (1 - 2\eps)^d} \times 2 \int_{\eps}^{1/2} e^{-\lambda u} \di u 
\leq \frac{e^{-\lambda \eps}}{\lambda} \frac{2 d \supp}{\infp (1 - 2\eps)^d}
\nonumber
\, .
\end{align}

\paragraph{Conclusion}

The decomposition~\eqref{eq:bias-variance-forests}, together with the bounds~\eqref{eq:variance-forests-cond} on the variance and~\eqref{eq:bound-expected} on the bias lead to
Inequality~\eqref{minimax-c2-rate} from the statement of Theorem~\ref{thm:minimaxc2}.
In particular, if $\eps \in (0, \frac 12)$ is fixed, Inequality~\eqref{minimax-c2-rate} writes
\begin{align*}{}
\E \big[ (\wh f_{\lambda, M} (X) - f(X))^2  | X \in B_\eps \big]
=  O \Big( \frac{\lambda^d}{n} + \frac{L^2}{\lambda^{2(1+\beta)}} + \frac{L^2}{M\lambda^2} \Big).
\end{align*}
One can optimize the right-hand side by setting $\lambda = \lambda_n \asymp L^{2/(d+2s)} n^{1/(d+2s)}$ and $M = M_n \gtrsim \lambda_n^{2\beta} \asymp L^{4\beta/(d+2s)} n^{2\beta/(d+2s)}$ with $s = 1 + \beta \in (1, 2]$.
This leads to the minimax rate $O (L^{2d/(d + 2s)} n^{-2s/(d+2s)})$ for $f \in \cl^{1, \beta}(L)$ as announced in the statement of Theorem~\ref{thm:minimaxc2}. 

On the other hand, we have $e^{-\lambda \eps} = 1$ whenever $\eps = 0$, so that Inequality~\eqref{minimax-c2-rate} becomes in this case
\begin{equation*}
	\E [ (\wh f_{\lambda, M} (X) - f(X))^2 ]
	\leq O \Big( \frac{\lambda^d}{n} + \frac{L^2}{\lambda^{3 \wedge (2s)}} + \frac{L^2}{M \lambda^2} \Big) .
\end{equation*}
When $2s \leq 3$ (\ie $\beta \leq 1/2$), this leads to the same rate as above, with the same choice of parameters.
When $2s > 3$, this leads to the suboptimal rate $O (L^{2d/(d + 3)} n^{-3/(d+3)})$ with the choice $M_n \gtrsim \lambda_n \asymp L^{2/(d+3)} n^{1/(d+3)}$.{}
This concludes the proof of all the claims from Theorem~\ref{thm:minimaxc2}.
\end{proof}

\newpage

\begin{center}
  \LARGE
  \textbf{Supplementary material for the paper ``Minimax optimal rates for Mondrian trees and forests''  }
\end{center}

\section{Introduction}

This supplementary material to the paper ``Minimax optimal rates for Mondrian trees and forests'' gathers several proofs and technical details and definitions that were omitted in the main paper.
Namely, we start with a glossary of notations, then give extra definitions and notations for trees and nested trees partitions in Section~\ref{sec:supp-specific-notations}.
Then, we provide proofs that were omitted in the main paper by order of appearance, namely the proofs of Proposition~\ref{prop:number-cells}, Theorem~\ref{thm:consistency-mondrian}, Proposition~\ref{prop:lower_bound_tree}, Proposition~\ref{prop:adaptive-rate} and Lemma~\ref{lem:int_against_Flambda}.

\glsaddall
\printnoidxglossary[sort=standard,title={List of Symbols}, type=symbolslist,style=notationlong]


\section{Specific notations}
\label{sec:supp-specific-notations}

Let us now introduce some specific notations to describe the decision tree structure and the Mondrian Process.

\subsection{Trees and nested trees partitions}

A decision tree $(\tree, \splits)$ is composed of the following components:
\begin{itemize}
  \item A finite rooted ordered binary tree $T$, with nodes $\nodes (T)$, interior nodes $\inodes (T)$ and leaves $\leaves (T)$ (so that $\nodes (T)$ is the disjoint union of $\inodes (T)$ and $\leaves (T)$).
  The nodes $\node \in \nodes(T)$ are finite words on the alphabet $\{ 0, 1 \}$, that is elements of the set $\{ 0, 1 \}^* = \bigcup_{n\geq 0} \{ 0, 1 \}^n$: the root $\root$ of $T$ is the empty word, and for every interior $\node \in \{ 0, 1 \}^*$, its left child is $\node 0$ (obtained by adding a $0$ at the end of $\node$) while its right child is $\node 1$ (obtained by adding a $1$ at the end of $\node$).
  \item A family of \emph{splits} $\splits = (\asplit_\node)_{\node \in \inodes (T)}$ at each interior node, where each split $\sigma_\node = (j_\node, \cut_\node)$ is characterized by its split dimension $j_\node \in \{ 1, \dots, d \}$ and its threshold $\cut_\node \in [0, 1]$.
\end{itemize}
We associate to $\treepart = (\tree, \splits)$ a partition $(\cell_{\leaf})_{\leaf \in \leaves (T)}$ of the unit cube $[0,1]^d$, called a \emph{tree partition} (or \emph{guillotine partition}).
For each node $\node \in \nodes (T)$, we define a hyper-rectangular region $\cell_\node$ recursively:
\begin{itemize}
  \item The cell associated to the root of $T$ is $[0,1]^d$;
  \item For each $\node \in \inodes(T)$, we define
  \begin{equation*}
  \cell_{\node 0} := \{ x \in \cell_\node : x_{j_\node} \leq \cut_{j_\node}  \} \quad \text{and} \quad \cell_{\node 1} := \cell_\node \setminus \cell_{\node 0}.
  \end{equation*}
\end{itemize}
The leaf cells $(\cell_{\leaf})_{\leaf \in \leaves (T)}$ form a partition of $[0,1]^d$ by construction. 
In what follows, we will identify a tree with splits $(\tree, \splits)$ with
its associated tree partition, and  a node $\node \in \nodes (T)$ with the
cell $\cell_\node \subset [0,1]^d$. 
The Mondrian process, described in the next Section, defines a distribution over nested tree partitions, defined below.

\begin{definition}
  [Nested tree partitions]
  \label{def:refinement}
  A tree partition $\treepart' = (T', \splits')$ is a \emph{refinement} of the tree partition $\treepart = (\tree, \splits)$ if $\tree$ is a subtree of $T'$ and, for every $\node \in \nodes (T) \subseteq \nodes (T')$, $\asplit_\node = \asplit'_\node$.
  A \emph{nested tree partition} is a family $(\treepart_t)_{t \geq 0}$ of tree partitions such that, for every $t, t'\in \R^+$ with $t\leq t'$, $\treepart_{t'}$ is a refinement of $\treepart_t$.
  Such a family can be described as follows: let $\globaltree$ be the (in general infinite, and possibly complete) rooted binary tree, such that $\nodes (\globaltree) = \bigcup_{t \geq 0} \nodes (\tree_t) \subseteq \{0, 1\}^*$.
  For each $\node \in \nodes (T)$, let $\birth_\node = \inf \{ t \geq 0 \mid \node \in \nodes (\tree_t) \} < \infty$ denote the \emph{birth time} of the node $\node$.
  Additionally, let $\asplit_\node$ be the value of the split $\asplit_{\node, t}$ in $\treepart_t$ for $t > \birth_\node$ (which does not depend on $t$ by the refinement property).
  Then, $\treepart$ is completely characterized by $\globaltree$, $\splits = (\asplit_\node)_{\node \in \nodes(\globaltree)}$ 
  and $\births = (\birth_\node)_{\node \in \nodes (\globaltree)}$.
\end{definition}

\subsection{Mondrian Process}

To define rigorously the Mondrian Process, we introduce the 
function $\Phi_C$, which maps any family of couples $(e_\node^j, u_\node^j) \in \R^+ \times [0, 1]$ indexed by the coordinates $j \in \{ 1, \dots, d \}$ and the nodes $\node \in \{ 0, 1 \}^*$  to a nested tree partition $\treepart = \Phi_C ((e_\node^j, u_\node^j)_{\node, j})$ of $C$.
The splits $\asplit_\node = (j_\node, \cut_\node)$ and birth times $\birth_\node$ of the nodes $\node \in \{ 0,1 \}^*$ are defined recursively, starting from the root $\root$:
\begin{itemize}    
  \item For the root node $\root$, we let $\birth_\root = 0$ and $\cell_\root = C$.
  \item At each node $\node \in \{ 0, 1 \}^*$, given the labels of all its ancestors $\node' \prefix \node$ (so that in particular $\birth_\node$ and $\cell_\node$ are determined), denote $\cell_\node = \prod_{j=1}^d [a_\node^j, b_\node^j]$.
  Then, select the split dimension $j_\node \in \{ 1, \dots, d \}$ and its location $\cut_\node$ as follows:
  \begin{equation}        
  \label{eq:select-dim-mondrian}        
  j_\node = \argmin_{j=1, \ldots, d} \frac{e_\node^j}{b_\node^j - a_\node^j} , \qquad        
  \cut_\node = a_\node^{j_\node} + (b_\node^{j_\node} - a_\node^{j_\node}) \cdot u_\node^{j_\node},
  \end{equation}      
  where we break ties in the choice of $j_\node$ e.g., by choosing the smallest index $j$ in the $\argmin$.
  The node $\node$ is then split at time $\birth_\node + e_\node^{j_\node} / (b_\node^{j_\node} - a_\node^{j_\node}) = \birth_{\node 0} = \birth_{\node 1}$, we let $\cell_{\node 0} = \{ x \in \cell_\node : x_{j_\node} \leq \cut_\node \}$, $\cell_{\node 1} = \cell_\node \setminus \cell_{\node 0}$ and 
  recursively apply the procedure to its children $\node 0$ and $\node 1$.  
\end{itemize}

For each $\lambda \in \R^+$, the tree partition $\treepart_\lambda = \Phi_{\lambda, C} ((e_\node^j, u_\node^j)_{\node, j})$ is the \emph{pruning of $\treepart$ at time $\lambda$}, obtained by removing all the splits in $\treepart$ that occurred strictly after $\lambda$, so that the leaves of the tree are the maximal nodes (in the prefix order) $\node$ such that $\birth_\node \leq \lambda$.

\begin{definition}[Mondrian process]
  \label{def:mondrian-process}
  Let $(E_\node^j, U_\node^j)_{\node, j}$ be a family of independent random variables, with $E_\node^j \sim \expdist(1)$, $U_\node^j \sim \uniformdist([0, 1])$.
  The \emph{Mondrian process} $\MP (C)$ on $C$ is the distribution of the random nested tree partition $\Phi_C ((E_\node^j, U_\node^j)_{\node, j})$.
  In addition, we denote $\MP (\lambda, C)$ the distribution of $\Phi_{\lambda, C} ((E_\node^j, U_\node^j)_{\node, j})$.
\end{definition}

\section{Proof of Proposition~\ref{prop:number-cells}}

At a high level, the idea of the proof is to modify the construction of the Mondrian partition (and hence, the distribution of the underlying process) without affecting the expected number of cells.
More precisely, we show a recursive way to transform the Mondrian process that
leaves $\E [K_\lambda]$ unchanged, and which eventually leads to a random partition $\wt \mondrian_\lambda$ for which this quantity can be computed directly and equals $(1+\lambda)^d$.
We will in fact show the result for a general box $C$ (not just the unit cube).  
The proof proceeds in two steps:
\begin{enumerate}
  \item Define a modified process $\wt \mondrian$, and show that $\E[ \wt K_\lambda]
   = \prod_{j=1}^d (1 + \lambda |C^j|)$.    
  \item It remains to show that $\E[ K_\lambda] = \E[ \wt K_\lambda]$.
  For this, it is sufficient to show that the distribution of the birth times $\birth_\node$ and $\wt \birth_\node$ of the node $\node$ is the same for both processes.
  This is done by induction on $\node$, by showing that the splits at one node of both processes have the same conditional distribution given the splits at previous nodes.    
\end{enumerate}

Let $(E_\node^j, U_\node^j)_{\node \in \{0, 1\}^*, 1\leq j \leq d}$ be a family of independent random variables with $E_\node^j \sim \expdist(1)$ and $U_\node^j \sim \uniformdist([0, 1])$.  
By definition, $\mondrian = \Phi_C ((E_\node^j, U_\node^j)_{\node, j})$ ($\Phi_C$ being defined in Section~\ref{sec:def-mondrian}) follows a Mondrian process distribution $\MP (C)$.
Denote for every node $\node\in \{0, 1\}^*$ $\cell_\node$ the cell of $\node$, $\birth_\node$ its birth time, as well as its split time $T_\node$, dimension $J_\node$, and threshold $\Cut_\node$ (note that $T_\node = \birth_{\node 0} = \birth_{\node 1}$).
In addition, for $\lambda \in \R^+$, denote $\mondrian_\lambda \sim \MP(\lambda, C)$ the tree partition restricted to time $\lambda$, and $K_\lambda \in \N \cup \{ + \infty \}$ its number of nodes.

\paragraph{Construction of the modified process}

Now, consider the following modified nested partition of $C$, denoted $\wt\mondrian$, and defined through its split times, dimension and threshold $\wt T_\node, \wt J_\node, \wt \Cut_\node$ (which determine the birth times $\birth_\node$ and cells $\cell_\node$), and \emph{current $j$-dimensional node} $\node_j (\node) \in \{0, 1\}^*$ $(1\leq j \leq d)$ 
at each node $\node$. 
First, for every $j = 1, \dots, d$, let $\mondrian'^j = \Phi_{C^j} ((E_\node^j , U_\node^j)_{\node\in \{0, 1\}^*}) \sim \MP (C^j)$ 
be the nested partition of the interval $C^j$ determined by $(E_\node^j , U_\node^j)_{\node}$; 
its split times and thresholds are denoted $(\Cut'^j_\node, T'^j_\node)$. 
Then, $\wt \mondrian$ is defined recursively as follows: 
\begin{itemize}
  \item At the root node $\root$, let $\wt \birth_\root = 0$, $\wt \cell_\root = C$ and $\node_j (\root) := \root$ for $1 \leq j \leq d$.  
  \item At node $\node$, given $(\birth_{\node'}, \cell_{\node'}, \node_j (\node'))_{\node' \prefixeq \node}$ (\ie, given $(\wt J_{\node'}, \wt \Cut_{\node'}, \wt T_{\node'})_{\node'\prefix \node}$) define:          
  \begin{align}    
  \label{eq:rec-modified-mondrian}    
  &\wt T_\node = \min_{1\leq j \leq d} T'^{j}_{\node_j(\node)}    
  , \qquad    
  \wt J_\node := \argmin_{1\leq j \leq d} T'^j_{\node_j (\node)}    
  , \qquad    
  \wt \Cut_\node = \Cut'^j_{\node_j (\node)}   , \\
  \label{eq:rec-modified-mondrian-node}    
  &\node_j (\node a) =    
  \begin{cases}      
  \node_j (\node) a & \mbox{if } j = \wt J_\node \\      
  \node_j (\node)   & \mbox{else.}      
  \end{cases}    
  \end{align}  
\end{itemize}
Finally, for every $\lambda \in \R^+$, define $\wt \mondrian_\lambda$ and $\wt K_\lambda$ as before from $\wt \mondrian$.
This construction is illustrated in Figure~\ref{fig:modified-proof}.

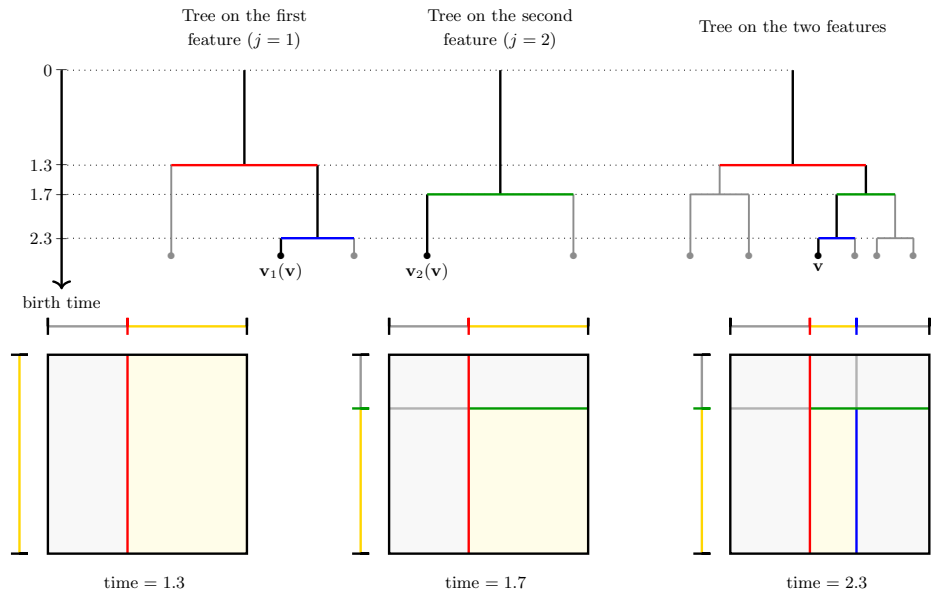
\begin{figure}  
  \centering  
  \scalebox{.9}{\begin{tikzpicture}[scale=0.54, every node/.style={scale=0.8}]  

      \pgfmathsetmacro{\tpx}{10} 
      \pgfmathsetmacro{\tpy}{17} 
      \pgfmathsetmacro{\tax}{-5}
      \pgfmathsetmacro{\tay}{\tpy}
      \pgfmathsetmacro{\tbx}{2}
      \pgfmathsetmacro{\tby}{\tpy}

      \pgfmathsetmacro{\tsx}{-10} 
      \pgfmathsetmacro{\tsby}{\tpy}
      \pgfmathsetmacro{\tsey}{\tpy-6}

      \pgfmathsetmacro{\ttly}{\tpy + 0.9} 

      \pgfmathsetmacro{\incrx}{2} 
      \pgfmathsetmacro{\incrxx}{0.8}
      \pgfmathsetmacro{\incrxxx}{0.5}
      \pgfmathsetmacro{\incrxxa}{1}

      \pgfmathsetmacro{\incrya}{-2.6} 
      \pgfmathsetmacro{\incryya}{-4.6}
      \pgfmathsetmacro{\incryb}{-3.4}
      \pgfmathsetmacro{\incryend}{-5.1} 

      \coordinate (PE) at (\tpx,\tpy) ; 
      \coordinate (PC) at (\tpx,\tpy+\incrya) ;
      \coordinate (PL) at (\tpx-\incrx,\tpy+\incrya) ;
      \coordinate (PR) at (\tpx+\incrx,\tpy+\incrya) ;
      \coordinate (PRC) at (\tpx+\incrx,\tpy+\incryb) ;      
      \coordinate (PRL) at (\tpx+\incrx-\incrxx,\tpy+\incryb) ;      
      \coordinate (PRR) at (\tpx+\incrx+\incrxx,\tpy+\incryb) ;
      \coordinate (PLC) at (\tpx-\incrx,\tpy+\incryb) ; 
      \coordinate (PLL) at (\tpx-\incrx-\incrxx,\tpy+\incryb) ;      
      \coordinate (PLR) at (\tpx-\incrx+\incrxx,\tpy+\incryb) ;
      \coordinate (PRLC) at (\tpx+\incrx-\incrxx,\tpy+\incryya) ;
      \coordinate (PRLL) at (\tpx+\incrx-\incrxx-\incrxxx,\tpy+\incryya) ;      
      \coordinate (PRLR) at (\tpx+\incrx-\incrxx+\incrxxx,\tpy+\incryya) ;
      \coordinate (PRRC) at (\tpx+\incrx+\incrxx,\tpy+\incryya) ;      
      \coordinate (PRRL) at (\tpx+\incrx+\incrxx-\incrxxx,\tpy+\incryya) ;      
      \coordinate (PRRR) at (\tpx+\incrx+\incrxx+\incrxxx,\tpy+\incryya) ;
      
      \coordinate (AE) at (\tax,\tay) ; 
      \coordinate (AC) at (\tax,\tay+\incrya) ;
      \coordinate (AL) at (\tax-\incrx,\tay+\incrya) ;
      \coordinate (AR) at (\tax+\incrx,\tay+\incrya) ;
      \coordinate (ARC) at (\tax+\incrx,\tay+\incryya) ;
      \coordinate (ARL) at (\tax+\incrx-\incrxxa,\tay+\incryya) ;
      \coordinate (ARR) at (\tax+\incrx+\incrxxa,\tay+\incryya) ;

      \coordinate (BE) at (\tbx,\tby) ; 
      \coordinate (BC) at (\tbx,\tby+\incryb) ;
      \coordinate (BL) at (\tbx-\incrx,\tby+\incryb) ;
      \coordinate (BR) at (\tbx+\incrx,\tby+\incryb) ;

      \draw (\tsx,\tsby) node {$\mathbf{-}$} [->,line width=1pt] -- (\tsx,\tsey) node [below] {birth time} ;
      \draw (\tsx,\tsby+\incrya) node {$\mathbf{-}$} ;
      \draw (\tsx,\tsby+\incryb) node {$\mathbf{-}$} ;
      \draw (\tsx,\tsby+\incryya) node {$\mathbf{-}$} ;
      \draw (\tsx,\tsby) node [left] {$0$\,} [dotted] -- (PE) ;
      \draw (\tsx,\tsby+\incrya) node [left] {$1.3$\,} [dotted] -- (PR) ;
      \draw (\tsx,\tsby+\incryb) node [left] {$1.7$\,} [dotted] -- (PRR) ;
      \draw (\tsx,\tsby+\incryya) node [left] {$2.3$\,} [dotted] -- (PRRR) ;

      \draw (\tax,\ttly+0.6) node {Tree on the first} ;
      \draw (\tax,\ttly-0.1) node {feature ($j=1$)} ;
      \draw (\tbx,\ttly+0.6) node {Tree on the second} ;      
      \draw (\tbx,\ttly-0.1) node {feature ($j=2$)} ;
      \draw (\tpx,\ttly+0.3) node {Tree on the two features} ;      

      
      \draw [line width=0.8pt,black!45] (PLL) -- (\tpx-\incrx-\incrxx,\tpy+\incryend) node {$\bullet$} ;
      \draw [line width=0.8pt,black!45] (PLR) -- (\tpx-\incrx+\incrxx,\tpy+\incryend) node {$\bullet$} ;
      \draw [line width=1pt,black] (PRLL) -- (\tpx+\incrx-\incrxx-\incrxxx,\tpy+\incryend) node {$\bullet$} node [below,black] {$\mathbf{v}$} ;
      \draw [line width=0.8pt,black!45] (PRLR) -- (\tpx+\incrx-\incrxx+\incrxxx,\tpy+\incryend) node {$\bullet$} ;
      \draw [line width=0.8pt,black!45] (PRRL) -- (\tpx+\incrx+\incrxx-\incrxxx,\tpy+\incryend) node {$\bullet$} ;
      \draw [line width=0.8pt,black!45] (PRRR) -- (\tpx+\incrx+\incrxx+\incrxxx,\tpy+\incryend) node {$\bullet$} ;
      \draw [line width=0.8pt,black!45] (PL) -- (PLC) ;
      \draw [line width=0.8pt,black!45] (PLL) -- (PLR) ;
      \draw [line width=1pt,black] (PRL) -- (PRLC) ;
      \draw [blue, line width=1pt] (PRLL) -- (PRLR) ;
      \draw [line width=0.8pt,black!45] (PRR) -- (PRRC) ;
      \draw [line width=0.8pt,black!45] (PRRL) -- (PRRR) ;
      \draw [line width=1pt,black] (PR) -- (PRC) ;
      \draw [mygreen, line width=1pt] (PRL) -- (PRR) ;
      \draw [line width=1pt,black] (PE) -- (PC) ;
      \draw [red, line width=1pt] (PL) -- (PR) ;
      %

      \draw [line width=0.8pt,black!45] (AL) -- (\tax-\incrx,\tay+\incryend) node {$\bullet$} ;
      \draw [black, line width=1pt] (ARL) -- (\tax+\incrx-\incrxxa,\tay+\incryend) node [below,black] {$\mathbf{v}_1(\mathbf{v})$} node {$\bullet$} ;
      \draw [line width=0.8pt,black!45] (ARR) -- (\tax+\incrx+\incrxxa,\tay+\incryend) node {$\bullet$} ;
      \draw [black, line width=1pt] (AR) -- (ARC) ;
      \draw [blue, line width=1pt] (ARL) -- (ARR) ;
      \draw [black, line width=1pt] (AE) -- (AC) ;
      \draw [red, line width=1pt] (AL) -- (AR) ;

      \draw [black, line width=1pt] (BL) -- (\tbx-\incrx,\tby+\incryend) node [below,black] {$\mathbf{v}_2 (\mathbf{v})$} node {$\bullet$} ;
      \draw [line width=0.8pt,black!45] (BR) -- (\tbx+\incrx,\tby+\incryend) node {$\bullet$} ;
      \draw [black, line width=1pt] (BE) -- (BC) ;
      \draw [mygreen, line width=1pt] (BL) -- (BR) ;

\end{tikzpicture}

}  
  \vspace{7pt}  
  \scalebox{.9}{\begin{tikzpicture}[scale=0.42, every node/.style={scale=0.8}]  

      \pgfmathsetmacro{\xshift}{12} 
      \pgfmathsetmacro{\Bli}{0}      
      \pgfmathsetmacro{\Bri}{7}
      \pgfmathsetmacro{\Bdi}{0}
      \pgfmathsetmacro{\Bui}{7}

      \pgfmathsetmacro{\Bl}{\Bli}      
      \pgfmathsetmacro{\Br}{\Bri}
      \pgfmathsetmacro{\Bd}{\Bdi}
      \pgfmathsetmacro{\Bu}{\Bui}

      \pgfmathsetmacro{\Buu}{\Bu+1}
      \pgfmathsetmacro{\Bll}{\Bl-1}

      \pgfmathsetmacro{\lg}{0.3} 

      \pgfmathsetmacro{\randxa}{0.4} 
      \pgfmathsetmacro{\randya}{0.73}
      \pgfmathsetmacro{\randxb}{0.39}
      \pgfmathsetmacro{\randyb}{0.43}
      \pgfmathsetmacro{\Xa}{(1-\randxa)*\Bl+\randxa*\Br}
      \pgfmathsetmacro{\Ya}{(1-\randya)*\Bd+\randya*\Bu}
      \pgfmathsetmacro{\Xb}{(1-\randxb)*\Xa+\randxb*\Br}
      \pgfmathsetmacro{\Xpoint}{0.5*\Xa+0.5*\Xb}
      \pgfmathsetmacro{\Ypoint}{0.6*\Bd+0.4*\Ya}      

      \draw (0.5*\Bl+0.5*\Br-0.1,\Bd-1) node {time $= 1.3$} ;      

      \fill [black!3] (\Bl,\Bd) rectangle (\Br,\Bu) ;
      \fill [yellow!10] (\Xa,\Bd) rectangle (\Br,\Bu) ;

      \draw [line width=1pt,black!40] (\Bl,\Buu) -- (\Xa,\Buu) ;
      \draw [line width=1pt,mygold] (\Xa,\Buu) -- (\Br,\Buu) ;
      \draw [line width=1pt,mygold] (\Bll,\Bd) -- (\Bll,\Bu) ;

      
      \draw [line width=1pt] (\Bl, \Buu) -- ++(0,\lg) -- ++(0,-2*\lg) ;
      \draw [line width=1pt] (\Br, \Buu) -- ++(0,\lg) -- ++(0,-2*\lg) ;
      \draw [red,line width=1pt] (\Xa, \Buu) -- ++(0,\lg) -- ++(0,-2*\lg) ;    

      \draw [line width=1pt] (\Bll, \Bd) -- ++(\lg,0) -- ++(-2*\lg,0) ;
      \draw [line width=1pt] (\Bll, \Bu) -- ++(\lg,0) -- ++(-2*\lg,0) ;      

      \draw [red,line width=1pt] (\Xa, \Bd) -- (\Xa, \Bu) ;

      \draw [line width=1pt] (\Bl,\Bd) rectangle (\Br,\Bu) ;

      \pgfmathsetmacro{\Bl}{\Bli+\xshift}      
      \pgfmathsetmacro{\Br}{\Bri+\xshift}

      \pgfmathsetmacro{\Buu}{\Bu+1}
      \pgfmathsetmacro{\Bll}{\Bl-1}

      \pgfmathsetmacro{\lg}{0.3} 

      \pgfmathsetmacro{\randxa}{0.4} 
      \pgfmathsetmacro{\randya}{0.73}
      \pgfmathsetmacro{\randxb}{0.39}
      \pgfmathsetmacro{\randyb}{0.43}
      \pgfmathsetmacro{\Xa}{(1-\randxa)*\Bl+\randxa*\Br}
      \pgfmathsetmacro{\Ya}{(1-\randya)*\Bd+\randya*\Bu}
      \pgfmathsetmacro{\Xb}{(1-\randxb)*\Xa+\randxb*\Br}
      \pgfmathsetmacro{\Xpoint}{0.5*\Xa+0.5*\Xb}
      \pgfmathsetmacro{\Ypoint}{0.6*\Bd+0.4*\Ya}

      \draw (0.5*\Bl+0.5*\Br-0.1,\Bd-1) node {time $= 1.7$} ;      

      \fill [black!3] (\Bl,\Bd) rectangle (\Br,\Bu) ;
      \fill [yellow!10] (\Xa,\Bd) rectangle (\Br,\Ya) ;

      \draw [line width=1pt,black!40] (\Bl,\Buu) -- (\Xa,\Buu) ;      
      \draw [line width=1pt,black!40] (\Bll,\Ya) -- (\Bll,\Bu) ;      
      \draw [line width=1pt,mygold] (\Bll,\Bd) -- (\Bll,\Ya) ;
      \draw [line width=1pt,mygold] (\Xa,\Buu) -- (\Br,\Buu) ;      

      
      \draw [line width=1pt] (\Bl, \Buu) -- ++(0,\lg) -- ++(0,-2*\lg) ;
      \draw [line width=1pt] (\Br, \Buu) -- ++(0,\lg) -- ++(0,-2*\lg) ;
      \draw [red,line width=1pt] (\Xa, \Buu) -- ++(0,\lg) -- ++(0,-2*\lg) ;     

      \draw [line width=1pt] (\Bll, \Bd) -- ++(\lg,0) -- ++(-2*\lg,0) ;
      \draw [line width=1pt] (\Bll, \Bu) -- ++(\lg,0) -- ++(-2*\lg,0) ;
      \draw [mygreen,line width=1pt] (\Bll, \Ya) -- ++(\lg,0) -- ++(-2*\lg,0) ;
      
      \draw [line width=1pt,black!30] (\Bl, \Ya) -- (\Xa, \Ya) ;
      \draw [mygreen,line width=1pt] (\Xa, \Ya) -- (\Br, \Ya) ;
      \draw [red,line width=1pt] (\Xa, \Bd) -- (\Xa, \Bu) ;      

      \draw [line width=1pt] (\Bl,\Bd) rectangle (\Br,\Bu) ;      

      \pgfmathsetmacro{\Bl}{\Bli+2*\xshift}      
      \pgfmathsetmacro{\Br}{\Bri+2*\xshift}

      \pgfmathsetmacro{\Buu}{\Bu+1}
      \pgfmathsetmacro{\Bll}{\Bl-1}

      \pgfmathsetmacro{\lg}{0.3} 

      \pgfmathsetmacro{\randxa}{0.4} 
      \pgfmathsetmacro{\randya}{0.73}
      \pgfmathsetmacro{\randxb}{0.39}
      \pgfmathsetmacro{\randyb}{0.43}
      \pgfmathsetmacro{\Xa}{(1-\randxa)*\Bl+\randxa*\Br}
      \pgfmathsetmacro{\Ya}{(1-\randya)*\Bd+\randya*\Bu}
      \pgfmathsetmacro{\Xb}{(1-\randxb)*\Xa+\randxb*\Br}
      \pgfmathsetmacro{\Xpoint}{0.5*\Xa+0.5*\Xb}
      \pgfmathsetmacro{\Ypoint}{0.6*\Bd+0.4*\Ya}

      \draw (0.5*\Bl+0.5*\Br-0.1,\Bd-1) node {time $= 2.3$} ;      
      
      \fill [black!3] (\Bl,\Bd) rectangle (\Br,\Bu) ;
      \fill [yellow!10] (\Xa,\Bd) rectangle (\Xb,\Ya) ;

      \draw [line width=1pt,black!40] (\Bl,\Buu) -- (\Xa,\Buu) ;
      \draw [line width=1pt,black!40] (\Xb,\Buu) -- (\Br,\Buu) ;
      \draw [line width=1pt,black!40] (\Bll,\Ya) -- (\Bll,\Bu) ;
      \draw [line width=1pt,mygold] (\Bll,\Bd) -- (\Bll,\Ya) ;
      \draw [line width=1pt,mygold] (\Xa,\Buu) -- (\Xb,\Buu) ;      

      
      \draw [line width=1pt] (\Bl, \Buu) -- ++(0,\lg) -- ++(0,-2*\lg) ;
      \draw [line width=1pt] (\Br, \Buu) -- ++(0,\lg) -- ++(0,-2*\lg) ;
      \draw [red,line width=1pt] (\Xa, \Buu) -- ++(0,\lg) -- ++(0,-2*\lg) ;
      \draw [blue,line width=1pt] (\Xb, \Buu) -- ++(0,\lg) -- ++(0,-2*\lg) ;

      \draw [line width=1pt] (\Bll, \Bd) -- ++(\lg,0) -- ++(-2*\lg,0) ;
      \draw [line width=1pt] (\Bll, \Bu) -- ++(\lg,0) -- ++(-2*\lg,0) ;
      \draw [mygreen,line width=1pt] (\Bll, \Ya) -- ++(\lg,0) -- ++(-2*\lg,0) ;

      \draw [line width=1pt,black!30] (\Xb, \Ya) -- (\Xb, \Bu) ;
      \draw [blue,line width=1pt] (\Xb, \Bd) -- (\Xb, \Ya) ;
      \draw [line width=1pt,black!30] (\Bl, \Ya) -- (\Xa, \Ya) ;
      \draw [mygreen,line width=1pt] (\Xa, \Ya) -- (\Br, \Ya) ;
      \draw [red,line width=1pt] (\Xa, \Bd) -- (\Xa, \Bu) ;

      \draw [line width=1pt] (\Bl,\Bd) rectangle (\Br,\Bu) ;

\end{tikzpicture}

}  
  \caption{Modified construction in dimension two. At the top, from left to right\textup: 
  trees associated to partitions $\mondrian'^{1}, \mondrian'^{2}$ and $\wt\mondrian$ respectively. At the bottom, from left to right: successive splits in $\wt\mondrian$ leading to the leaf $\node$ \textup(depicted in yellow\textup).}
  \label{fig:modified-proof}  
\end{figure}

\paragraph{Computation of $\E [\wt K_\lambda]$}

Now, it can be seen that the partition $\wt\mondrian_\lambda$ is a rectangular grid which is the ``product'' of the partitions $\mondrian'^j$ of the intervals $C^j$, $1\leq j \leq d$.
Indeed, let $x \in [0, 1]^d$, and let $\wt \cell_\lambda (x)$ be the cell in $\wt \mondrian_\lambda$ that contains $x$; we need to show that $\wt \cell_\lambda (x) = \prod_{j=1}^d \cell'^j_\lambda (x)$, where $\cell'^j_\lambda (x)$ is the subinterval of $C^j$ in the partition $\mondrian'^j$ that contains $x_j$.
The proof proceeds in several steps:
\begin{itemize}  
  \item First, Equation~\eqref{eq:rec-modified-mondrian} shows that, for every node $\node$, we have $\wt \cell_\node = \prod_{1\leq j \leq d} \cell'^j_{\node_j (\node)}$, since the successive splits on the $j$-th coordinate of $\wt\cell_\node$ are precisely the ones of $\cell'^j_{\node_j (\node)}$.  
  \item Second, 
  it follows from~\eqref{eq:rec-modified-mondrian} that $\wt T_\node = \min_{1\leq j \leq d} T'^j_{\node_j (\node)}$; also, since the cell $\cell_\node$ is formed when its last split is performed, $\wt \birth_\node = \max_{1\leq j \leq d} \birth'^j_{\node_j (\node)}$.  
  \item Let $\wt\node$ be the node such that $\wt \cell_{\wt \node} = \wt\cell_\lambda (x)$, and $\node'^j$ be such that $\cell'^j_{\node'^j} = \cell'^j_\lambda (x_j)$.  
  By the first point, it suffices to show that $\node_j(\wt \node) = \node'_j$ for $1\leq j \leq d$.  
  \item Observe that $\wt\node$ (resp. $\node'_j$) is characterized by the fact that $x \in \wt\cell_{\wt \node}$ and $\wt \birth_{\wt \node} \leq \lambda < \wt T_{\wt \node}$ (resp. $x_j \in \cell'^j_{\node'^j}$ and $\birth'^j_{\node'^j} \leq \lambda < T'^j_{\node'^j}$).  
  But since $\wt \cell_{\wt \node} = \prod_{1\leq j \leq d} \cell'^j_{\node_j (\wt\node)}$ (first point), $x \in \wt \cell_{\wt \node}$ implies $x_j \in \cell'^j_{\node_j (\wt\node)}$.  
  Likewise, since $\wt \birth_{\wt \node} = \max_{1\leq j \leq d} \birth'^j_{\node_j (\wt \node)}$ and  $\wt T_{\wt \node} = \min_{1\leq j \leq d} T'^j_{\node_j (\wt \node)}$ (second point), $\wt \birth_{\wt \node} \leq \lambda < \wt T_{\wt \node}$ implies $\birth'^j_{\node_j (\wt \node)} \leq \lambda < T'^j_{\node_j (\wt \node)}$.  
  Since these properties characterize $\node'^j$, we have $\node_j(\wt \node) = \node'_j$, which concludes the proof.  
\end{itemize}

Hence, the partition $\wt\mondrian_\lambda$ is the product of the partitions $\mondrian'^j = \Phi_{C^j} ((E_\node^j, $ $U_\node^j)_\node)_\lambda$ of the intervals $C^j$, $1\leq j \leq d$, which are independent Mondrians distributed as $\MP (\lambda, C^j)$.
By Fact~\ref{fac:mondrian-poisson}, the splits of the Mondrian partition $\MP(\lambda, C^j)$ are distributed as a Poisson point process on $C^j$ of intensity $\lambda$, so that the expected number of cells in such a partition is $1 + \lambda |C^j|$.
Since $\wt \mondrian_\lambda$ is a ``product'' of such independent partitions, we have:
\begin{equation}  
\label{eq:expected-ncells-modified}  
\E [ \wt K_\lambda ]  
= \prod_{j=1}^d (1+ \lambda |C^j|)  
\, .  
\end{equation}

\paragraph{Equality of $\E [K_\lambda]$ and $\E [\wt K_\lambda]$}

In order to establish Proposition~\ref{prop:number-cells}, it is thus sufficient to prove that $\E [ K_\lambda ] = \E [\wt K_\lambda]$.
First, note that, since the number of cells in a partition is one plus the number of splits (each split increases the number of cells by one)
\begin{equation*}  
K_\lambda =  
1 + \sum_{\node \in \{ 0,1 \}^*} \bm 1 ( T_\node \leq \lambda )  
\end{equation*}
so that we have, respectively,
\begin{align}  
\label{eq:ncells-proof-12}  
\E [K_\lambda]  
&= 1 + \sum_{\node \in \{ 0,1 \}^*} \P ( T_\node \leq \lambda )  \\
\E [ \wt K_\lambda ]  
&= 1 + \sum_{\node \in \{ 0,1 \}^*} \P ( \wt T_\node \leq \lambda )
\label{eq:ncells-proof-13}  
\, .  
\end{align}
Hence, it suffices to show that $\P (T_\node \leq \lambda) = \P(\wt T_\node\leq \lambda)$ for every $\node \in \{0, 1\}^*$ and $\lambda \geq 0$, \ie that $T_\node$ and $\wt T_\node$ have the same distribution for every~$\node$.

In order to establish this, we show that, for every $\node \in \{0, 1\}^*$, the conditional distribution of $(\wt T_\node, \wt J_\node, \wt \Cut_\node)$ given $\wt \F_\node = \sigma ((\wt T_{\node'}, \wt J_{\node'}, \wt \Cut_{\node'}), \node' \prefix \node)$ has the same form as the conditional distribution of $(T_\node, J_\node, \Cut_\node)$ given $\F_\node = \sigma ((T_{\node'}, J_{\node'}, \Cut_{\node'}), \node' \prefix \node)$,
in the sense that there exits a family of conditional distributions $(\Psi_\node)_\node$ such that, for every $\node$, the conditional distribution of $(T_\node, J_\node, \Cut_\node)$ given $\F_\node$ is $\Psi_\node ( \cdot | (T_{\node'}, J_{\node'}, \Cut_{\node'}), \node' \prefix \node)$ and the conditional distribution of $(\wt T_\node, \wt J_\node, \wt \Cut_\node)$ given $\wt \F_\node$ is $\Psi_\node( \cdot | (\wt T_{\node'}, \wt J_{\node'}, \wt \Cut_{\node'}), \node' \prefix \node)$.

First, recall that the variables $(E_{\node'}^j, U_{\node'}^j)_{\node' \in \{ 0, 1\}^*, 1\leq j\leq d}$ are independent, so $(E_{\node}^j, U_{\node}^j)_{1\leq j \leq d}$ is independent from $\F_\node$.
Hence, conditionally on $\F_\node$, $E_{\node}^j, U_{\node}^j$, $1\leq j \leq d$ are independent with $E_{\node}^j \sim \expdist(1)$ and $U_\node^j\sim \uniformdist([0, 1])$.
Also, recall that if $T_1, \dots, T_d$ are independent exponential random variables of intensities $\lambda_1, \dots, \lambda_d$, and if $T = \min_{1\leq j \leq d} T_j$ and $J = \argmin_{1\leq j \leq d} T_j$, then $\P(J=j) = \lambda_j / \sum_{j'=1}^d \lambda_{j'}$, $T\sim \expdist(\sum_{j=1}^d \lambda_j)$ and $J$ and $T$ are independent.
Hence, conditionally on $\F_\node$, $T_\node - \birth_\node =  \min_{1\leq j \leq d} E_\node^j/|\cell_\node^j| \sim \expdist(\sum_{j=1}^d|\cell_\node^j|) = \expdist(|\cell_\node|)$, $J_\node := \argmin_{1\leq j \leq d} E_\node^j / |\cell_\node^j|$ equals $j$ with probability $|\cell_\node^j|/|\cell_\node|$, $T_\node, J_\node$ are independent and $(\Cut_\node | T_\node, J_\node) \sim \uniformdist (\cell_\node^{J_\node})$.

Now consider the conditional distribution of $(\wt T_\node, \wt J_\node, \wt \Cut_\node)$ given $\wt \F_\node$.
Let $(\node_v)_{v\in \N}$ be a path in $\{0, 1\}^*$ from the root: $\node_0 := \root$, $\node_{v+1}$ is a child of $\node_v$ for $v \in \N$, and $\node_v \prefixeq \node$ for $0 \leq v \leq \depth (\node)$.
Define for $v \in \N$, $E_v^j = E_{\node_v}^j$ and $U_v^j = U_{\node_v}^j$ if $\node_{v+1}$ is the left child of $\node_v$, and $1-U_{\node_v}^j$ otherwise.
Then, the variables $(E_v^j, U_v^j)_{v \in \N, 1\leq j \leq d}$ are independent, with $E_v^j \sim \expdist(1)$, $U_v^j \sim \uniformdist ([0, 1])$, so that the following Lemma applies.
\begin{lemma}
  \label{lem:cond-mod-exp}
  Let $(E_v^j, U_v^j)_{v \in \N^\star, 1\leq j \leq d}$ be a family of independent random variables\textup,
   with $U_v^j \sim \uniformdist([0, 1])$ and $E_v^j \sim \expdist(1)$.
  Let $a_1, \dots, a_d >0$.
  For $1\leq j \leq d$\textup, define the sequence $(T_v^j, L_v^j)_{v \in \N}$ as follows\textup:
  \begin{itemize}
    \item $L_0^j = a_j 
    $, $T_0^j = \frac{E_0^j}{a_j}$;
    \item for $v\in \N$, $L_{v+1}^j = U_{v}^j L_v^j 
    $, $T_{v+1}^j = T_v^j + \frac{E_{v+1}^j}{L_{v+1}^j}$.
  \end{itemize}
  Define recursively the variables
  $\wt V_v^j$ $(v\in \N, 1\leq j \leq d)$ as well as 
  $\wt J_v, \wt T_v, \wt U_v$ $(v \in \N)$
  as follows\textup:
  \begin{itemize}
    \item $\wt V_0^j = 0$ for $j=1, \dots, d$.
    \item for $v\in \N$, given $\wt V_{v}^j$ $(1\leq j \leq d)$\textup, denoting $\wt T_v^j = T_{\wt V_v^j}^j$ and $\wt U_v^j = U_{\wt V_v^j}^j$\textup, set
    \begin{align}
    \label{eq:ncells-proof-9}
    \wt J_v = \argmin_{1\leq j \leq d} \wt T_v^j
    , \quad
    & \wt T_v = \min_{1\leq j \leq d} \wt T_v^j = \wt T_v^{\wt J_v}, \wt U_v = \wt U_v^{\wt J_v}, \nonumber \\
    \textrm{and} \quad  & \wt V_{v+1}^j = \wt V_{v}^j + \bm 1 ( \wt J_v = j ).
    \end{align}    
  \end{itemize}
  Then\textup, the conditional distribution of $(\wt J_v, \wt T_v, \wt U_v)$ given 
  $\F_v = \sigma ((\wt J_{v'}, \wt T_{v'}, \wt U_{v'}),$ $ {0\leq v'<v})$ is the following 
  \textup(denoting $\wt L_v^j = L_{\wt V_v^j}^j$\textup)\textup:
  \begin{itemize}
    \item  $\wt J_v, \wt T_v, \wt U_v$ are independent\textup,
    
    \item $\P ( \wt J_v = j \cond \F_v ) = \wt L_v^j / (\sum_{j'=1}^d \wt L_v^{j'})$,
    
    \item $\wt T_v - \wt T_{v-1} \sim \expdist(\sum_{j=1}^d \wt L_v^j)$ 
    \textup(with the convention $\wt T_{-1} = 0$\textup) and $\wt U_v \sim \uniformdist([0, 1])$.
  \end{itemize}
\end{lemma}

In addition, note that, with the notations of Lemma~\ref{lem:cond-mod-exp}, a simple induction 
shows that $\wt J_v = \wt J_{\node_v}$, $\wt T_v = \wt T_{\node_v}$, $\wt U_v = \wt U_{\node_v}$ and $L_v^j = |\wt\cell^j_{\node_v}|$, so that $\F_v = \F_{\node_v}$.
Applying Lemma~\ref{lem:cond-mod-exp} for $v = \depth(\node)$ (so that $\node_v = \node$) therefore gives the following: conditionally on $\F_\node$, the  variables $\wt T_\node, \wt J_\node, \wt U_\node$ are independent, $\wt T_\node - \wt \birth_\node \sim \expdist (|\wt\cell^j_\node|)$, $\P ( \wt J_\node = j \cond \F_\node ) = | \wt \cell_{\node}^j | / \big( \sum_{j'=1}^d | \wt \cell_{\node}^j | \big)$ and $\wt U_\node \sim \uniformdist([0, 1])$, so that
$(\wt\Cut_{\node} | \F_\node,\wt  T_{\node}, \wt J_{\node}) \sim \uniformdist (\wt\cell_\node^{\wt J_{\node}})$.
Hence, we have proven that, for every $\node$, the conditional distribution of $(T_\node, J_\node, \Cut_\node)$ given $\F_\node$ is the same as that of $(\wt T_\node, \wt J_\node, \wt \Cut_\node)$ given $\wt\F_\node$.
By induction on $\node$, since $\F_\root = \wt\F_\root$ is the trivial $\sigma$-algebra, this shows that $T_\node$ and $\wt T_\node$ have the same distribution for every $\node$.
Plugging this into~\eqref{eq:ncells-proof-12} and~\eqref{eq:ncells-proof-13} and combining it with~\eqref{eq:expected-ncells-modified} completes the proof of Proposition~\ref{prop:number-cells}. $\hfill \qed$

\begin{proof}[Proof of Lemma~\ref{lem:cond-mod-exp}]
  We show by induction on $v \in \N$ the following property: conditionally on $\F_\node$,  $(\wt T_v^j, \wt U_v^j)_{1\leq j \leq d}$ are independent, $\wt T_v^j - \wt T_{v-1} \sim \expdist (L_v^j)$ and $\wt U_v^j \sim \uniformdist([0, 1])$.
  
  \begin{description}
    \item[Initialization] For $v=0$ (with $\F_0$ the trivial $\sigma$-algebra), since $\wt V_0^j = 0$ we have $\wt T_0^j = E_0^j / a_j \sim \expdist(a_j) = \expdist(L_0^j)$, $\wt U_0^j = U_0^j \sim \uniformdist([0, 1])$ and these random variables are independent.
    \item[Inductive step] Let $v \in \N$, and assume the property is true up to step $v$.
    Conditionally on $\F_{v+1}$, \ie on $\F_v, \wt T_v, \wt J_v, \wt U_v$,
    we have:
    \begin{itemize}
      \item for $j \neq \wt J_v$, the variables $\wt T_{v+1}^j - \wt T_{v-1} = \wt T_v^j - \wt T_{v-1}$ are independent $\expdist(\wt L_{v}^j) = \expdist(\wt L_{v+1}^j)$ random variables (when conditioned only on $\F_v$, by the induction hypothesis), conditioned on $\wt T_{v+1}^j - \wt T_{v-1} \geq \wt T_{v} - \wt T_{v-1}$, so
      by the memory-less property of exponential random variables
      $\wt T_{v+1}^j - \wt T_{v} = (\wt T_{v+1}^j - \wt T_{v-1}) - (\wt T_{v} - \wt T_{v-1}) \sim \expdist(\wt L_{v+1}^j)$ (and those variables are independent).
      \item for $j \neq \wt J_v$, the variables $\wt U_{v+1}^j = \wt U_{v}^j$ are independent $\uniformdist([0, 1])$ random variables (conditionally on $\F_v$), conditioned on the independent variables $\wt T_v, \wt J_v, \wt U_v$, 
      so they remain independent $\uniformdist([0, 1])$ random variables.
      \item $(\wt T_{v+1}^{\wt J_v} - \wt T_{v}, \wt U_{v+1}^{\wt J_v}) = ( E_{\wt V_{v+1}^{\wt J_v}}^{\wt J_v} / \wt L_{v+1}^{\wt J_v} , U_{\wt V_{v+1}^{\wt J_v}}^{\wt J_v})$ is distributed, conditionally on $\F_{v+1}$, \ie on $\wt J_v, \wt T_v, \wt V_{v+1}^{\wt J_v}, \wt L_{v+1}^{\wt J_v}$, as $\expdist(\wt L_{v+1}^{\wt J_v}) \otimes \uniformdist([0, 1])$,
      and independent of $(\wt T_{v+1}^{j}, \wt U_{v+1}^{j})_{j\neq \wt J_v}$.
    \end{itemize}
    This completes the proof by induction.
  \end{description}
  
  Let $v\in \N$.
  We have established that, conditionally on $\F_v$, the variables $(\wt T_v^j, \wt U_v^j)_{1\leq j \leq d}$ are independent, with $\wt T_v^j - \wt T_{v-1} \sim \expdist(\wt L_v^j)$ and $\wt U_v^j \sim \uniformdist([0, 1])$.
  In particular, conditionally on $\F_v$, $\wt U_v$ is independent from $(\wt J_v, \wt T_v)$, $\wt U_v \sim \uniformdist([0, 1])$, and (by the property of the minimum of independent exponential random variables)
  $\wt J_v$ is independent of $\wt T_v$, $\wt T_v \sim \expdist(\sum_{j=1}^d \wt L_v^j)$ and $\P ( \wt J_v = j \cond \F_v ) = \wt L_v^j / (\sum_{j'=1}^d \wt L_v^{j'})$.
  This concludes the proof of Lemma~\ref{lem:cond-mod-exp}.
\end{proof}

\section{Proof of Theorem~\ref{thm:consistency-mondrian}}
\label{ap:consistency}

Recall that a Mondrian Forest estimate with lifetime parameter $\lambda$, is defined, for all $x\in [0,1]^d$, by
\begin{equation*}
  \wh f_{\lambda,n, M}(x) = \wh f_{\lambda,n, M}(x, \mondrian_{\lambda, M}) = \frac 1M \sum_{m=1}^M \wh f_{\lambda,n}^{(m)}(x, \mondrian_\lambda^{(m)}),
\end{equation*}
where $\wh f_{\lambda,n}^{(m)}(x, \mondrian_\lambda^{(m)})$ denotes the Mondrian Tree based on the random partition $\mondrian_\lambda^{(m)}$ and $ \mondrian_{\lambda, M} = (\mondrian_{\lambda}^{(1)}, \dots, \mondrian_\lambda^{(M)})$. To ease notation, we will write $\wh f_{\lambda,n}^{(m)}(x)$ instead of $\wh f_{\lambda,n}^{(m)}(x, $ $ \mondrian_\lambda^{(m)})$.
First, note that, by Jensen's inequality, 
\begin{align*}
\risk (\wh f_{\lambda,n, M})
&= \E_{(X, \mondrian_{\lambda, M})} [( \wh f_{\lambda,n, M}(x, \mondrian_{\lambda, M}) - f(X) )^2] \\
& \leq \frac 1M \sum_{m=1}^M \E_{(X,  \mondrian_\lambda^{(m)})} [(\wh f_{\lambda,n}^{(m)}(X, \mondrian_\lambda^{(m)}) - f(X) )^2] \\
& \leq \E_{(X, \mondrian_\lambda^{(1)})} [(\wh f_{\lambda,n}^{(1)}(X, \mondrian_\lambda^{(1)}) - f(X) )^2] \, ,
\end{align*}
since each Mondrian tree has the same distribution. 
Therefore, it is sufficient to prove that a single Mondrian tree is consistent. 
Now, since Mondrian partitions are independent of the dataset $\dataset_n$, we can apply Theorem~4.2 from~\cite{gyorfi2002nonparametric}, which states that a Mondrian tree estimate is consistent if
\begin{enumerate}
  \item[$(i)$] $D_{\lambda}(X) \to 0$ in probability, as $n \to \infty$,  
  \item[$(ii)$] $K_{\lambda}/n \to \infty$ in probability, as $n \to \infty$,  
\end{enumerate}
where $D_{\lambda}(X)$ is the diameter of the cell of the Mondrian tree that contains $X$, and $K_{\lambda}$ is the number of cells in the Mondrian tree. 
Note that the initial assumptions in Theorem~4.2 in~\cite{gyorfi2002nonparametric} contains deterministic convergence, but can be relaxed to convergences in probability by a close inspection of the proof. In the sequel, we prove that an individual Mondrian tree satisfies $(i)$ and $(ii)$ which will conclude the proof.
To prove $(i)$, just note that, according to Corollary~\ref{lem:diameter},
\begin{align*}
\E [D_{\lambda}(X)^2] = \E [ \E [D_{\lambda}(X)^2 \cond X]]  \leq \frac{4d}{\lambda^2},
\end{align*}
which tends to zero, since $\lambda = \lambda_n \to \infty$, as $n \to \infty$. Thus, condition~$(i)$ holds. 
Now, to prove $(ii)$, observe that
\begin{align*}
\E\Big[\frac{K_{\lambda}}{n}\Big] = \frac{(1+\lambda)^d}{n},
\end{align*}
which tends to zero since $\lambda_n^d/n \to 0$ by assumption, as $n \to \infty$. 
This concludes the proof of Theorem~\ref{thm:consistency-mondrian}. $\hfill \qed$

\section{Proof of Proposition~\ref{prop:lower_bound_tree}}
\label{sec:lower-bound-risk}

Let $\mondrian_\lambda^{(1)}$ be the Mondrian partition of $[0, 1]$ used to construct the randomized estimator $\wh f_{\lambda, n}^{(1)}$.
Denote by $\bar f_\lambda^{(1)} $ the random function $\bar f_\lambda^{(1)} (x) = \E_X \left[ f (X) \cond X \in \cell_\lambda (x) \right]$,
and define $\wt f_\lambda (x) = \E \left[ \bar f_\lambda^{(1)} (x) \right]$ (which is deterministic). For the seek of clarity, we will drop the exponent ``$(1)$'' in all notations, keeping in mind that we consider only one particular Mondrian partition, whose associated Mondrian Tree estimate is denoted by $\wh f_{\lambda,n}$. 
Recall the bias-variance decomposition~\eqref{eq:bias-variance} for Mondrian trees:
\begin{equation}
\label{eq:bias-variance-prooflemma}
R (\wh f_{\lambda,n}^{(1)}) = \E \big[ (f (X) - \bar f_{\lambda} (X))^2 \big] +      
\E \big[ (\bar f_{\lambda} (X) - \wh f_{\lambda,n}^{(1)} (X))^2 \big] \, .      
\end{equation}
We will provide lower bounds for the first term (the bias, depending on $\lambda$) and the second (the variance, depending on both $\lambda$ and $n$), which will lead to the stated lower bound on the risk, valid for every value of $\lambda$.

\paragraph{Lower bound on the bias}

As we will see, the point-wise bias $\E [ (\bar f_\lambda (x) - f(x))^2 ]$ can be computed explicitly given our assumptions.
Let $x\in [0, 1]$.
Since $\wt f_\lambda (x) = \E [\bar f_\lambda (x)]$, we have
\begin{equation}
\label{eq:biasvar-bias}
\E \left[ (\bar f_\lambda (x) - f(x))^2 \right]
= \Var (\bar f_\lambda (x)) + (\wt f_\lambda (x) - f(x))^2 \, .
\end{equation}
By Proposition~\ref{prop:cell-distribution}, the cell of $x$ in $\mondrian_\lambda$ can be written as $\cell_\lambda (x) = [L_\lambda (x), R_\lambda (x)]$, with $L_\lambda (x) = (x - \lambda^{-1} E_L) \vee 0$ and $R_\lambda (x) = (x + \lambda^{-1} E_R) \wedge 1$, where $E_L,E_R$ are two independent $\expdist(1)$ random variables.
Now, since $X \sim \uniformdist([0 ,1])$ and $f(u) = 1+u$,
\begin{equation*}
\bar f_\lambda (x) = \frac{1}{R_\lambda (x) - L_\lambda(x)} \int_{L_\lambda (x)}^{R_\lambda(x)} (1+u) \di u
= 1 + \frac{L_\lambda (x) + R_\lambda (x)}{2}
\, .
\end{equation*}
Since $L_\lambda (x)$ and $R_\lambda (x)$ are independent, we have
\begin{equation*}
\Var (\bar f_\lambda (x))  
= \frac{\Var (L_\lambda (x)) + \Var (R_\lambda (x))}{4}  .
\end{equation*}
In addition,
\begin{equation*}
\Var (R_\lambda (x))
= \Var \big( x + \lambda^{-1} [ E_R \wedge \lambda (1-x)] \big)
= \lambda^{-2} \Var ( E_R \wedge [\lambda (1-x)])
\end{equation*}
Now, if $E \sim \expdist(1)$ and $a \geq 0$, we have
\begin{align}
\label{eq:expect-trunc-exp}
\E [E \wedge a] &= \int_{0}^a u e^{-u} \di u + a \P (E \geq a) = 1 - e^{-a}  \\
\E [(E \wedge a)^2] &= \int_{0}^a u^2 e^{-u} \di u + a^2 \P (E \geq a)
= 2 \left( 1 - (a+1) e^{-a} \right) , \nonumber
\end{align}
so that
\begin{equation*}
\Var (E \wedge a) = \E [(E \wedge a)^2] - \E [E\wedge a]^2
= 1 - 2a e^{-a} - e^{-2a} .
\end{equation*}
The formula above gives the variances of $R_\lambda(x)$ and $L_\lambda(x)$ respectively:
\begin{align*}
\Var (R_\lambda (x))
&= \lambda^{-2} \big( 1 - 2 \lambda (1-x) e^{-\lambda (1-x)} - e^{-2\lambda (1-x)} \big) \\
\Var (L_\lambda (x))
&= \lambda^{-2} \big( 1 - 2 \lambda x e^{-\lambda x} - e^{-2\lambda x} \big) \, ,
\end{align*}
and thus
\begin{equation}
\label{eq:var-fbar}
\Var (\bar f_\lambda (x))
= \frac{1}{4 \lambda^{2}} \big( 2 - 2 \lambda x e^{-\lambda x} - 2 \lambda (1-x) e^{-\lambda (1-x)} - e^{-2\lambda x} - e^{-2\lambda (1-x)} \big) \, .
\end{equation}
In addition, the formula~\eqref{eq:expect-trunc-exp} yields
\begin{align*}
\E [R_\lambda (x)] &= x + \lambda^{-1} \big( 1 - e^{-\lambda (1-x)} \big) \\
\E [L_\lambda (x)] &= x - \lambda^{-1} \big( 1 - e^{-\lambda x} \big) \, ,
\end{align*}
and thus
\begin{equation}
\label{eq:ftilde}
\wt f_\lambda (x) = 1 + \frac{\E [L_\lambda (x)] + \E [R_\lambda (x)]}{2}
= 1 + x + \frac{1}{2 \lambda} \big( e^{-\lambda x} - e^{-\lambda (1-x)} \big) \, .
\end{equation}
Combining~\eqref{eq:var-fbar} and~\eqref{eq:ftilde} with the decomposition~\eqref{eq:biasvar-bias} gives
\begin{equation}
\label{eq:pointwise-bias-prooflemma}
\E \big[ \big( \bar f_\lambda (x) - f (x) \big)^2 \big]
= \frac{1}{2 \lambda^2} \left( 1 - \lambda x e^{-\lambda x} - \lambda (1-x) e^{-\lambda (1-x)} - e^{-\lambda} \right) .
\end{equation}
Integrating over $X$, we obtain
\begin{align}
& \E \left[ (\bar f_\lambda (X) - f(X))^2 \right] \nonumber \\
&= \frac{1}{2 \lambda^2} \left( 1 - \int_0^1 \lambda x e^{-\lambda x} \di x - \int_0^1 \lambda (1-x) e^{-\lambda (1-x)} \di x - e^{-\lambda} \right) \nonumber \\
&= \frac{1}{2 \lambda^2} \left( 1 - 2 \times \frac{1}{\lambda} \big( 1 - (\lambda +1) e^{-\lambda} \big) - e^{-\lambda} \right) \nonumber \\
&= \frac{1}{2 \lambda^2} \left( 1 - \frac{2}{\lambda} + e^{-\lambda} + \frac{2}{\lambda} e^{-\lambda} \right) . \label{eq:bias-prooflemma}
\end{align}
Now, note that the bias $\E [ (\bar f_\lambda (X) - f(X))^2 ]$ is positive for $\lambda \in \R_+^*$ (indeed, it is nonnegative, and non-zero since $f$ is not piecewise constant).
In addition, the expression~\eqref{eq:bias-prooflemma} shows that it is continuous in $\lambda$ on $\R_+^*$, and that it admits a limit $\frac{1}{12}$ as $\lambda \to 0$ (using the fact that $e^{-\lambda} = 1 - \lambda + \frac{\lambda^2}{2} - \frac{\lambda^3}{6} + o(\lambda^3)$).
Hence, the function $\lambda \mapsto \E [ (\bar f_\lambda (X) - f(X))^2 ]$ is positive and continuous on $\R_+$, so that it admits a minimum $C_1 >0$ on the compact interval $[0, 6]$.
In addition, the expression~\eqref{eq:bias-prooflemma} shows that for $\lambda \geq 6$, we have
\begin{equation}
\label{eq:lowerbound-bias-prooflemma}
\E \left[ (\bar f_\lambda (X) - f(X))^2 \right]
\geq \frac{1}{2\lambda^2} \left( 1 - \frac{2}{6} \right)
= \frac{1}{3 \lambda^2} \, .
\end{equation}

\paragraph{First lower bound on the variance}
We now turn to the task of bounding the variance from below.
In order to avoid restrictive conditions on $\lambda$, we will provide two separate lower bounds, valid in two different regimes.

Our first lower bound on the variance, valid for $\lambda \leq n/3$, controls the error of estimation of the optimal labels in nonempty cells.
It depends on $\sigma^2$, and is of order $\Theta \big( \sigma^2 \frac{\lambda}{n} \big)$.
We use a general bound on the variance of regressograms \cite[Proposition~2]{arlot2014purf_bias} (note that while this result is stated for a fixed number of cells, it can be adapted to a random number of cells by conditioning on $K_\lambda = k$ and then by averaging):
\begin{align}
\label{eq:lowerbound-variance}
& \E \left[ \big( \wh f_{\lambda,n} (X) - \wt f_\lambda (X)  \big)^2 \right] \nonumber \\
& \geq \frac{\sigma^2}{n} \bigg( \E \left[ K_\lambda \right] - 2 \E_{\mondrian_\lambda} \bigg[ \sum_{\node \in \leaves(\mondrian_\lambda)} \exp ( - n \P ( X \in \cell_\node ) ) \bigg] \bigg).
\end{align}
Now, recall that the splits defining $\mondrian_\lambda$ form a Poisson point process on $[0, 1]$ of intensity $\lambda \di x$ (Fact~\ref{fac:mondrian-poisson}).
In particular, the splits can be described as follows.
Let $(E_k)_{k \geq 1}$ be an \iid sequence of $\expdist(1)$ random variables, and $S_p := \sum_{k=1}^p E_k$ for $p \geq 0$.
Then, the (ordered) splits in $\mondrian_\lambda$ have the same distribution as $(\lambda^{-1} S_1, \dots, \lambda^{-1} S_{K_\lambda - 1})$, where $K_\lambda := 1 + \sup \{ p \geq 0 : S_p \leq \lambda \}$.
In addition, the probability that $X \sim \uniformdist([0, 1])$ falls in the cell $[\lambda^{-1} S_{k-1}, \lambda^{-1} S_{k} \wedge 1)$ ($1 \leq k \leq K_\lambda$) is $\lambda^{-1} (S_{k} \wedge 1 - S_{k-1})$, so that
\begin{align}
& \E \bigg[ \sum_{\node \in \leaves(\mondrian_\lambda)} \exp ( - n \P \left( X \in \cell_\node \right)) \bigg] \nonumber \\
&= \E \bigg[ \sum_{k=1}^{K_\lambda-1} e^{-n\lambda^{-1} (S_k - S_{k-1})} + e^{-n (1-\lambda^{-1} S_{K_\lambda - 1})} \bigg] \nonumber \\
&\leq \E \bigg[ \sum_{k=1}^\infty \indic{S_k \leq \lambda} e^{- n \lambda^{-1} E_k} \bigg] + 1 \nonumber  \\
&= \sum_{k=1}^\infty \E \big[ \indic{S_k \leq \lambda} \big] \E \big[ e^{- n \lambda^{-1} E_k} \big]  + 1 \label{eq:prooflemma-indep-ES} \\
&= \sum_{k=1}^\infty \E \big[ \indic{S_k \leq \lambda} \big] \cdot \int_{0}^\infty e^{-n \lambda^{-1} u} e^{-u} \di u + 1 \nonumber \\
&= \frac{\lambda}{n + \lambda} \E \bigg[ \sum_{k=1}^\infty \indic{S_k \leq \lambda} \bigg] + 1 \nonumber \\
&= \frac{\lambda}{n + \lambda} \E \left[ K_\lambda \right] + 1 \nonumber \\
&= \frac{\lambda}{n + \lambda} (1+\lambda) + 1 \label{eq:prooflemma-3}
\end{align}
where~\eqref{eq:prooflemma-indep-ES} comes from the fact that $E_k$ and $S_{k-1}$ are independent.
Plugging Equation~\eqref{eq:prooflemma-3} in the lower bound~\eqref{eq:lowerbound-variance} yields
\begin{align*}  
\E \left[ \big( \wh f_{\lambda,n} (X) - \wt f_\lambda (X)  \big)^2 \right]  
& \geq \frac{\sigma^2}{n} \left( (1+\lambda) - 2 (1+\lambda) \frac{\lambda}{n + \lambda} - 2 \right) \\
& = \frac{\sigma^2}{n} \left( (1+ \lambda)\frac{n - \lambda}{n + \lambda} - 2 \right) .
\end{align*}
Now, assume that $6 \leq \lambda \leq \frac{n}{3}$.
Since
\begin{align*}
(1+ \lambda) \frac{n-\lambda}{n+\lambda} - 2
\underset{(\lambda \leq n/3)}{\geq} (1+\lambda) \frac{n - n/3}{n+n/3} - 2  
= (1+\lambda) \frac{1}{2} - 2 
\underset{(\lambda \geq 6)}{\geq} \frac{\lambda}{4} \, ,
\end{align*}
the above lower bound implies, for $6 \leq \lambda \leq \frac{n}{3}$, 
\begin{equation}
\label{eq:lowerbound-variance-1}
\E \left[ \big( \wh f_{\lambda,n} (X) - \wt f_\lambda (X)  \big)^2 \right]
\geq \frac{\sigma^2 \lambda}{4 n} \, .
\end{equation}

\paragraph{Second lower bound on the variance}

The lower bound~\eqref{eq:lowerbound-variance-1} is only valid for $\lambda \leq n/3$;
as $\lambda$ becomes of order $n$ or larger, the previous bound  becomes vacuous.
We now provide another lower bound on the variance, valid when $\lambda \geq n/3$, by considering the contribution of empty cells to the variance.

Let $\node \in \leaves(\mondrian_\lambda)$.
If $\cell_\node$ contains no sample point from $\dataset_n$, then for $x \in \cell_\node$: $\wh f_{\lambda,n} (x) = 0$ and thus $(\wh f_{\lambda,n} (x) - \bar f_{\lambda}(x))^2 = \bar f_{\lambda}(x)^2 \geq 1$.
Hence, the variance term is lower bounded as follows, denoting $N_{n} (\cell)$ the number of $1\leq i\leq n$ such that $X_i \in \cell$ and $N_{\lambda,n} (x) = N_n (\cell_\lambda(x))$:
\begin{align}
\E \big[ ( \wh f_{\lambda,n} (X) - \bar f_\lambda (X))^2 \big]
&\geq \P \big( N_{\lambda,n} (X) = 0 \big) \nonumber \\
&= \E \bigg[ \sum_{\node \in \leaves(\mondrian_\lambda)} \P (X \in \cell_\node) \, \P (N_{n} (\cell_\node) = 0) \bigg] \nonumber \\
&= \E \bigg[ \sum_{\node \in \leaves(\mondrian_\lambda)} \P (X \in \cell_\node) \big( 1 - \P (X \in \cell_\node) \big)^n \bigg] \nonumber \\
&\geq \E \bigg[ \Big( \sum_{\node \in \leaves(\mondrian_\lambda)} \P (X \in \cell_\node) \big( 1 - \P (X \in \cell_\node) \big) \Big)^n \bigg] \label{eq:prooflemma-jensen-1} \\  
&\geq \E \bigg[ \sum_{\node \in \leaves(\mondrian_\lambda)} \P (X \in \cell_\node) \big( 1 - \P (X \in \cell_\node) \big) \bigg]^n \label{eq:prooflemma-jensen-2} \\
&= \bigg( 1 - \E \bigg[ \sum_{\node \in \leaves (\mondrian_\lambda)} \P ( X \in \cell_\node)^2 \bigg] \bigg)^n  \label{eq:prooflemma-var-empty-1}
\end{align}
where~\eqref{eq:prooflemma-jensen-1} and~\eqref{eq:prooflemma-jensen-2} come from Jensen's inequality applied to the convex function $x \mapsto x^n$.
Now, using the notations defined above, we have
\begin{align}
\E \bigg[ \sum_{\node \in \mondrian_\lambda} \P ( X \in \cell_\node)^2 \bigg]
&\leq \E \bigg[ \sum_{k=1}^{K_\lambda} (\lambda^{-1} E_k)^2 \bigg] \nonumber \\
&= \lambda^{-2} \, \E \bigg[ \sum_{k=1}^{\infty} \indic{S_{k-1} \leq \lambda} E_k^2 \bigg] \nonumber \\
&= \lambda^{-2} \, \E \bigg[ \sum_{k=1}^{\infty} \indic{S_{k-1} \leq \lambda} \E \big[ E_k^2 \cond S_{k-1} \big] \bigg] \nonumber \\
&= 2 \lambda^{-2} \, \E \bigg[ \sum_{k=1}^{\infty} \indic{S_{k-1} \leq \lambda} \bigg] \label{eq:prooflemma-exp-square} \\
&= 2 \lambda^{-2} \, \E \left[ K_\lambda \right] \nonumber \\
&= \frac{2 (\lambda + 1)}{\lambda^2} \label{eq:prooflemma-var-empty-2} \, ,
\end{align}
where the equality $\E [ E_k^2 \cond S_{k-1}]= 2$ (used in Equation~\eqref{eq:prooflemma-exp-square}) comes from the fact that $E_k \sim \expdist(1)$ is independent of $S_{k-1}$.

The bounds~\eqref{eq:prooflemma-var-empty-1} and~\eqref{eq:prooflemma-var-empty-2} imply that, if $2(\lambda+1)/\lambda^2 \leq 1$, then
\begin{equation}
\label{eq:prooflemma-var-empty-3}
\E \big[ ( \wh f_{\lambda,n} (X) - \bar f_\lambda (X))^2 \big]
\geq \left( 1 - \frac{2 (\lambda + 1)}{\lambda^2} \right)^n
.
\end{equation}
Now, assume that $n \geq 18$ and $\lambda \geq \frac{n}{3} \geq 6$.
Then
\begin{equation*}
\frac{2(\lambda+1)}{\lambda^2}
\leq 2 \cdot \frac{3}{n} \left( 1 + \frac{3}{n} \right)
\leq 2 \cdot \frac{3}{n} \left( 1 + \frac{3}{18} \right)
= \frac{7}{n}
\underset{(n\geq 18)}{\leq} 1 \, ,
\end{equation*}
so that, using the inequality $(1-x)^m \geq 1 - mx$ for $m \geq 0$ and $x \in \R$,
\begin{equation*}
\left( 1 - \frac{2(\lambda+1)}{\lambda^2} \right)^{n/8}
\geq \left( 1 - \frac{7}{n} \right)^{n/8}
\geq 1 - \frac{n}{8} \cdot \frac{7}{n}
= \frac{1}{8} \, .
\end{equation*}
Combining the above inequality with~\eqref{eq:prooflemma-var-empty-3} gives, letting $C_2 := 1/8^8$,
\begin{equation}
\label{eq:prooflemma-var-empty}
\E \big[ ( \wh f_{\lambda,n} (X) - \bar f_\lambda (X))^2 \big]
\geq C_2 \, .
\end{equation}

\paragraph{Summing up}
Assume that $n \geq 18$.
Recall the bias-variance decomposition~\eqref{eq:bias-variance-prooflemma} of the risk $\risk(\wh f_{\lambda,n})$ of the Mondrian tree.

\begin{itemize}
  \item If $\lambda \leq 6$, we saw that the bias (and hence the risk) is larger than $C_1$;
  \item If $\lambda \geq \frac{n}{3}$, Equation~\eqref{eq:prooflemma-var-empty-3} implies that the variance (and hence the risk) is larger than $C_2$;
  \item If $6 \leq \lambda \leq \frac{n}{3}$, Equations~\eqref{eq:lowerbound-bias-prooflemma} (bias term) and~\eqref{eq:lowerbound-variance-1} (variance term)
  imply that
  \begin{equation*}
  \risk(\wh f_{\lambda,n})
  \geq \frac{1}{3 \lambda^2} + \frac{\sigma^2 \lambda}{4n} \, .
  \end{equation*}
\end{itemize}
In particular,
\begin{equation}
\label{eq:prooflemma-final}
\inf_{\lambda \in \R^+} \risk(\wh f_{\lambda,n})
\geq C_1 \wedge C_2 \wedge \inf_{\lambda \in \R^+} \left( \frac{1}{3 \lambda^2} + \frac{\sigma^2 \lambda}{4n} \right)
= C_0 \wedge \frac{1}{4} \left( \frac{3 \sigma^2}{n} \right)^{2/3}  
\end{equation}
where we let $C_0 = C_1 \wedge C_2$.

\section{Proof of Proposition~\ref{prop:adaptive-rate}} 
\label{ssub:proof_of_proposition_prop:adaptive-rate}

First, note that in all cases, since $| Y | \leq B$ almost surely, we also have $|\wh g_n (X) | \leq B$ almost surely, so that $(Y - \wh g_n (X))^2 \leq 4 B^2$.  
Let $N_\eps = | I_\eps |$.
Note that $N_\eps$ is a binomial variable with parameters $n-n_0 \geq n/2$ and $\P (X \in B_\eps) \geq p_0 (1-2\eps)^d$ (since $p \geq p_0$).
Now, recall Chernoff's bound: if $N \sim \binomdist(m,p)$ and $\delta \in (0, 1)$, then $\P ( N \leq (1-\delta) m q) \leq e^{- m q \delta^2/2}$; in particular, $\P (N \leq m q /2) \leq e^{- m q/8}$.
Hence, letting $c_1 = p_0 (1-2\eps)^d/4$,
\begin{equation}
\label{eq:binomial-lower}
\P \big( N_\eps \leq  c_1 n \big)
\leq \exp (- c_1 n/4 )
\, .
\end{equation}
Conditionally on $I_{\eps}$, the sample ${\dataset}' = \{ (X_i, Y_i) : i \in I_\eps \}$ is an \iid sample of size $N_\eps$ of the conditional distribution of $(X,Y)$ given $X \in B_\eps$; it is also independent of $\dataset_{n_0}$, and thus of the estimators $\wh f_\alpha$, $\alpha = 0, \dots, A$.  
It follows from Theorem~1 in the supplementary material ``Proof of the optimality of the empirical star algorithm'' of~\cite{audibert2008deviation} that the estimator $\wh g_n$ defined by~\eqref{eq:star-aggregated-estimator} satisfies, with probability $1-\delta$ over the random sample $\dataset'$ conditionally on $N_\eps$,
\begin{align}
\label{eq:proof-adaptive-1}
\E_{(X,Y)} \big[ (\wh g_n (X) - Y)^2 \cond X \in B_\eps \big]
&- \min_{0 \leq \alpha \leq A} \E_{(X,Y)} \big[ (\wh f_\alpha (X) - Y)^2 \cond X \in B_\eps \big] \nonumber \\
&\leq \frac{C B^2 \log [(A+1) \delta^{-1}]}{N_\eps}
\end{align}
for every $\delta \in (0, 1)$, where $C=600$ and the expectation is taken with respect to an independent sample $(X,Y)$ (the bound~\eqref{eq:proof-adaptive-1} is deduced from the aforementioned theorem by replacing $Y$ by $Y/B$, which lies in $[-1, 1]$).
Since $Y = f(X) + \eps$ with $\E [\eps | X] = 0$, we have $\E [ (g(X) - Y)^2 | X ] = \E [ (g(X) - f(X))^2 | X ] + \E [\eps^2 \cond X]$.
Hence, inequality~\eqref{eq:proof-adaptive-1} writes
\begin{align*}
&\E_{(X,Y)} \big[ (\wh g_n (X) - f(X))^2 \cond X \in B_\eps \big] \\
& \leq \min_{0 \leq \alpha \leq A} \E_{(X,Y)} \big[ (\wh f_\alpha (X) - f(X))^2 \cond X \in B_\eps \big] 
+ \frac{C B^2 \log [(A+1) \delta^{-1}]}{N_\eps} .
\end{align*}
By integrating the above inequality over the confidence level $\delta$, 
we obtain
\begin{align*}
&\E_{(X,Y), \dataset'} \big[ (\wh g_n (X) - f(X))^2 \cond X \in B_\eps, N_\eps \big] \\
& \leq \min_{0 \leq \alpha \leq A} \E_{(X,Y)} \big[ (\wh f_\alpha (X) - f(X))^2 \cond X \in B_\eps \big] 
+ \frac{C B^2 [\log (A+1) + 1]}{N_\eps} ;
\end{align*}
by taking the expectation over $\dataset_{n_0}$, conditioning on $N_\eps > c_1 n$, and recalling that $A \leq \log_2 (n)$, we get
\begin{align}
\label{eq:oracle-N-large}
& \E \big[ (\wh g_n (X) - f(X))^2 \cond X \in B_\eps, N_\eps > c_1 n \big]  \\
&\leq \min_{0 \leq \alpha \leq A} \E \big[ (\wh f_\alpha (X) - f(X))^2 \cond X \in B_\eps \big] 
+ \frac{C B^2 [\log (1+ \log_2 n) + 1]}{c_1 n}.  \nonumber
\end{align}
Finally, combining the bounds~\eqref{eq:binomial-lower} and~\eqref{eq:oracle-N-large} yields
\begin{align}
\label{eq:proof-adaptive-2}
&\E \big[ (\wh g_n (X) - f(X))^2 \cond X \in B_\eps \big] \nonumber \\
&\leq \P \left( N_\eps \leq c_1 n \right) \cdot 4B^2 + \E \big[ (\wh g_n (X) - f(X))^2 \cond X \in B_\eps, N_\eps > c_1 n \big] \nonumber \\
&\leq 4 B^2 e^{-c_1 n/4} + \min_{0 \leq \alpha \leq A} \E \big[ (\wh f_\alpha (X) - f(X))^2 \cond X \in B_\eps \big]  \\
&\qquad + \frac{C B^2 [\log (1+ \log_2 n) + 1]}{c_1 n} \, , \nonumber
\end{align}
which is precisely inequality~\eqref{eq:oracle-inequality}.

Assume that $f$ belongs to the class $\cl^{p,\beta} (L)$, with $p \in \{ 0, 1\}$, $\beta \in (0, 1]$ and $L > 0$; we now proceed to show that $\wh g_n$ achieves the minimax rate of estimation for this class.
Let $s = p + \beta \in (0, 2]$.
If $p = 0$ (namely, $s \leq 1$), it follows from Theorem~\ref{thm:minimax-regression} (with the same adaptation as in the proof of Theorem~\ref{thm:minimaxc2} to bound the variance term conditionally on $X \in B_\eps$) that, for every $\lambda > 0$,
\begin{equation*}
\E \big[ (\wh f_{\lambda, n_0, M} (X) - f(X))^2 \cond X \in B_\eps \big]
\leq \frac{(4d)^s L^2}{\lambda^{2s}} + \frac{11 B^2 (1+\lambda)^d}{p_0 (1-2\eps)^d n_0} 
\end{equation*}
(note that $\sigma, \| f\|_\infty \leq B$ since $|Y| \leq B$).
It follows that, for some constants $C_1, C_2$ independent of $\lambda,L,n$,
\begin{align}
\label{eq:bound-grid-lambda}
\min_{0 \leq \alpha \leq A} \E \big[ (\wh f_\alpha (X) - f(X))^2 \cond X \in B_\eps \big]
&\leq \min_{0 \leq \alpha \leq A} \left[ \frac{C_1 L^2}{(2^\alpha)^{2s}} + \frac{C_2 (1+2^{\alpha})^d}{n} \right] \nonumber \\
&\leq 4 \min_{\lambda \in [1, n^{1/d}]} \left[ \frac{C_1 L^2}{\lambda^{2s}} + \frac{C_2 (1+\lambda)^d}{n} \right] \, ,
\end{align}
where we used the fact that, for every $\lambda \in [1, n^{1/d}]$, there exists some $\alpha$, $0 \leq \alpha \leq A$, such that $\lambda / 2 \leq 2^\alpha \leq \lambda$.
It follows from~\eqref{eq:proof-adaptive-2} and~\eqref{eq:bound-grid-lambda} that
\begin{align*}
&\E \big[ (\wh g_n (X) - f(X))^2 \cond X \in B_\eps \big] \\
&\leq O \Big( \min_{0 \leq \lambda \leq n^{1/d}} \Big[ \frac{C_1 L^2}{\lambda^{2s}} + \frac{C_2 (1+\lambda)^d}{n} \Big] + \frac{\log \log n}{n} \Big) \\
&= O \left( L^{2d/(d+2s)} n^{-2s/(d+2s)} \right)
\end{align*}
where the last bound follows from the fact that $\lambda_* = (L^2 n)^{1/(d+2s)}$ belongs to $[1, n^{1/d}]$ for $n$ large enough (and $\log \log n / n = o (n^{2s/(d+2s)})$).

Now, consider the case $p=1$, \ie, $1 < s \leq 2$.
It follows from Theorem~\ref{thm:minimaxc2} that for some constants $C_3, C_4$ independent of $\lambda, L, n$, we have for every $\lambda \in [1, n^{1/d}]$ (using the fact that $M \geq n^{2/d} \geq \lambda^2$, so that $1/(M\lambda^2) \leq 1/\lambda^4 \leq 1/\lambda^{2s}$, and $e^{-\lambda \eps}/\lambda^3 = O (1/\lambda^{2s})$)
\begin{equation}
\label{eq:proof-adaptive-c2-lambda}
\E \big[ (\wh f_{\lambda, n, M} (X) - f(X))^2 \cond X \in B_\eps \big] \\
\leq \frac{C_3 L^2}{\lambda^{2s}} + \frac{C_4 (1+\lambda)^d}{n}
\, .
\end{equation}
From the same argument as in the case $0 < s \leq 1$, combining inequalities~\eqref{eq:proof-adaptive-c2-lambda} and~\eqref{eq:proof-adaptive-2} yields
\begin{equation*}
\E \big[ (\wh g_n (X) - f(X))^2 \cond X \in B_\eps \big]
= O \big( L^{2d/(d+2s)} n^{-2s/(d+2s)} \big)
\end{equation*}
which concludes the proof of Proposition~\ref{prop:adaptive-rate}. $\hfill \qed$

\section{Proof of Lemma~\ref{lem:int_against_Flambda}}

According to Equation~\eqref{eq:F_lambda} from the main text, we have
\begin{equation}
\label{eq:mf-bias-form}
F_\lambda (x, z) = \lambda^d \exp (-\lambda \| x - z \|_1) \prod_{1\leq j \leq d} G_\lambda (x_j, z_j)
\end{equation}
where
we defined, for $u, v \in [0, 1]$,
\begin{align*}
\label{eq:mf-bias-form-G}
G_\lambda (u, v) &= \E \left[ \left( \lambda |u - v| + E_1 \wedge \lambda (u \wedge v) + E_2 \wedge \lambda (1 - u \vee v) \right)^{-1} \right] \\
&= H ( \lambda |u-v|, \lambda u \wedge v, \lambda (1 - u \vee v) )
\end{align*}
with $E_1, E_2$ two independent $\expdist (1)$ random variables, and $H : (\R_+^*)^3 \to \R$ the function defined by 
\begin{equation*}
H (a, b_1, b_2) = \E \left[ \left( a + E_1 \wedge b_1 + E_2 \wedge b_2 \right)^{-1} \right] \, ;
\end{equation*}
also, let
\begin{equation*}
H (a)
= \E \left[ \left( a + E_1 + E_2 \right)^{-1} \right] .
\end{equation*}
Denote
\begin{align*}
A &= \int_{[0, 1]^d} (z-x)  F_\lambda (x, z) \di z  \\
B &= \int_{[0, 1]^d} \frac{1}{2} \| z - x \|^2 F_\lambda (x, z) \di z.
\end{align*}
Since $1= \int F_\lambda^{(1)} (u, v) \di v = \int \lambda \exp (-\lambda |u-v|) G_\lambda (u, v) \di v$, applying Fubini's theorem we obtain
\begin{equation}
\label{eq:term-1}
A_j
= \Phi_\lambda^1 (x_j)
\qquad \mbox{and} \qquad
B 
= \sum_{j=1}^d \Phi_\lambda^2 (x_j)
\end{equation}
where we define for $u \in [0, 1]$ and $k \in \N$
\begin{equation}
\label{eq:term-def}
\Phi_\lambda^k (u) = \int_0^1 \lambda \exp (-\lambda |u-v|) G_\lambda (u,v) \frac{(v-u)^k}{k!} \di v
\, .
\end{equation}
Observe that
\begin{equation*}
\Phi_\lambda^k (u) = \lambda^{-k} \int_{-\lambda u}^{\lambda (1-u)} \frac{v^k}{k!} \exp (- |v|) H ( |v| , \lambda u + v \wedge 0, \lambda (1 - u) - v \vee 0) \di v
\, .
\end{equation*}
We will control $\Phi_\lambda^k(u)$ for $k=1, 2$.
First, write
\begin{align*}
\lambda \Phi_\lambda^1 (u)
& = - \int_{0}^{\lambda u} v e^{-v} H(v, \lambda u - v, \lambda (1-u)) \di v \\
& \quad + \int_{0}^{\lambda (1-u)} v e^{-v} H(v, \lambda u , \lambda (1-u) - v ) \di v  
\end{align*}
Now, let $\beta := \lambda \frac{u \wedge (1-u)}{2}$.
We have
\begin{align*}
&\quad \lambda \Phi_\lambda^1 (u) - \int_{0}^{\beta} v e^{-v} \left[ H(v, \lambda u , \lambda (1-u) - v ) - H(v, \lambda u - v, \lambda (1-u)) \right] \di v = \\
&- \underbrace{\int_{\beta}^{\lambda u} v e^{-v} H(v, \lambda u - v, \lambda (1-u)) \di v}_{:= I_1 \geq 0} 
+ \underbrace{\int_{\beta}^{\lambda (1-u)} v e^{-v} H(v, \lambda u , \lambda (1-u) - v ) \di v}_{:= I_2 \geq 0}
\end{align*}
so that the left-hand side of the above equation is between $-I_1 \leq 0$ and $I_2 \geq 0$, and thus its absolute value is bounded by $|I_1| \vee |I_2|$.
Now, note that, since $H(v, \cdot, \cdot) \leq v^{-1}$, we have
\begin{equation*}
|I_2| \leq \int_{\beta}^{\infty} v e^{-v} v^{-1} \di v = e^{-\beta}
\end{equation*}
and similarly $|I_1|\leq e^{-\beta}$, so that
\begin{align}
\label{eq:bound1-step1}
& \bigg| \lambda \Phi_\lambda^1 (u) - \underbrace{\int_{0}^{\beta} v e^{-v} \left[ H(v, \lambda u , \lambda (1-u) - v ) - H(v, \lambda u - v, \lambda (1-u)) \right] \di v}_{:= I_3} \bigg| \nonumber \\
& \leq e^{-\beta}
\end{align}
It now remains to bound $|I_3|$.
For that purpose, note that since $H$ is decreasing in its second and third argument, we have
\begin{align*}
& H(v) - H(v, \lambda u - v, \lambda (1-u))\\
&\leq H(v, \lambda u , \lambda (1-u) - v ) - H(v, \lambda u - v, \lambda (1-u)) \\
&\leq H(v, \lambda u , \lambda (1-u) - v ) - H(v)
\end{align*}
which implies
\begin{align*}
& | H(v, \lambda u , \lambda (1-u) - v ) - H(v, \lambda u - v, \lambda (1-u)) |  \\
&\leq \max (|H(v, \lambda u , \lambda (1-u) - v ) - H(v)|, |H(v) - H(v, \lambda u - v, \lambda (1-u))|). 
\end{align*}
Besides, since $(a + E_1 \wedge b_1 + E_2 \wedge b_2)^{-1} \leq (a + E_1 + E_2)^{-1} + a^{-1} (\bm 1 \{ E_1 \geq b_1 \} + \bm 1 \{  E_2 \geq b_2 \} )$,
\begin{equation}
\label{eq:H-bound-2}
H(a, b_1, b_2) - H(a)
\leq a^{-1} (e^{-b_1} + e^{-b_2}),
\end{equation}
for all $a, b_1, b_2$. Since  $\lambda u - v \geq \beta$ and $\lambda (1-u) - v \geq \beta$ for $v\in [0, \beta]$, we have
\begin{align*}
| H(v) - H(v, \lambda u - v, \lambda (1-u)) |
, | H(v) - H(v, \lambda u , \lambda (1-u) - v) |
&\leq 2 v^{-1} e^{-\beta}
\end{align*}
so that for $v \in [0, \beta]$
\begin{equation*}
| H(v, \lambda u , \lambda (1-u) - v ) - H(v, \lambda u - v, \lambda (1-u)) |
\leq 2 v^{-1} e^{-\beta}
\end{equation*}
and hence
\begin{align}
\left| I_3 \right|
&\leq \int_{0}^{\beta} v e^{-v} \left| H(v, \lambda u , \lambda (1-u) - v ) - H(v, \lambda u - v, \lambda (1-u)) \right| \di v \nonumber  \\
&\leq \int_{0}^{\beta} v e^{-v} 2v^{-1} e^{-\beta} \di v \nonumber \\
&\leq 2e^{-\beta} \int_0^{\infty} e^{-v} \di v \nonumber \\
&= 2 e^{-\beta} \label{eq:bound1-step2}
\end{align}
Combining Equations~\eqref{eq:bound1-step1} and~\eqref{eq:bound1-step2} yields:
\begin{equation}
\label{eq:bound-moment-1}
| \Phi_\lambda^1 (u) |
\leq \frac{3}{\lambda} e^{- \lambda [u \wedge (1-u)]/{2}}
\end{equation}
that is,
\begin{align*}
\left\| \int_{[0, 1]^d} (z-x) F_\lambda (x,z) \di z \right\|^2
= \sum_{j=1}^d \left( \Phi_\lambda^1 (x_j) \right)^2
\leq \frac{9}{\lambda^2} \sum_{j=1}^d e^{-\lambda [x_j \wedge (1-x_j)]} \, .
\end{align*}
Furthermore,
\begin{align*}
0 &\leq \Phi_\lambda^2 (u)
= \lambda^{-2} \int_{-\lambda u}^{\lambda (1-u)} \frac{v^2}{2} e^{-|v|} H ( |v| , \lambda u + v \wedge 0, \lambda (1 - u) - v \vee 0) \di v \\
&\leq \lambda^{-2} \int_{0}^{\infty} v^{2} e^{-v} v^{-1} \di v \\
&= \lambda^{-2}
\end{align*}
so that
\begin{equation*}
\label{eq:bound-moment-2}
0
\leq \Phi_\lambda^2 (u)
\leq \frac{1}{\lambda^2},
\end{equation*}
which proves the second inequality by summing over $j=1,\hdots,d$.
This concludes the proof of Lemma~\ref{lem:int_against_Flambda}. $\hfill \qed$

\end{document}